\documentclass[letterpaper]{article} 
\usepackage{aaai2026}  
\usepackage{times}  
\usepackage{helvet}  
\usepackage{courier}  
\usepackage[hyphens]{url}  
\usepackage{graphicx} 
\usepackage{natbib}  
\usepackage{caption} 
\usepackage{algorithm}
\usepackage{algorithmic}
\usepackage{xspace}

\usepackage{newfloat}
\usepackage{listings}

\usepackage{times}
\usepackage{soul}
\usepackage{url}
\usepackage[utf8]{inputenc}
\usepackage{caption}
\usepackage{graphicx}
\usepackage{amsmath}
\usepackage{amsthm}
\usepackage{booktabs}
\usepackage[switch]{lineno}
\usepackage{amsfonts}

\usepackage{tabularx}
\usepackage{adjustbox}
\usepackage{multirow}
\usepackage{subcaption}
\usepackage{xcolor}  
\usepackage{tikz}
\usepackage{rotating}  
\usepackage{pgfplots}  
\pgfplotsset{compat=newest} 

\DeclareCaptionStyle{ruled}{labelfont=normalfont,labelsep=colon,strut=off} 
\floatstyle{ruled}
\newfloat{listing}{tb}{lst}{}
\floatname{listing}{Listing}
\title{On Robustness of Linear Classifiers to Targeted Data Poisoning}
\author {
    Nakshatra Gupta\textsuperscript{\rm 1},
    Sumanth Prabhu S\textsuperscript{\rm 1}\thanks{Now at Relyance AI, Bangalore, India},
    Supratik Chakraborty\textsuperscript{\rm 2},
    R Venkatesh\textsuperscript{\rm 1}
}
\affiliations {
    \textsuperscript{\rm 1}Tata Consultancy Services, Pune, India\\
    \textsuperscript{\rm 2}IIT Bombay, Mumbai, India\\
    nakshatra.g@tcs.com, sumanthsprabhu@gmail.com, supratik@cse.iitb.ac.in, r.venky@tcs.com
}

\usepackage{latexsym}

\newtheorem{theorem}{Theorem}

\newcommand{\dimension}{\mathit{d}}
\newcommand{\features}{\mathcal{X}}
\newcommand{\labels}{\mathcal{Y}}
\newcommand{\model}{\mathit{f}}
\newcommand{\weight}{w}
\newcommand{\weights}{\mathbf{\weight}}
\newcommand{\bias}{\mathit{b}}
\newcommand{\linmodel}{\mathit{f_{\weights, \bias}}}
\newcommand{\hypo}{\mathcal{H}}
\newcommand{\sgnfn}{\operatorname{sign}}
\newcommand{\tdata}{\mathcal{D}}
\newcommand{\ctdata}{\mathcal{D}_c}
\newcommand{\ptdata}{\mathcal{D}_p}
\newcommand{\robustness}{\mathit{r}}
\newcommand{\uprob}{\hat{\robustness}}
\newcommand{\lowrob}{\check{\robustness}}
\newcommand{\bigM}{\mathit{M}}
\newcommand{\nbigM}{\mathit{-M}}
\newcommand{\tsize}{\mathit{m}}
\newcommand{\lossfn}{\mathit{l}}
\newcommand{\robint}{\textsc{RobustRange}\xspace}
\newcommand{\ipr}{IPr\xspace}
\newcommand{\iprtool}{\textsc{IP-relabel}\xspace}

\begin{document}

\maketitle

\begin{abstract}
Data poisoning is a training-time attack that undermines the trustworthiness of learned models. In a targeted data poisoning attack, an adversary manipulates the training dataset to alter the classification of a targeted test point. Given the typically large size of training dataset, manual detection of poisoning is difficult. An alternative is to automatically measure a dataset's robustness against such an attack, which is the focus of this paper. We consider a threat model wherein an adversary can only perturb the labels of the training dataset, with knowledge limited to the hypothesis space of the victim's model. In this setting, we prove that finding the robustness is an NP-Complete problem, even when hypotheses are linear classifiers. To overcome this, we present a technique that finds lower and upper bounds of robustness. Our implementation of the technique computes these bounds efficiently in practice for many publicly available datasets. We experimentally demonstrate the effectiveness of our approach. Specifically, a poisoning exceeding the identified robustness bounds significantly impacts test point classification. We are also able to compute these bounds in many more cases where state-of-the-art techniques fail.
\end{abstract}

\section{Introduction}
\label{sec:intro}

\emph{Data poisoning} is a training-time attack wherein an adversary perturbs the training dataset to alter the predictions of the trained models~\cite{biggio2012poisoning}. Here, the perturbation is called poisoning. Manually determining whether a training dataset is poisoned or not is challenging due its size. Consequently, this attack affects the trustworthiness of machine learning models and hinders their deployment in industries~\cite{kumar2020adversarial}.

\emph{Targeted data poisoning} refers to an adversary manipulating the training data to alter the prediction of a single (or a small fixed subset of) test data. While targeted data poisoning can be sophisticated, such as backdoor attacks that use patterns called triggers, we focus on trigger-less attack. In such an attack, an adversary targets a specific test point and aims for the victim's classifier to classify this point as a desired class. An example is a loan applicant adversary who poisons the training data so that the model trained on the poisoned data approves the adversary's loan.

 We consider a practical adversary who can poison the training dataset only by contaminating labels. While contaminating the features of a dataset can be difficult, labels can be easily manipulated as label errors are common, especially when they are collected from external sources~\cite{adebayo2023quantifying}. Furthermore, we assume that the adversary has knowledge of only the hypothesis space of the victim's classifier, but not the training procedure or the actual classifier. In other words, the victim's training procedure is treated as a black-box. In this setting, we define robustness as:
\begin{quote}
\itshape
The minimum number of label perturbations in the training dataset required for a classifier from the hypothesis space to classify the test point as desired.    
\end{quote}

For example, consider the dataset Pen-based Handwritten Digit Recognition~\cite{pen-based_recognition_of_handwritten_digits_81}. In the dataset, a test point shown in Figure~\ref{fig:digit_attack_plot} represents the handwritten digit `0'. However, our tool finds that it can be labeled as the digit `4' by changing the labels of just two points in the training dataset, out of thousands. 

The above notion of robustness is not only useful to measure confidence in the training data, but also in contesting a model's predictions. When a model produces an undesirable decision, the user may want to identify the minimal set of training point labels that potentially influenced the outcome. If any of these labels are incorrect, the user can flag and contest the decision. This application is directly addressed by a solution to the robustness problem defined above. 

In our setting of targeted black-box poisoning by label perturbation, we focus on linear classifiers in this work. While existing data poisoning work typically considers linear classifiers~\cite{biggio2011support, biggio2012poisoning, xiao2012adversarial, xiao2015support,suya2024distributions, zhao2017efficient, paudice2019label, yang2023relabeling}, our setting considers a weaker threat model. Moreover, linear models are still relevant as they perform well on many tasks ~\cite{ferrari2019we,tramer2020differentially,chen2021learning,chen2023continuous} and are subject of several recent related works~\cite{yang2023relabeling,yang2023many,suya2024distributions}.

We make five contributions in this paper. We first establish a theoretical result showing that the problem of finding robustness is NP-Complete. To address this challenge, we propose techniques that approximate robustness. In this direction, our second contribution is a partition-based method that computes a lower bound for robustness. Third, we introduce an augmented learning procedure designed to find an upper bound for robustness. Our fourth contribution is a prototype tool \robint, which implements our novel techniques. Finally, we evaluate \robint on 15 publicly available datasets of different sizes. Our experiments reveal several interesting facts: the average robustness can be as low as 3\%, \robint performs better than a  SOTA tool~\cite{yang2023relabeling}, even when the SOTA tool knows the victim's classifier and learning procedure. 

\begin{figure}[t]  
    \centering  
    \includegraphics[width=0.45\textwidth]{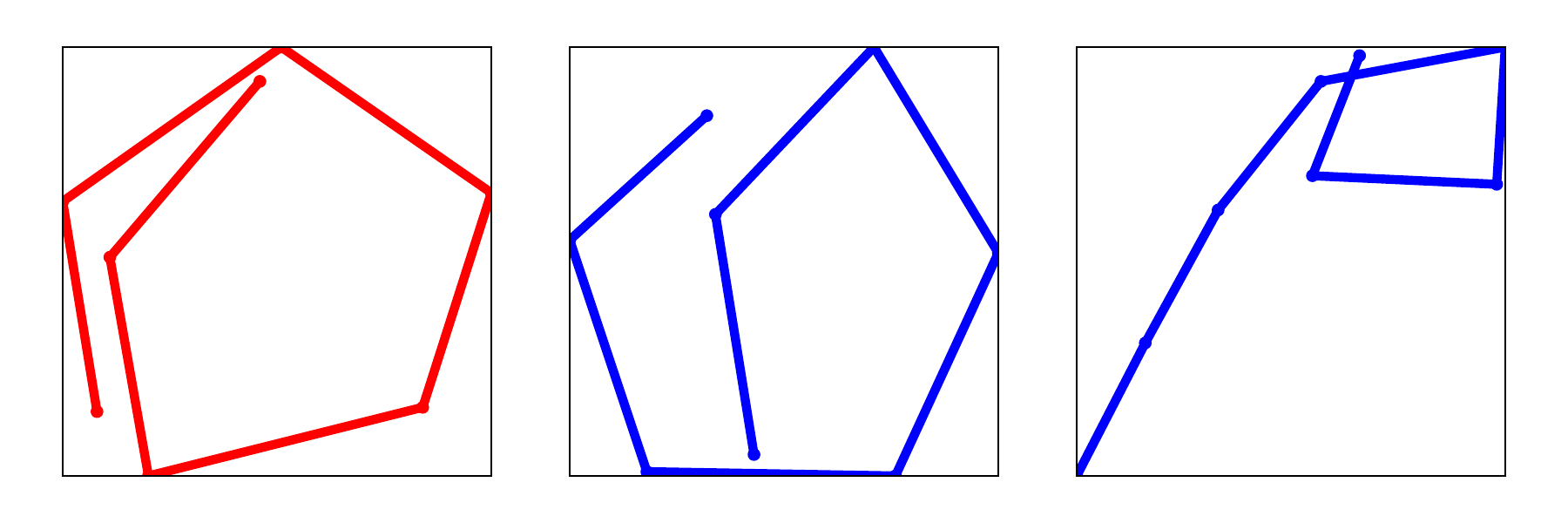}  
    \caption{
    A test point (leftmost plot) and two training points all originally labeled as digit `0'. Our tools finds that modifying only these two training points as digit `4' can change the test point's label to digit `4'. 
    }  
    \label{fig:digit_attack_plot}  
\end{figure}

\section{Related Work}
Data poisoning attacks have received considerable attention as a primary security threat during the training stage of machine learning~\cite{barreno2010security,tian2022comprehensive}. Depending on the adversary's goals, these attacks can be either indiscriminate or targeted. In an indiscriminate attack, the adversary aims to maximize the overall loss of victim's classifier, which was the focus of many works~\cite{biggio2011support, biggio2012poisoning, xiao2012adversarial, xiao2015support,suya2024distributions}. In contrast, a targeted attack aims to alter the prediction of a specific test point while minimally affecting overall loss, which is the focus of this paper.  

A targeted attack can be achieved by assuming different capabilities of the adversary. For instance, an attacker can add/modify/remove training points without changing labels~\cite{shafahi2018poison, suciu2018does, zhu2019transferable,yang2023many}, posses full or partial knowledge of the victim's learning process~\cite{cina2021hammer,csuvak2022design,paudice2019label}, or trigger poisoning when the test point is embedded by certain patterns~\cite{chen2017targeted,saha2020hidden,gu2019badnets}. We assume a more constrained attacker who can only perturb labels, lacks knowledge of the victim's learning process (black-box), and targets a specific test point. 

There are techniques that construct a poisoned dataset with stronger threat model.  In particular, ~\cite{zhao2017efficient} also considers label perturbations and a black-box victim model, but requires a bound on robustness to compute the poisoned dataset, along with an objective model. The work in ~\cite{paudice2019label} also computes a perturbation set, but it similarly requires a bound and the victim's loss function. The work closest to ours is ~\cite{yang2023relabeling}, which aims to find the minimal label-perturbed set. However, it only computes an upper bound, and, more importantly, requires knowledge of the victim's model and loss function to determine the \emph{influence function}~\cite{koh2017understanding}. 

 ~\cite{gao2021learning} gives a relationship between the dataset size and maximum allowed perturbations, whose accuracy has been subsequently improved in ~\cite{wang2022lethal}. The work in ~\cite{hanneke2022optimal} characterizes the optimal error rate when the dataset is poisoned. In these works, data points are either removed or inserted. Furthermore, these works characterizes bounds on targeted data poison, while we compute upper and lower bounds. An orthogonal line of work is that of certified robustness~\cite{rosenfeld2020certified,levine2020deep,wang2022improved,jia2022certified}, which gives a model that predicts with guaranteed robustness.

\section{Problem Setting}
\subsection{Preliminaries}
Consider the task of binary classification from $\dimension$ dimensional input features $\features \subseteq \mathbb{R}^{\dimension}$ to binary output labels $\labels = \{+1 , -1\}$. The goal of the classification task is to find a function (called classifier) $\model : \features \to \labels$ that has a small generalization error. In machine learning, $\model$ is derived from a given set of hypotheses (called hypothesis space) $\hypo$ and a set of labeled training data $\tdata$. We assume $\hypo$ is a set of linear functions $\linmodel$ of the form $\linmodel(x) = \sgnfn(\weights^T x + b)$, where $\weights \in \mathbb{R}^{\dimension}$ is the weight parameter, and $b \in \mathbb{R}$ is the bias parameter. The labeled training data $\tdata$ is generated from an i.i.d over $\features \times \labels$. Given $\hypo$ and $\tdata$, the classification task is to find a model $\linmodel \in \hypo$ that minimizes a loss function (\text{e.g.}, hinge loss) over $\tdata$.

\subsection{Threat Model}
We model the overall workflow via \emph{causative targeted attack} ~\cite{barreno2010security}, which is a game between a \emph{victim} and an \emph{attacker}, that proceeds as follows:
\begin{enumerate}
    \item The victim generates a clean training data $\ctdata = \{(x_i, y_i) \mid i \in [\tsize]\}$ from an i.i.d over $\features \times \labels$.
    \item The attacker \emph{poisons} $\ctdata$ by perturbing (a subset of) labels to get $\ptdata = \{(x_i, y_i') \mid i \in [\tsize],\; y_i'=y_i \text{ or } y_i' = -y_i\}$.
    \item The victim finds a model $\linmodel \in \hypo$ that minimizes a loss function over $\ptdata$.
    \item The attacker has a target test point with label $(x_t, y_t) \in \features \times \labels$, where $(x_t, y)$ is not in $\ctdata$ for any $y \in \labels$.  The attacker wins the game if the victim learns a model $\linmodel$ from $\ptdata$ such that $\linmodel(x_t) = y_t$. Otherwise, victim wins the game.     
\end{enumerate}

We assume an attacker who can only perturb labels in $\ctdata$, and is aware of the victim's hypothesis space, but not the training procedure, the loss function, or the model's weights. 

\subsection{Robustness}
The robustness $\robustness$ of a dataset $\ctdata = \{(x_i, y_i) \mid i \in [\tsize]\}$ with respect to a target test point $(x_t, y_t)$ is the minimum number of label perturbations required for the attacker to win the above game.  Sans knowledge of the victim's loss function, we define robustness as follows.  Let $\mathfrak{P}$ be the set of all (possibly) poisoned datasets $\ptdata = \{(x_i, y_i') \mid i \in [m],\; y_i'=y_i \text{ or } y_i' = -y_i\}$ s.t.
 there exists $\linmodel \in \hypo$ with: 
    \begin{enumerate}
        \item $\linmodel(x_i) = y_i', \quad \forall i \in [m]$
        \item $\linmodel(x_t) = y_t$
    \end{enumerate}
Then $\robustness = \min_{\ptdata \in \mathfrak{P}} \sum_{i=1}^{m}\mathbb{I}(y_i' \neq y_i)$, where $\mathbb{I}$ is the indicator function. 

Note that some authors prefer to define robustness as the \emph{maximum} number of label perturbations in $\ctdata$ such that no classifier in $\hypo$ learned from the poisoned data classifies the test point $x_t$ as $y_t$.  Clearly, this is $\robustness - 1$, where $\robustness$ is as defined above. For computational simplicity, we use the definition provided above.

\section{Computing Robustness}
A naive way to compute robustness is by perturbing labels of subsets of the training dataset $\ctdata$ in increasing order of subset size, and finding whether there exists a classifier that satisfies all required conditions for the poisoned dataset. The time complexity of such an algorithm is clearly exponential in $|\ctdata|$. Hence, a natural question is whether this algorithm can be improved. Theorem~\ref{thm:npc} shows that this is unlikely.

\begin{theorem}
\label{thm:npc}
Given a dataset $\ctdata = \{(x_i, y_i) \mid i \in [m]\}$ and a target test point $(x_t, y_t)$, deciding whether its robustness $\robustness$ is less than a threshold $\kappa$ is NP-Complete, when $\hypo$ is a set of linear binary classifiers.
\end{theorem}
\begin{proof}[Proof Sketch]
    \textbf{In NP}: Given a witness classifier $\linmodel$, we can check whether (a) $\linmodel(x_t) = y_t$, and (b) $\sum_{i=1}^{m}\mathbb{I}(\linmodel(x_i) \neq y_i) < \kappa$ in polynomial time. Alternatively, given a witness poisoned dataset $\ptdata = \{(x_i, y_i') \mid i \in [m]\}$, we can check in polynomial time if (a) $\sum_{i=1}^{m}\mathbb{I}(y_i' \neq y_i) < \kappa$, and (b) there exists a linear classifier $\linmodel$ that satisfies $\linmodel(x_t) = y_t$ and $\linmodel(x_i) = y_i'$ for all $i \in [m]$.  The latter can be done by standard techniques for solving linear systems of equations, viz. Gaussian elimination.
    
    \textbf{NP-hardness}: We show this by  reduction from the vertex cover problem. Let $G = (V, E)$ be an undirected graph, where $V$ is the set of vertices, and $E \subseteq V \times V$ is the set of edges. A vertex cover of $G$ is a subset $C$ of $V$ such that for every edge $(u, v) \in V$, at least one of $u, v$ is in $C$. Given $G$ and a threshold $\kappa$, the vertex cover problem asks whether $G$ has a vertex cover $C$ s.t. $|C| < \kappa$.  It is well known~\cite{karp2009reducibility} that the vertex cover problem is NP-complete. We give below a polynomial-time reduction from the vertex cover problem to the problem of deciding if the robustness of a dataset w.r.t. a target test point is less than a threshold. This proves NP-hardness of deciding if robustness is less than a threshold.

    Let $V = \{v_1, \ldots v_n\}$ be the set of vertices in $G$. We create a training dataset $\ctdata \subseteq \features \times \labels$, where $\features = \{0, 1\}^{n+1}$ and $\labels = \{+1, -1\}$.  Thus, every $x \in \features$ can be thought of as a 0-1 vector of $n+1$ dimensions.  Below, we use $x[j]$ to denote the $j^{th}$ component of vector $x$, for $j \in \{1, \ldots n+1\}$.  
    
    For every $v_i \in V$, we add a training datapoint $(x_i, y_i)$ to $\ctdata$, where (a) $x_i[j] = 1$ iff $i = j$, and (b) $y_i = -1$.  We say that these datapoints encode vertices in $V$. Similarly, for every edge $(v_i, v_j) \in E$, we add a training datapoint $(x_{i,j}, y_{i,j})$, where (a) $x_{i,j}[k] = 1$ iff either $i = k$ or $j = k$ or $k = n+1$, and (b) $y_{i,j} = +1$. We say that these datapoints encode edges in $E$.  As an example, if $V = \{v_1, v_2, v_3\}$ and $E = \{(v_1, v_2), (v_2, v_3), (v_3, v_1)\}$, then $\ctdata = \ctdata^V \cup \ctdata^E$, where $\ctdata^V = \{\big((1,0,0,0),-1\big),$ $\big((0,1,0,0), -1\big),$ $\big((0,0,1,0), -1\big)\}$ encodes the three vertices, and $\ctdata^E = \{\big((1,1,0,1),+1\big),$ $\big((0,1,1,1), +1\big),$ $\big((1,0,1,1),+1\big)\}$ encodes the three edges.  Notice that for $\ctdata$ defined above, there exists a linear classifier $\linmodel$ such that $\linmodel(x_i) = y_i$ for all $(x_i, y_i) \in \ctdata$.  Indeed, if $w[j] = -1$ for $j \in \{1, \ldots n\}$, $w[n+1] = 3$, and if $b = 0$, then $\sgnfn(\weights^T x + b) = y$ for all $(x, y) \in \ctdata$.  
    
    Finally, we construct the test datapoint $(x_t, y_t)$, where $x_t[j] = 0$ for all $j \in \{1, \ldots n\}$, $x_t[n+1] = 1$ and $y_t = -1$.  For the abve example, $(x_t, y_t) = \big((0,0,0,1), -1\big)$.  Notice that the linear classifier discussed above that correctly classifies all datapoints in $\ctdata$ no longer gives $\linmodel(x_t) = y_t$.

    With $\ctdata$ and $(x_t, y_t)$ defined as above, we claim that the graph $G$ has a vertex cover of size less than $\kappa$ iff the robustness of $\ctdata$ w.r.t. $(x_t, y_t)$ is less than $\kappa$. The only if part is easy to prove. Suppose $C$ is a vertex cover of $G$, and $|C| < \kappa$.  We construct the poisoned dataset $\ptdata$ by changing the label (to $+1$) of only those datapoints in $\ctdata$ that encode vertices in $C$. A linear classifier $\linmodel(x) = \sgnfn(\weights^T x + b)$ for $\ptdata$ that also satisfies $\linmodel(x_t) = y_t$ can now be obtained as follows: $w[i] = 3$ for all $v_i \in C$, $w[i] = -1$ for all $v_i \not\in C$, $w[n+1] = -1$ and $b=0$.  In our running example with $3$ vertices, considering the vertex cover $C = \{v_1, v_2\}$, we get $\ptdata =$ $\{\big((1,0,0,0),{+1}\big),$ $\big((0,1,0,0), {+1}\big),$ $\big((0,0,1,0), -1\big),$ $\big((1,1,0,1),+1\big),$ $\big((0,1,1,1), +1\big),$ $\big((1,0,1,1),+1\big)\}$.

    To prove the if part, let $S$ be the subset of datapoints in $\ctdata$ whose labels are flipped to obtain $\ptdata$, and suppose $|S| < \kappa$.  Let $\linmodel$ be a linear classifier corresponding to $\ptdata$ that satisfies $\linmodel(x_t) = y_t$.  From the definition of $(x_t, y_t)$, it follows that $w[n+1] + b < 0$.  Note that for every non-poisoned datapoint $(x_{i,j}, +1)$ corresponding to an edge $(v_i, v_j) \in E$, we must also have $\linmodel(x_{i,j}) = 1$, or $w[i] + w[j] + w[n+1] + b \ge 0$.  This requires at least one of the datapoints corresponding to $v_i$ or $v_j$ to be poisoned, as otherwise, we would have $w[i] < 0$ and $w[j] < 0$, which is inconsistent with $w[i] + w[j] + w[n+1] + b \ge 0$. In general, the set $S$ may contain datapoints encoding both vertices and edges in $G$.  We describe below a process for successively transforming $S$, such that it eventually contains only datapoints encoding vertices in $V$.  Specifically, for every datapoint $(x_{i,j}, +1)$ in $S$ that encodes an edge $(v_i, v_j)$ in graph $G$, we check if the datapoint encoding $v_i$ (or $v_j$) is also in $S$.  If so, we simply remove $(x_{i,j}, +1)$ from $S$; otherwise, we add the datapoint corresponding to $v_i$ (resp. $v_j$) to $S$ and remove all datapoints corresponding to edges incident on $v_i$ (resp. $v_j$) from $S$. By repeating this process, we obtain a poisoned dataset $\ptdata'$ in which every
    poisoned datapoint corresponds to a vertex in $V$. Let the corresponding set of datapoints whose labels have been flipped in $\ptdata'$ be called $S'$. Following the same reasoning as in the proof of the "only if" part above, the linear classifier $\linmodel$ corresponding to $\ptdata$ can be modified to yield a linear classifier $\linmodel'$ for $\ptdata'$ that also satisfies $\linmodel'(x_t) = y_t$.  Since the datapoint corresponding to every edge $(v_i, v_j) \in E$ has its label unchanged in $\ptdata'$, the datapoint corresponding to either $v_i$ or $v_j$ must have been poisoned in $\ptdata'$. Hence, the set of poisoned datapoints must correspond to a subset of vertices that forms a vertex cover of $G$.  Since we added at most one datapoint corresponding to a vertex in $S$ when removing a datapoint corresponding to an edge from $S$, we have $|S'| \leq |S| < \kappa$.  
\end{proof}

Since perturbing subsets results in an exponential algorithm, it does not scale even for a small dataset. To overcome this, we propose to compute an approximation of robustness $\robustness$. Notice that the victim is interested in how low the robustness of their dataset is, to ensure that it is not susceptible to an attack, while the attacker's interest lies in how high the robustness is, so that the attack is likely to succeed without detection. Considering these perspectives, we compute both a lower bound $\lowrob$ and an upper bound $\uprob$ of robustness $\robustness$, such that it is guaranteed that $\lowrob \le \robustness \le \uprob$. 

\subsection{Lower Bound Robustness $\lowrob$ via Partitioning}
We observe that the problem of finding robustness for linear classifiers can be formulated as an optimization problem. In particular, it can be encoded as a mixed-integer linear program (MILP) using the big M method~\cite{bazaraa2011linear}, assuming a known bound on the range of the linear classifier function. A solution to this optimization problem determines the weights of the function that classifies the test point as required, while minimizing misclassifications on the training dataset. More formally, for a dataset $\ctdata$ and a test point $(x_t, y_t)$, we encode the problem as the following optimization problem:

\textbf{Variables:}
\begin{itemize}
    \item $\weight_1, \dots, \weight_{\dimension}, \bias \in \mathbb{R}$ as weights and bias, and 
    \item $\delta_1, \dots, \delta_{\tsize} \in \{0, 1\}$ as indicators for label perturbations.
\end{itemize}

\textbf{Objective:}
Minimize $\sum_{i=1}^{\tsize} \delta_i$

\textbf{Constraints:}
\begin{align*}
    y_t \left( \weights \cdot x_t + \bias \right) &\geq \epsilon \\
    y_i \left( \weights \cdot x_i + \bias \right) + \bigM \delta_i &\geq \epsilon, \quad \forall i \in [\tsize] \\
    y_i \left( \weights \cdot x_i + \bias \right) - \bigM (1 - \delta_i) &\leq -\epsilon, \quad \forall i \in [\tsize] \\
\end{align*}

\textbf{Parameters:}
$\bigM \gg 0$ is a large integer constant, $\epsilon \approx 0$, and $(x_i, y_i) \in \ctdata$. 

A solution to the problem provides weights and bias ($\weight_1, \dots, \weight_{\dimension}, \bias$) that classifies the test point $x_t$ as $y_t$ and minimizes the sum of perturbed labels in $\ctdata$, denoted by $\sum_{i=1}^{\tsize} \delta_i$. The first constraint ensures that $y_t = \sgnfn(\weights \cdot x_i + \bias)$. For each point $(x_i, y_i) \in \ctdata$, the next two constraints ensure the following:  
\[
    \delta_i = 
    \begin{cases}
        0 &  y_i = \sgnfn(\weights \cdot x_i + \bias) \\
        1 & y_i \ne \sgnfn(\weights \cdot x_i + \bias)
    \end{cases}
\]

For instance, $\delta_i = 0$ makes the second constraint $y_i \left(\weights \cdot x_i + \bias\right) \ge \epsilon$, which in turn makes $y_i = \sgnfn(\weights \cdot x_i + \bias)$. Similarly, $\delta_i = 1$ makes the third constraint $y_i \left(\weights \cdot x_i + \bias\right) \le \epsilon$, which makes $y_i \ne \sgnfn(\weights \cdot x_i + \bias)$. These constraints additionally adds the conditions $y_i \left(\weights \cdot x_i + \bias\right) < \bigM$ and $y_i \left(\weights \cdot x_i + \bias\right) > \nbigM$, resp..

Hence, a solution gives the robustness $\robustness$ (where $ \robustness = \sum_{i=1}^{\tsize} \delta_i$), along with the labels that were perturbed ($y_i$ for which $\delta_i$ is $1$). A solution exists only when $\weights \cdot x_i + \bias$ is in the interval $(\nbigM, \bigM)$. 

Since the complexity of solving such optimization problems is exponential, we can get robustness in this way for only small datasets. To scale this, we propose an optimization approach that finds a lower bound of robustness. In this approach, the dataset is partitioned into smaller size, then robustness is computed individually for each partition, and finally the results are summed. Although this sum provides a lower bound, it significantly improves the scalability as the optimization problem is on small sized sets. 

Consider the partition of $\ctdata$ into $k \in \mathbb{N}$ disjoint subsets $\ctdata^1 \dots \ctdata^k$, where each subset $\ctdata^j$ contains $\tsize/k$ points, except for the last subset, which will contain the remaining points. 

For each partition $\ctdata^j$ we compute robustness $\robustness^j$ by encoding it as the optimization problem(as before). Finally, we compute the sum of these robustness $\lowrob = \sum_{j=1}^k \robustness^j$. 

\begin{theorem}
    For dataset $\ctdata$ and a test point $(x_t, y_t)$, $\lowrob \le \robustness$.\footnote{Remaining proofs are in Appendix.}
\end{theorem}

\subsection{Upper Bound Robustness $\uprob$ via Augmentation}

In addition to the lower bound, we also compute an upper bound of robustness. For this purpose, we train a linear classifier on an augmented version of the training dataset $\ctdata$. The augmentation biases the learning to train a classifier that classifies the test point $x_t$ as $y_t$. While theoretically this classifier may not minimize the label perturbations of $\ctdata$, in practice we observe that it will give a tighter upper bound. Moreover, this procedure will be quick, as it requires only learning a classifier, hence scales for large datasets. 

In order to find such a classifier, we restrict our objective to learn a classifier $\linmodel \in \hypo$ such that $\linmodel(x_t) = y_t$. Since not all classifiers can achieve this objective, we first introduce a targeted augmentation scheme, wherein $k' \in \mathbb{N}$ identical copies of the test point $(x_t, y_t)$ are added to $\ctdata$. Let the augmented set be $\ctdata' = \ctdata \bigcup \{(x_t,y_t)\}^{k'}$, where $\{(x_t,y_t)\}^{k'}$ is a multi-set of $k'$ copies of $(x_t, y_t)$. 

We then train a classifier on $\ctdata'$ by minimizing the empirical loss with respect to a loss function $\lossfn$:
\begin{align*}
\linmodel = \arg \min_{\linmodel \in \hypo} \frac{1}{\tsize + k'} \left(
\begin{aligned}
&\sum_{i=1}^\tsize \lossfn(\linmodel(x_i), y_i) + \\
& \sum_{j=1}^{k'} \lossfn(\linmodel(x_t), y_t)
\end{aligned}
\right)
\end{align*}

Once a classifier $\linmodel$ is learned on the augmented dataset, it is checked whether the classifier correctly classifies the test point $x_t$ as $y_t$. If it does, then the number of misclassified points from the original dataset $\ctdata$ will be the the upper bound. In other words, $\uprob = \sum_{i=1}^\tsize \mathbb{I}(\linmodel(x_i) \ne y_i)$. 

\begin{theorem}
    For dataset $\ctdata$ and a test point $(x_t, y_t)$, $\robustness \le \uprob$.
\end{theorem}

\subsection{Algorithm}

\begin{algorithm}[tb]
\caption{\textsc{RobustnessInterval}($\ctdata$, $(x_t, y_t)$)}
\label{alg:main}
\textbf{Input}: $\ctdata$ -- training dataset, $(x_t, y_t)$ -- test point \\
\textbf{Parameter}: $\bigM, k, k', \lossfn, \epsilon$ -- hyperparameters\\
\textbf{Output}: lower bound $\lowrob$ and upper bound $\uprob$ robustness
\begin{algorithmic}[1] 
\STATE $\ctdata^1, \dots, \ctdata^k \gets \textsc{RandomPartition}(\ctdata, k)$
\FOR{$j=1$ \textbf{to} $k$} 
\STATE $\robustness^j \gets \textsc{MILPSolve}(\ctdata^j, (x_t, y_t), \bigM)$
\ENDFOR
\STATE $\lowrob \gets \sum_{j=1}^k \robustness^j$
\STATE $\ctdata' \gets \textsc{Augment}(\ctdata, (x_t, y_t), k')$
\STATE $\linmodel \gets \textsc{LearnClassifier}(\ctdata', \lossfn)$ 
\STATE $\uprob \gets \sum_{i=1}^m \mathbb{I}(\linmodel(x_i) \ne y_i)$
\STATE \textbf{return} $(\lowrob, \uprob)$
\end{algorithmic}
\end{algorithm}

Algorithm~\ref{alg:main} gives a formal description of our technique discussed in previous sections. It takes as input a training dataset $\ctdata$ and a test point $x_t$ with target label $y_t$. It also assumes hyperparameter $\bigM, k, k', \lossfn$ are set. It starts with computing random $k$ partitions of $\ctdata$ into $\ctdata^1, \dots, \ctdata^k$. Then, for each partition $\ctdata^j$ it computes its robustness $\robustness^j$ by encoding the problem as a MILP using $\bigM$ and then using a solver to find the optimal solution. These robustness are then summed up to get $\lowrob$. In order to computer $\uprob$, the algorithm starts with augmenting $\ctdata$ with $k'$ copies of the test point and target label. This augmented data $\ctdata$ is then passed to a learning algorithm along with a loss function ($\lossfn$). When the learning algorithm returns a classifier $\linmodel$, $\uprob$ is computed by calculating the number of misclassifications $\linmodel$ makes with respect to $\ctdata$. Finally, the pair $(\lowrob, \uprob)$ is returned to the user.

\section{Implementation Details}
We implemented Algorithm~\ref{alg:main} as a Python tool \robint. It uses SCIP~\cite{BolusaniEtal2024OO} as the MILP solver (within Google OR-Tools v9.12)  to compute lower bound robustness, and \texttt{SGDClassifier} from scikit-learn~\cite{scikit-learn} (v1.3.2) to learn a classifier from the augmented dataset. As an optimization, 10 classifiers were learned from the augmented dataset and the minimum upper bound robustness $\uprob$ among them was chosen. The hyperparameters used had the following values: $\bigM$ and $\epsilon$ in the MILP encoding are set to 1000 and $10^{-10}$, resp., the number of partitions $k$ had multiple values such as 20, 1000, 250, and 100 depending on the dataset, the number of augmented test point $k'$ was set to $\tsize + 1$, and the loss function $\lossfn$ used were hinge loss, log loss, and modified hueber.

\section{Evaluation}

\begin{table*}[htbp]  
\centering  
\begingroup  
\fontsize{9}{9}\selectfont  
\begin{tabular}{clrrrrrrrrrr}  
\toprule  
\# & \textbf{Dataset} & \textbf{$\dimension$} & $\mathbf{\tsize}$ & $\mathbf{\#x_t}$ & $\mathbf{\uprob}$ & $\mathbf{\uprob}_{\ipr}$(normalized) & $\mathbf{\uprob}_{\ipr}$ & \%Found $\mathbf{\uprob}_{\ipr}$ & $\rho$ & $\rho_{\ipr}$ & $\mathbf{\lowrob}$ \\  
\midrule  
1 & Letter Recognition & 16 & 1335 & 149 & \textbf{38.16} & 569.62 & 168.18 & 66 & \textbf{0.55} & 0.34 & 0.68 \\  
2 & Digits Recognition & 16 & 1404 & 156 & \textbf{10.79} & 1305.88 & 226.54 & 8 & \textbf{0.17} & 0.04 & 0.03 \\  
3 & SST (BOW) & 300 & 6920 & 872 & 1758.08 & \textbf{93.76} & 66.99 & 100 & \textbf{0.99} & 0.61 & 0.96 \\  
4 & SST (BERT) & 768 & 6920 & 872 & 1141.54 & \textbf{439.64} & 247.90 & 97 & \textbf{0.97} & 0.52 & 0.00 \\  
5 & Emo (BOW) & 300 & 9025 & 1003 & 1558.92 & \textbf{237.95} & 167.28 & 99 & \textbf{0.99} & 0.52 & 0.00 \\  
6 & Speech (BOW) & 300 & 9632 & 1071 & 1071.90 & \textbf{624.09} & 373.20 & 97 & \textbf{0.61} & 0.58 & 0.49 \\  
7 & Speech (BERT) & 768 & 9632 & 1071 & \textbf{891.28} & 931.04 & 593.20 & 96 & \textbf{0.54} & 0.49 & 0.10 \\  
8 & Fashion-MNIST & 784 & 10800 & 1200 & 1389.91 & \textbf{1037.25} & 274.13 & 93 & \textbf{0.98} & 0.45 & 0.01 \\  
9 & Loan (BOW) & 18 & 11200 & 2800 & \textbf{3548.16} & 10069.39 & 42.22 & 60 & \textbf{0.93} & 0.46 & 111.21 \\  
10 & Emo (BERT) & 768 & 11678 & 1298 & \textbf{2153.26} & 2913.84 & 232.23 & 77 & \textbf{0.97} & 0.40 & 560.45 \\  
11 & Essays (BOW) & 300 & 11678 & 1298 & \textbf{495.90} & 2884.23 & 1664.89 & 88 & 0.21 & \textbf{0.50} & 0.00 \\  
12 & Essays (BERT) & 768 & 11678 & 1298 & \textbf{350.85} & 3761.81 & 1245.04 & 76 & 0.22 & \textbf{0.39} & 0.00 \\  
13 & Tweet (BOW) & 300 & 18000 & 1000 & 4256.26 & \textbf{304.15} & 216.32 & 99 & \textbf{1.00} & 0.54 & 7.70 \\  
14 & Tweet (BERT) & 768 & 18000 & 1000 & 3795.24 & \textbf{352.38} & 343.72 & 100 & \textbf{0.99} & 0.55 & 0.32 \\  
15 & Census Income & 41 & 80136 & 8905 & \textbf{2542.06} & 66003.10 & 7012.80 & 19 & \textbf{0.33} & 0.07 & 87.19 \\  
\bottomrule  
\end{tabular}  
\endgroup  
\caption{Here, $\dimension$, $\tsize$ -- dimension and size of datasets, $\mathbf{\#x_t}$-- count of test points, $\mathbf{\uprob}$, $\mathbf{\lowrob}$ -- avg upper and lower bound robustness from \robint, $\mathbf{\uprob}_{\ipr}$(normalized), $\mathbf{\uprob}_{\ipr}$ -- avg upper bound robustness from \iprtool, \%Found $\mathbf{\uprob}_{\ipr}$-- \% of robustness found by \iprtool, and $\rho$, $\rho_{\ipr}$-- avg likelihood of getting desired classification by \robint and \iprtool.} 
\label{tab:robustness_metrics}  
\end{table*}

In this section, we present evaluation of our tool \robint on several publicly available datasets with different sizes. We first describe the datasets used. 

\subsection{Datasets}
We use datasets that represent different classification tasks where linear classifiers are used. Specifically, we use: Census Income~\cite{census_income_20}, Fashion-MNIST~\cite{xiao2017fashionmnistnovelimagedataset}, Pen-Based Handwritten Digits Recognition~\cite{pen-based_recognition_of_handwritten_digits_81}, Letter Recognition~\cite{letter_recognition_59}, Emo~\cite{yang2023relabeling}, Speech~\cite{de-gibert-etal-2018-hate}, SST~\cite{socher-etal-2013-recursive}, and Tweet~\cite{article}, as well as the Essays Scoring~\cite{EssaysDataset} and the Loan Prediction~\cite{LoanDataset} dataset from Kaggle. For the text datasets, we use both bag-of-words (BOW) and BERT-based (BERT) representations where indicated. Multi-class classification datasets were converted to binary classification datasets by only two classes. This was done for Fashion-MNIST(Pullover vs. Shirt), Pen-Based Handwritten Digits Recognition(4 vs. 0) and Letter Recognition(D vs. O).  


\robint was used to compute both robustness bounds for all test points in each dataset. For datasets without predefined test-train splits (e.g., Census Income, Fashion MNIST, Letter Recognition, Pen-Based Handwritten Digits Recognition), 10\% of points were randomly sampled as test data, with the remainder constituted the train data. Dataset summaries are in Table~\ref{tab:robustness_metrics}.

In order to evaluate \robint, we address the following research questions (RQs):
\begin{itemize}
    \item[] RQ1: How effective is our technique in finding upper bound robustness $\uprob$?
    \item[] RQ2: How does upper bound robustness $\uprob$ impact victim's training process?
    \item[] RQ3: How does our technique for finding $\uprob$ compare against SOTA?
    \item[] RQ4: How effective is our technique in finding lower bound robustness $\lowrob$?
\end{itemize}

In the remaining section, we answer our research questions through a series of experiments. All experiments were conducted on an 8-core CPU with 30~GB RAM running Ubuntu 20.04. 


\subsection{RQ1: How effective is our technique in finding upper bound robustness $\uprob$?} 
\begin{figure}[h!]  
    \centering  
    \includegraphics[width=0.45\textwidth]{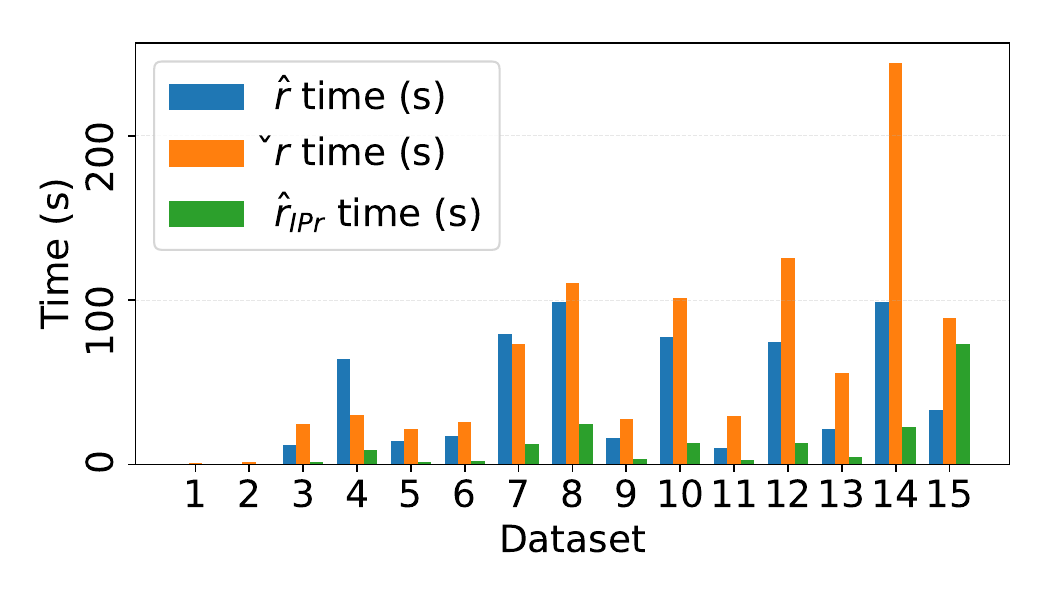}  
    \caption{  
        Average time taken in seconds per test point by \robint ($\uprob$, $\lowrob$) and \iprtool ($\uprob_{\ipr}$). 
    }  
    \label{fig:time}  
\end{figure}    
\begin{figure}[h!]  
    \centering  
        {\includegraphics[width=\linewidth]{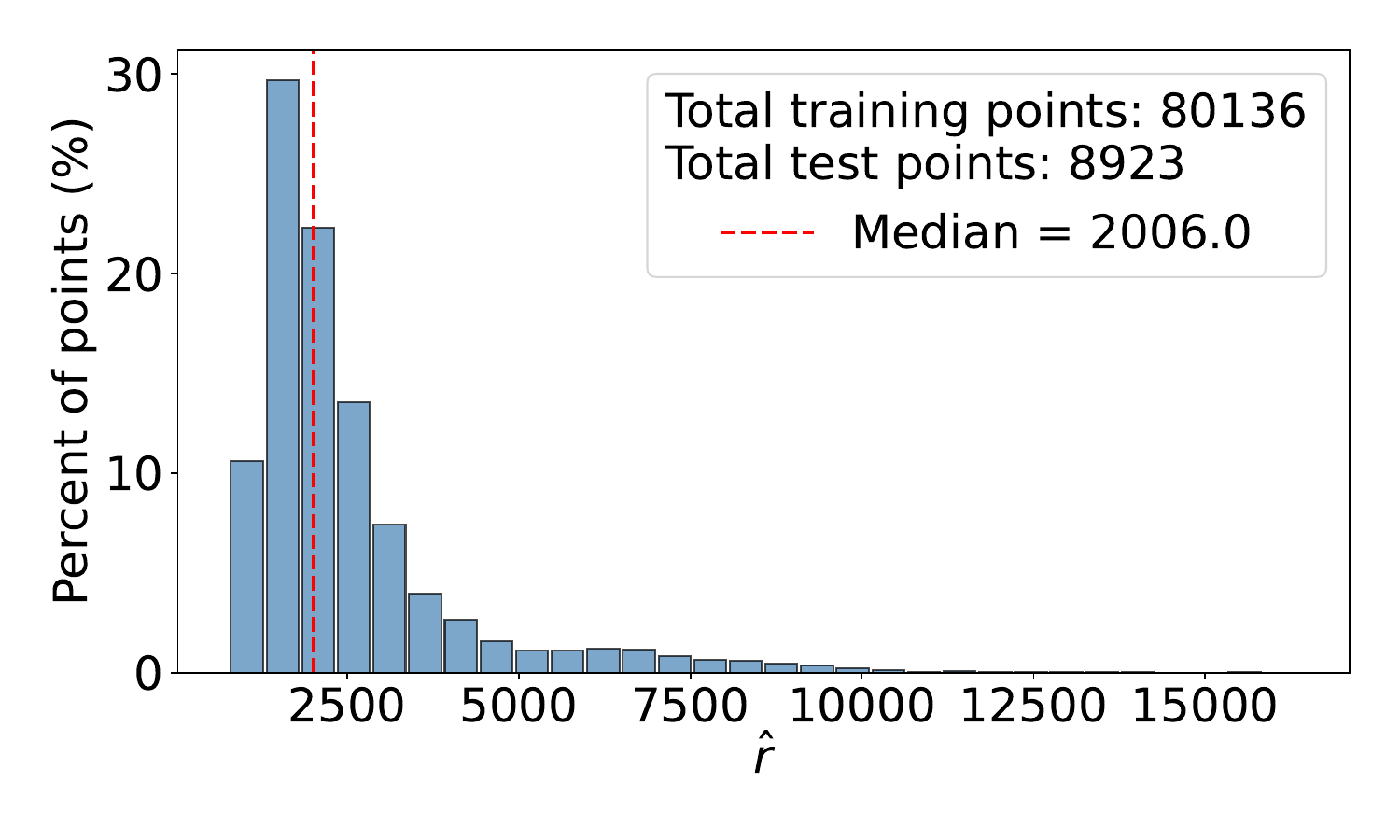}}  
    \caption{  
        Histograms of the distribution of points across $\uprob$ values for the Census Income dataset.    
    }  
    \label{fig:census-income-histograms}  
\end{figure}  

Table~\ref{tab:robustness_metrics} reports the average upper bound robustness $\uprob$ (column $\uprob$) computed using the hinge loss function (denoted as $\lossfn$ in Algorithm~\ref{alg:main}), while Figure~\ref{fig:time} shows the average computation time. The average $\uprob$ value ranged from 1\% (Digits Recognition) to 31\% (Loan (BOW)) of the training points. For five datasets (specifically, Census Income, Essays BOW and BERT,  and Letter and Digits Recognition), \robint identified $\uprob$ values as low as 4\% (or even less) of the training points. The average time required to compute $\uprob$ for a test point was under two minutes.


Figure~\ref{fig:census-income-histograms} shows the $\uprob$ histogram for the Census Income dataset, where nearly 60\% of test points have robustness values around 3\% of the training points. Similar results for other datasets are provided in the Appendix.

Overall, \robint was able to compute low $\uprob$ values for most datasets within a short duration.

\subsection{RQ2: How does  upper bound robustness $\uprob$ impact victim's training process?}
\begin{figure*}[t]  
    \centering  
    \includegraphics[width=0.95\textwidth]{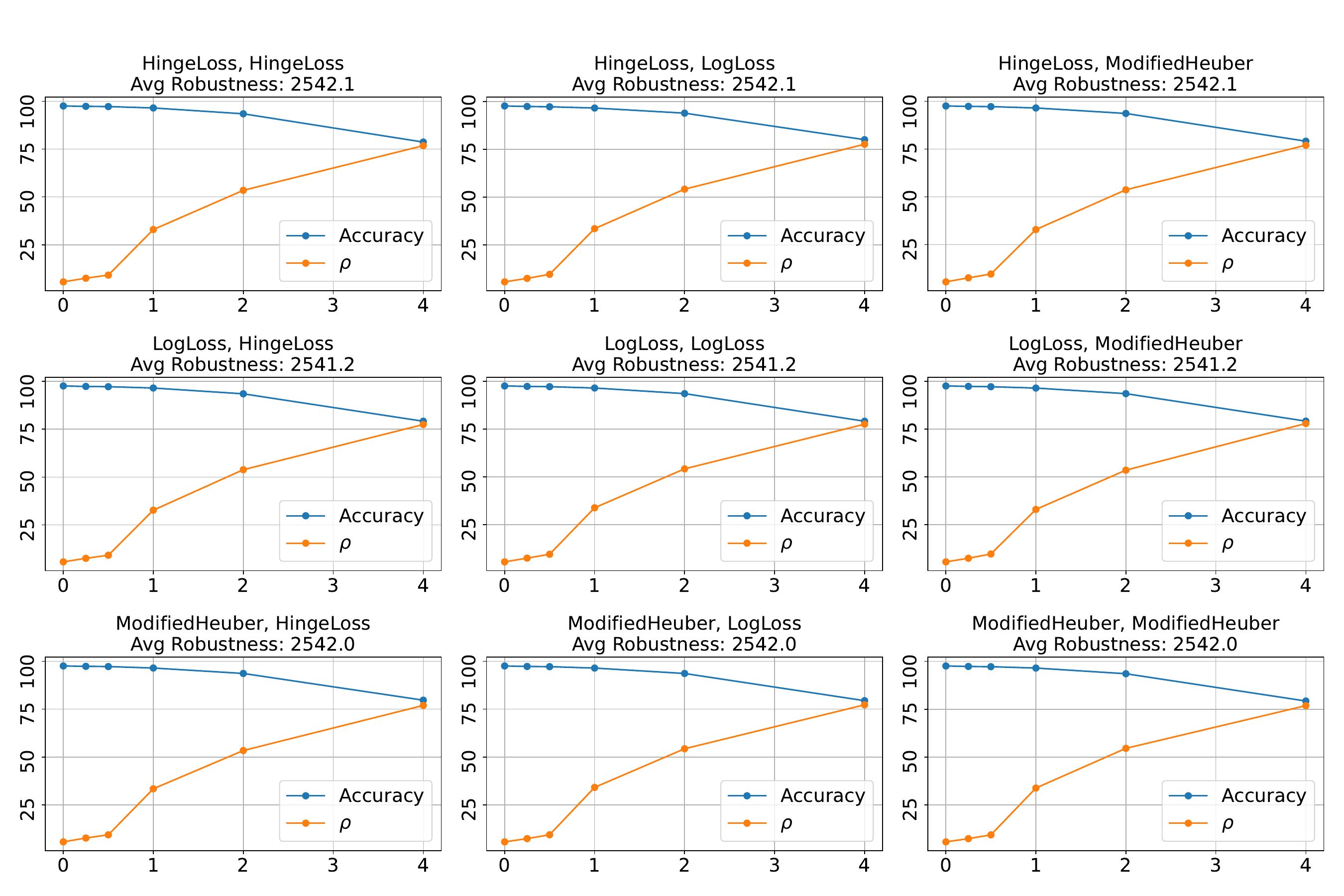}  
    \caption{  
    Comparison of average accuracy vs $\rho$  for Census Income dataset with $\{0, \frac{\uprob}{4}, \frac{\uprob}{2}, \uprob, 2\uprob, 4\uprob\}$ and loss functions.
    }  
    \label{fig:classifier-grid} 
\end{figure*}

To assess the impact of $\uprob$ on the victim's training process, we poisoned the training dataset by perturbing the labels of the training points corresponding to $\uprob$. We then trained a classifier on the poisoned dataset and compared its accuracy with the frequency of desired classification for test points (denoted by $\rho)$. Since we assume the victim's training process is a black-box, we compared all possible pairs of loss functions.  

More specifically, we assumed a loss function ($\lossfn$) for \robint and computed $\uprob$ for each test point. We then created poisoned datasets by perturbing the labels of $0$, $\frac{\uprob}{4}$, $\frac{\uprob}{2}$, $\uprob$, $2\uprob$, and $4\uprob$ points (for $2\uprob, 4\uprob$, additional random points were chosen). Each poisoned dataset was used to train a linear classifier with a loss function. We calculated the classifier's accuracy and $\rho$. This process was repeated for all possible pairs of loss functions, three to compute $\uprob$ and three for training the classifier. 


Figure~\ref{fig:classifier-grid} shows results for the Census Income dataset: the optimal outcome for all loss functions is when the poisoned value is $\uprob,$ as higher values sharply reduce accuracy without increasing flip likelihood. Similar trends appear in other datasets (see Appendix), confirming that $\uprob$ computed from \robint is an effective poisoning measure.

\subsection{RQ3: How does our technique for finding $\uprob$ compare against SOTA?}
We compare our tool against the \iprtool technique~\cite{yang2023relabeling}. While \iprtool assumes a white-box victim's model and computes only an upper bound robustness, we include it in our comparison as it is the most closely related recent work. The average upper bound robustness computed by \iprtool is presented in Table~\ref{tab:robustness_metrics} (column $\uprob_{\ipr}$). Notably, \iprtool does not guarantee an upper bound robustness for all points. In our evaluation, it was able to compute robustness for all points in only 2/15 datasets (see column \%Found $\mathbf{\uprob}_{\ipr}$). In contrast, \robint was able to find robustness for all test points across all datasets. For a fair comparison, when \iprtool failed to generate an upper bound, we assigned the size of the training dataset as the bound (column $\mathbf{\uprob}_{\ipr}$(normalized)). Under this metric, \robint found lower average robustness for 8 datasets. 


Additionally, we poison the dataset using each tool’s robustness and compare the success rate of desired test classifications (columns $\rho$ and $\rho_{\ipr}$). In this evaluation, \robint outperforms \iprtool on 12 datasets, including 6 where its average robustness is lower.

In summary, our tool assumes a realistic black-box adversary, provides tighter robustness bounds, and more reliably achieves the desired test point classification.

\subsection{RQ4: How effective is our technique in finding lower bound robustness $\lowrob$?}
\begin{figure}[h!]  
    \centering  
        {\includegraphics[width=\linewidth]{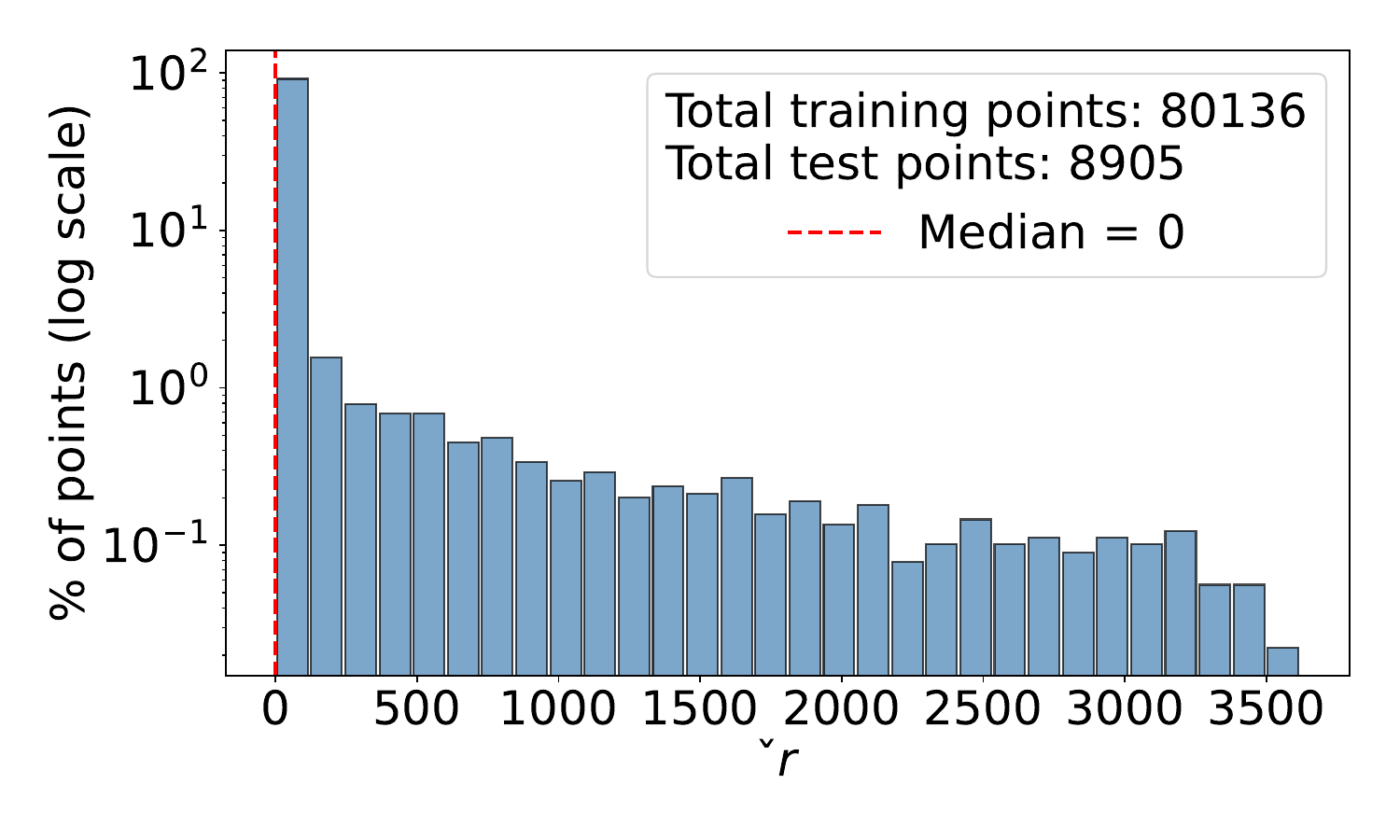}}  
    \caption{  
        Histograms of the distribution of points across $\lowrob$ values for the Census Income dataset.    
    }  
    \label{fig:census-income-histograms2}  
\end{figure}  

The average lower bound $\lowrob$ computed by \robint is presented in Table~\ref{tab:robustness_metrics} (column $\lowrob$), along with the average time taken per test point in Figure~\ref{fig:time}. \robint generated a non-zero average $\lowrob$ for ten datasets. The average time taken per test point was under four minutes. 

While the average lower bound robustness can be low for some datasets, our tool is capable of generating high low bounds. For example, consider the histograms in Figure~\ref{fig:census-income-histograms2} for Census Income dataset. Although the median robustness was 0, \robint was able to generate robustness values greater than 500 for 5\% of the test points, and non-zero robustness for 15\% of the points. Therefore, the lower bound robustness generated by our tool can still be useful for individual test points.

\section{Conclusion and Future Work}
In this paper, we have proven that the problem of finding robustness in targeted data poisoning is NP-complete in the setting considered and have introduced effective techniques for computing lower and upper bound of robustness. An interesting future direction is to extend the lower bound algorithm for non-linear classifiers.

\bibliography{references}

@article{biggio2012poisoning,
  title={Poisoning attacks against support vector machines},
  author={Biggio, Battista and Nelson, Blaine and Laskov, Pavel},
  journal={arXiv preprint arXiv:1206.6389},
  year={2012}
}

@inproceedings{biggio2011support,
  title={Support vector machines under adversarial label noise},
  author={Biggio, Battista and Nelson, Blaine and Laskov, Pavel},
  booktitle={Asian conference on machine learning},
  pages={97--112},
  year={2011},
  organization={PMLR}
}

@inproceedings{kumar2020adversarial,
  title={Adversarial machine learning-industry perspectives},
  author={Kumar, Ram Shankar Siva and Nystr{\"o}m, Magnus and Lambert, John and Marshall, Andrew and Goertzel, Mario and Comissoneru, Andi and Swann, Matt and Xia, Sharon},
  booktitle={2020 IEEE security and privacy workshops (SPW)},
  pages={69--75},
  year={2020},
  organization={IEEE}
}

@article{shafahi2018poison,
  title={Poison frogs! targeted clean-label poisoning attacks on neural networks},
  author={Shafahi, Ali and Huang, W Ronny and Najibi, Mahyar and Suciu, Octavian and Studer, Christoph and Dumitras, Tudor and Goldstein, Tom},
  journal={Advances in neural information processing systems},
  volume={31},
  year={2018}
}

@inproceedings{suciu2018does,
  title={When does machine learning $\{$FAIL$\}$? generalized transferability for evasion and poisoning attacks},
  author={Suciu, Octavian and Marginean, Radu and Kaya, Yigitcan and Daume III, Hal and Dumitras, Tudor},
  booktitle={27th USENIX Security Symposium (USENIX Security 18)},
  pages={1299--1316},
  year={2018}
}

@inproceedings{zhu2019transferable,
  title={Transferable clean-label poisoning attacks on deep neural nets},
  author={Zhu, Chen and Huang, W Ronny and Li, Hengduo and Taylor, Gavin and Studer, Christoph and Goldstein, Tom},
  booktitle={International conference on machine learning},
  pages={7614--7623},
  year={2019},
  organization={PMLR}
}

@inproceedings{koh2017understanding,
  title={Understanding black-box predictions via influence functions},
  author={Koh, Pang Wei and Liang, Percy},
  booktitle={International conference on machine learning},
  pages={1885--1894},
  year={2017},
  organization={PMLR}
}

@article{wang2022lethal,
  title={Lethal dose conjecture on data poisoning},
  author={Wang, Wenxiao and Levine, Alexander and Feizi, Soheil},
  journal={Advances in Neural Information Processing Systems},
  volume={35},
  pages={1776--1789},
  year={2022}
}

@article{tian2022comprehensive,
  title={A comprehensive survey on poisoning attacks and countermeasures in machine learning},
  author={Tian, Zhiyi and Cui, Lei and Liang, Jie and Yu, Shui},
  journal={ACM Computing Surveys},
  volume={55},
  number={8},
  pages={1--35},
  year={2022},
  publisher={ACM New York, NY}
}

@article{suya2024distributions,
  title={What Distributions are Robust to Indiscriminate Poisoning Attacks for Linear Learners?},
  author={Suya, Fnu and Zhang, Xiao and Tian, Yuan and Evans, David},
  journal={Advances in neural information processing systems},
  volume={36},
  year={2024}
}

@incollection{xiao2012adversarial,
  title={Adversarial label flips attack on support vector machines},
  author={Xiao, Han and Xiao, Huang and Eckert, Claudia},
  booktitle={ECAI 2012},
  pages={870--875},
  year={2012},
  publisher={IOS Press}
}

@inproceedings{zhao2017efficient,
  title={Efficient label contamination attacks against black-box learning models.},
  author={Zhao, Mengchen and An, Bo and Gao, Wei and Zhang, Teng},
  booktitle={IJCAI},
  pages={3945--3951},
  year={2017}
}

@inproceedings{paudice2019label,
  title={Label sanitization against label flipping poisoning attacks},
  author={Paudice, Andrea and Mu{\~n}oz-Gonz{\'a}lez, Luis and Lupu, Emil C},
  booktitle={ECML PKDD 2018 Workshops: Nemesis 2018, UrbReas 2018, SoGood 2018, IWAISe 2018, and Green Data Mining 2018, Dublin, Ireland, September 10-14, 2018, Proceedings 18},
  pages={5--15},
  year={2019},
  organization={Springer}
}

@article{barreno2010security,
  title={The security of machine learning},
  author={Barreno, Marco and Nelson, Blaine and Joseph, Anthony D and Tygar, J Doug},
  journal={Machine learning},
  volume={81},
  pages={121--148},
  year={2010},
  publisher={Springer}
}

@book{bazaraa2011linear,
  title={Linear programming and network flows},
  author={Bazaraa, Mokhtar S and Jarvis, John J and Sherali, Hanif D},
  year={2011},
  publisher={John Wiley \& Sons}
}

@misc{census_income_20,  
  author       = {Kohavi, Ron},  
  title        = {{Census Income}},  
  year         = {1996},  
  howpublished = {UCI Machine Learning Repository},  
  note         = {{DOI}: https://doi.org/10.24432/C5GP7S}  
}

@misc{xiao2017fashionmnistnovelimagedataset,  
  title        = {Fashion-MNIST: a Novel Image Dataset for Benchmarking Machine Learning Algorithms},  
  author       = {Han Xiao and Kashif Rasul and Roland Vollgraf},  
  year         = {2017},  
  eprint       = {1708.07747},  
  archivePrefix= {arXiv},  
  primaryClass = {cs.LG},  
  url          = {https://arxiv.org/abs/1708.07747}  
}

@misc{pen-based_recognition_of_handwritten_digits_81,  
  author       = {Alpaydin, E. and Alimoglu, Fevzi.},  
  title        = {{Pen-Based Recognition of Handwritten Digits}},  
  year         = {1996},  
  howpublished = {UCI Machine Learning Repository},  
  note         = {{DOI}: https://doi.org/10.24432/C5MG6K}  
}

@misc{letter_recognition_59,  
  author       = {Slate, David},  
  title        = {{Letter Recognition}},  
  year         = {1991},  
  howpublished = {UCI Machine Learning Repository},  
  note         = {{DOI}: https://doi.org/10.24432/C5ZP40}  
}

@inproceedings{de-gibert-etal-2018-hate,  
  title     = {Hate Speech Dataset from a White Supremacy Forum},  
  author    = {de Gibert, Ona and Perez, Naiara and Garc{\'i}a-Pablos, Aitor and Cuadros, Montse},  
  editor    = {Fi{\v{s}}er, Darja and Huang, Ruihong and Prabhakaran, Vinodkumar and Voigt, Rob and Waseem, Zeerak and Wernimont, Jacqueline},  
  booktitle = {Proceedings of the 2nd Workshop on Abusive Language Online ({ALW}2)},  
  month     = oct,  
  year      = {2018},  
  address   = {Brussels, Belgium},  
  publisher = {Association for Computational Linguistics},  
  url       = {https://aclanthology.org/W18-5102/},  
  doi       = {10.18653/v1/W18-5102},  
  pages     = {11--20}  
}

@inproceedings{socher-etal-2013-recursive,  
  title     = {Recursive Deep Models for Semantic Compositionality Over a Sentiment Treebank},  
  author    = {Socher, Richard and Perelygin, Alex and Wu, Jean and Chuang, Jason and Manning, Christopher D. and Ng, Andrew and Potts, Christopher},  
  editor    = {Yarowsky, David and Baldwin, Timothy and Korhonen, Anna and Livescu, Karen and Bethard, Steven},  
  booktitle = {Proceedings of the 2013 Conference on Empirical Methods in Natural Language Processing},  
  month     = oct,  
  year      = {2013},  
  address   = {Seattle, Washington, USA},  
  publisher = {Association for Computational Linguistics},  
  url       = {https://aclanthology.org/D13-1170/},  
  pages     = {1631--1642}  
}

@article{article,  
  author  = {Go, Alec and Bhayani, Richa and Huang, Lei},  
  year    = {2009},  
  month   = {01},  
  pages   = {},  
  title   = {Twitter sentiment classification using distant supervision},  
  volume  = {150},  
  journal = {Processing}  
}

@techreport{BolusaniEtal2024OO,
  author = {Suresh Bolusani and Mathieu Besan{\c{c}}on and Ksenia Bestuzheva and Antonia Chmiela and Jo{\~{a}}o Dion{\'{i}}sio and Tim Donkiewicz and Jasper van Doornmalen and Leon Eifler and Mohammed Ghannam and Ambros Gleixner and Christoph Graczyk and Katrin Halbig and Ivo Hedtke and Alexander Hoen and Christopher Hojny and Rolf van der Hulst and Dominik Kamp and Thorsten Koch and Kevin Kofler and Jurgen Lentz and Julian Manns and Gioni Mexi and Erik~M\"{u}hmer and Marc E. Pfetsch and Franziska Schl{\"o}sser and Felipe Serrano and Yuji Shinano and Mark Turner and Stefan Vigerske and Dieter Weninger and Lixing Xu},
  title = {{The SCIP Optimization Suite 9.0}},
  type = {Technical Report},
  institution = {Optimization Online},
  month = {February},
  year = {2024},
  url = {https://optimization-online.org/2024/02/the-scip-optimization-suite-9-0/}
}

@article{yang2023relabeling,
  title={Relabeling minimal training subset to flip a prediction},
  author={Yang, Jinghan and Xu, Linjie and Yu, Lequan},
  journal={arXiv preprint arXiv:2305.12809},
  year={2023}
}

@misc{EssaysDataset,
    author = {Ben Hamner and Jaison Morgan and lynnvandev and Mark Shermis and Tom Vander Ark},
    title = {The Hewlett Foundation: Automated Essay Scoring},
    year = {2012},
    howpublished = {\url{https://kaggle.com/competitions/asap-aes}},
    note = {Kaggle}
}

@misc{LoanDataset,
  author = {Subham Surana},
  title = {Loan Prediction based on Customer Behavior},
  howpublished = {\url{https://www.kaggle.com/datasets/subhamjain/loan-prediction-based-on-customer-behavior}},
  note = {Accessed: 2025-07-31}
}

@article{adebayo2023quantifying,
  title={Quantifying and mitigating the impact of label errors on model disparity metrics},
  author={Adebayo, Julius and Hall, Melissa and Yu, Bowen and Chern, Bobbie},
  journal={arXiv preprint arXiv:2310.02533},
  year={2023}
}

@article{xiao2015support,
  title={Support vector machines under adversarial label contamination},
  author={Xiao, Huang and Biggio, Battista and Nelson, Blaine and Xiao, Han and Eckert, Claudia and Roli, Fabio},
  journal={Neurocomputing},
  volume={160},
  pages={53--62},
  year={2015},
  publisher={Elsevier}
}

@article{chen2017targeted,
  title={Targeted backdoor attacks on deep learning systems using data poisoning},
  author={Chen, Xinyun and Liu, Chang and Li, Bo and Lu, Kimberly and Song, Dawn},
  journal={arXiv preprint arXiv:1712.05526},
  year={2017}
}

@inproceedings{saha2020hidden,
  title={Hidden trigger backdoor attacks},
  author={Saha, Aniruddha and Subramanya, Akshayvarun and Pirsiavash, Hamed},
  booktitle={Proceedings of the AAAI conference on artificial intelligence},
  volume={34},
  number={07},
  pages={11957--11965},
  year={2020}
}

@article{levine2020deep,
  title={Deep partition aggregation: Provable defense against general poisoning attacks},
  author={Levine, Alexander and Feizi, Soheil},
  journal={arXiv preprint arXiv:2006.14768},
  year={2020}
}

@inproceedings{wang2022improved,
  title={Improved certified defenses against data poisoning with (deterministic) finite aggregation},
  author={Wang, Wenxiao and Levine, Alexander J and Feizi, Soheil},
  booktitle={International Conference on Machine Learning},
  pages={22769--22783},
  year={2022},
  organization={PMLR}
}

@inproceedings{rosenfeld2020certified,
  title={Certified robustness to label-flipping attacks via randomized smoothing},
  author={Rosenfeld, Elan and Winston, Ezra and Ravikumar, Pradeep and Kolter, Zico},
  booktitle={International Conference on Machine Learning},
  pages={8230--8241},
  year={2020},
  organization={PMLR}
}

@inproceedings{cina2021hammer,
  title={The hammer and the nut: Is bilevel optimization really needed to poison linear classifiers?},
  author={Cin{\`a}, Antonio Emanuele and Vascon, Sebastiano and Demontis, Ambra and Biggio, Battista and Roli, Fabio and Pelillo, Marcello},
  booktitle={2021 International Joint Conference on Neural Networks (IJCNN)},
  pages={1--8},
  year={2021},
  organization={IEEE}
}

@article{csuvak2022design,
  title={Design of poisoning attacks on linear regression using bilevel optimization},
  author={{\c{S}}uvak, Zeynep and Anjos, Miguel F and Brotcorne, Luce and Cattaruzza, Diego},
  year={2022}
}

@article{yang2023many,
  title={How many and which training points would need to be removed to flip this prediction?},
  author={Yang, Jinghan and Jain, Sarthak and Wallace, Byron C},
  journal={arXiv preprint arXiv:2302.02169},
  year={2023}
}

@article{gu2019badnets,
  title={Badnets: Evaluating backdooring attacks on deep neural networks},
  author={Gu, Tianyu and Liu, Kang and Dolan-Gavitt, Brendan and Garg, Siddharth},
  journal={Ieee Access},
  volume={7},
  pages={47230--47244},
  year={2019},
  publisher={IEEE}
}

@inproceedings{jia2022certified,
  title={Certified robustness of nearest neighbors against data poisoning and backdoor attacks},
  author={Jia, Jinyuan and Liu, Yupei and Cao, Xiaoyu and Gong, Neil Zhenqiang},
  booktitle={Proceedings of the AAAI conference on artificial intelligence},
  volume={36},
  number={9},
  pages={9575--9583},
  year={2022}
}

@inproceedings{gao2021learning,
  title={Learning and certification under instance-targeted poisoning},
  author={Gao, Ji and Karbasi, Amin and Mahmoody, Mohammad},
  booktitle={Uncertainty in Artificial Intelligence},
  pages={2135--2145},
  year={2021},
  organization={PMLR}
}

@article{hanneke2022optimal,
  title={On optimal learning under targeted data poisoning},
  author={Hanneke, Steve and Karbasi, Amin and Mahmoody, Mohammad and Mehalel, Idan and Moran, Shay},
  journal={Advances in Neural Information Processing Systems},
  volume={35},
  pages={30770--30782},
  year={2022}
}

@inproceedings{chen2023continuous,
  title={Continuous learning for android malware detection},
  author={Chen, Yizheng and Ding, Zhoujie and Wagner, David},
  booktitle={32nd USENIX Security Symposium (USENIX Security 23)},
  pages={1127--1144},
  year={2023}
}

@inproceedings{chen2021learning,
  title={Learning security classifiers with verified global robustness properties},
  author={Chen, Yizheng and Wang, Shiqi and Qin, Yue and Liao, Xiaojing and Jana, Suman and Wagner, David},
  booktitle={Proceedings of the 2021 ACM SIGSAC Conference on Computer and Communications Security},
  pages={477--494},
  year={2021}
}

@article{tramer2020differentially,
  title={Differentially private learning needs better features (or much more data)},
  author={Tramer, Florian and Boneh, Dan},
  journal={arXiv preprint arXiv:2011.11660},
  year={2020}
}

@inproceedings{ferrari2019we,
  title={Are we really making much progress? A worrying analysis of recent neural recommendation approaches},
  author={Ferrari Dacrema, Maurizio and Cremonesi, Paolo and Jannach, Dietmar},
  booktitle={Proceedings of the 13th ACM conference on recommender systems},
  pages={101--109},
  year={2019}
}

@article{scikit-learn,
  title={Scikit-learn: Machine Learning in {P}ython},
  author={Pedregosa, F. and Varoquaux, G. and Gramfort, A. and Michel, V.
          and Thirion, B. and Grisel, O. and Blondel, M. and Prettenhofer, P.
          and Weiss, R. and Dubourg, V. and Vanderplas, J. and Passos, A. and
          Cournapeau, D. and Brucher, M. and Perrot, M. and Duchesnay, E.},
  journal={Journal of Machine Learning Research},
  volume={12},
  pages={2825--2830},
  year={2011}
}

@incollection{karp2009reducibility,
  title={Reducibility among combinatorial problems},
  author={Karp, Richard M},
  booktitle={50 Years of Integer Programming 1958-2008: from the Early Years to the State-of-the-Art},
  pages={219--241},
  year={2009},
  publisher={Springer}
}

\appendix
\section{Appendix}
\subsection{Additional Robustness Examples}
The below figures are additional examples of robustness that our tool computed to the test points from the dataset Pen-based Handwritten Digit Recognition~\cite{pen-based_recognition_of_handwritten_digits_81}. In these figures, left most plot is the test point, while the remaining points are training points with respect to robustness. The labels mentioned are true labels. By changing the training points' labels, the test point's label can be changed.   

\begin{figure}[!htbp]  
    \centering  
    \includegraphics[width=0.45\textwidth]{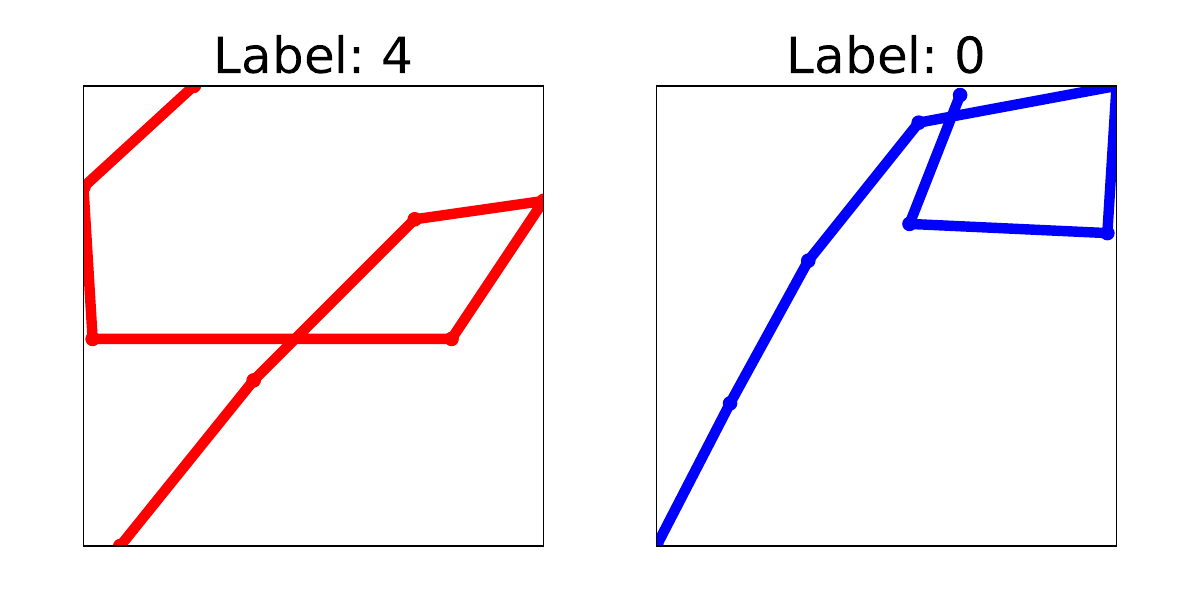}  
    \includegraphics[width=0.45\textwidth]{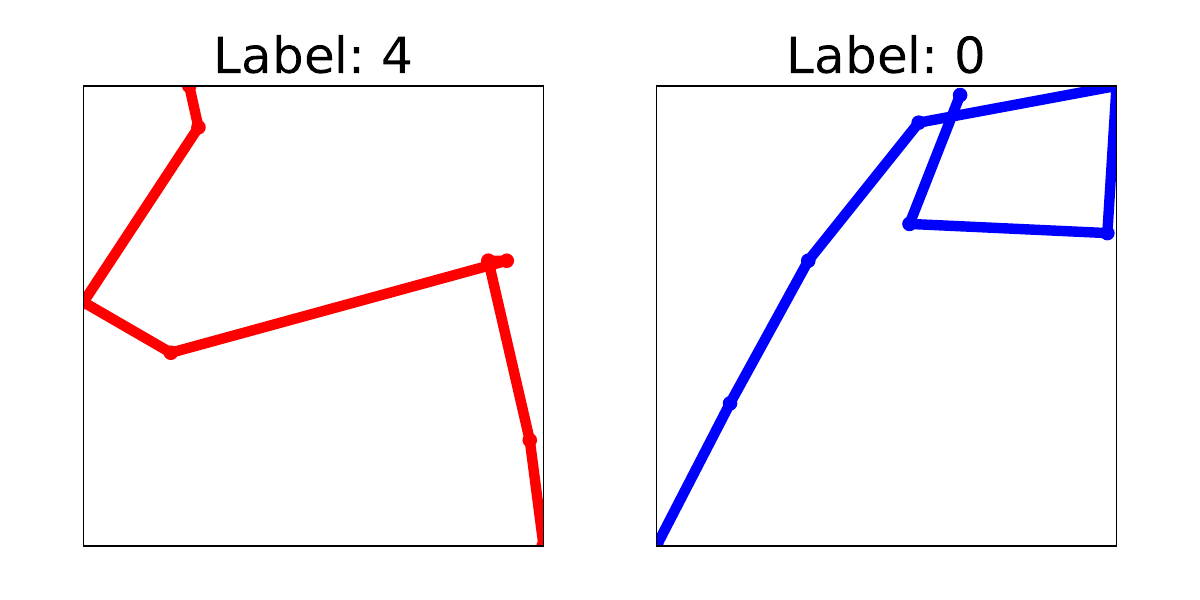} 
    \caption{Robustness 1 with different labels for test and training points.}  
    \label{fig:digit_attack_test38_train1136}  
\end{figure}  
  
\begin{figure}[!htbp]  
    \centering  
    \includegraphics[width=0.45\textwidth]{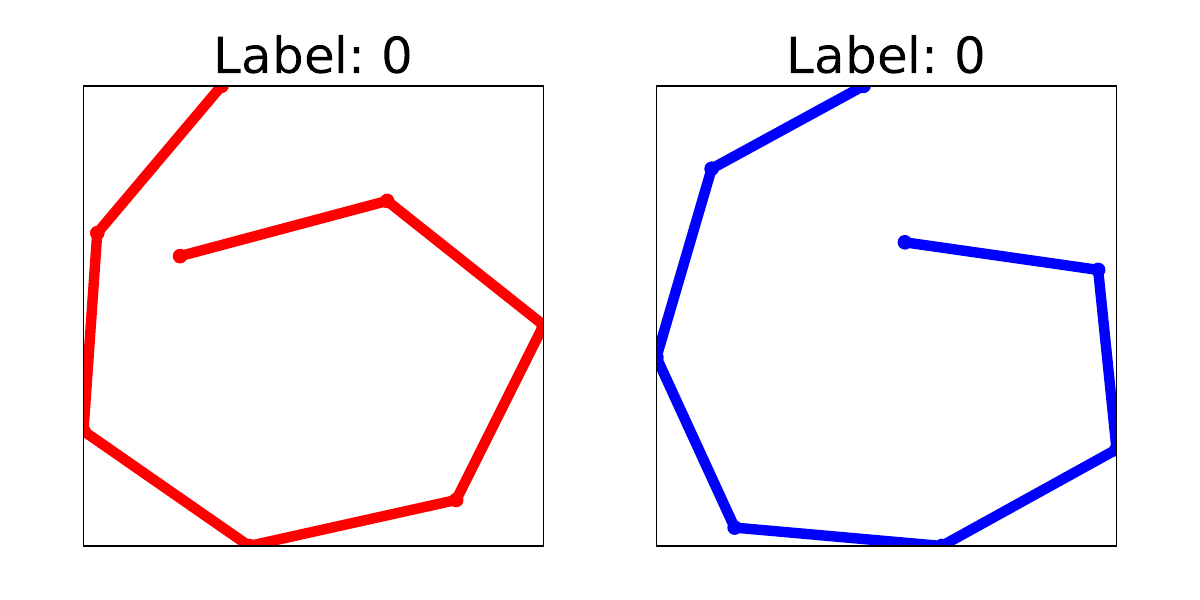} 
    \includegraphics[width=0.45\textwidth]{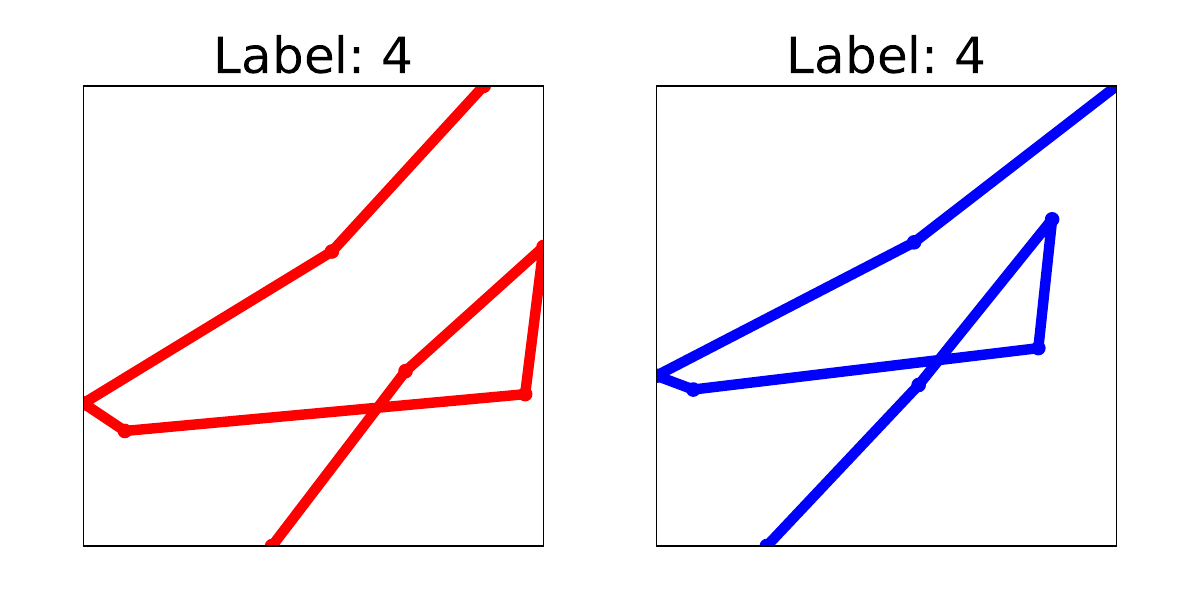}  
    \caption{Robustness 1 with same labels for test and training points.}  
    \label{fig:digit_attack_test61_train112}  
\end{figure}  
  
\begin{figure}[!htbp]  
    \centering 
    \includegraphics[width=0.45\textwidth]{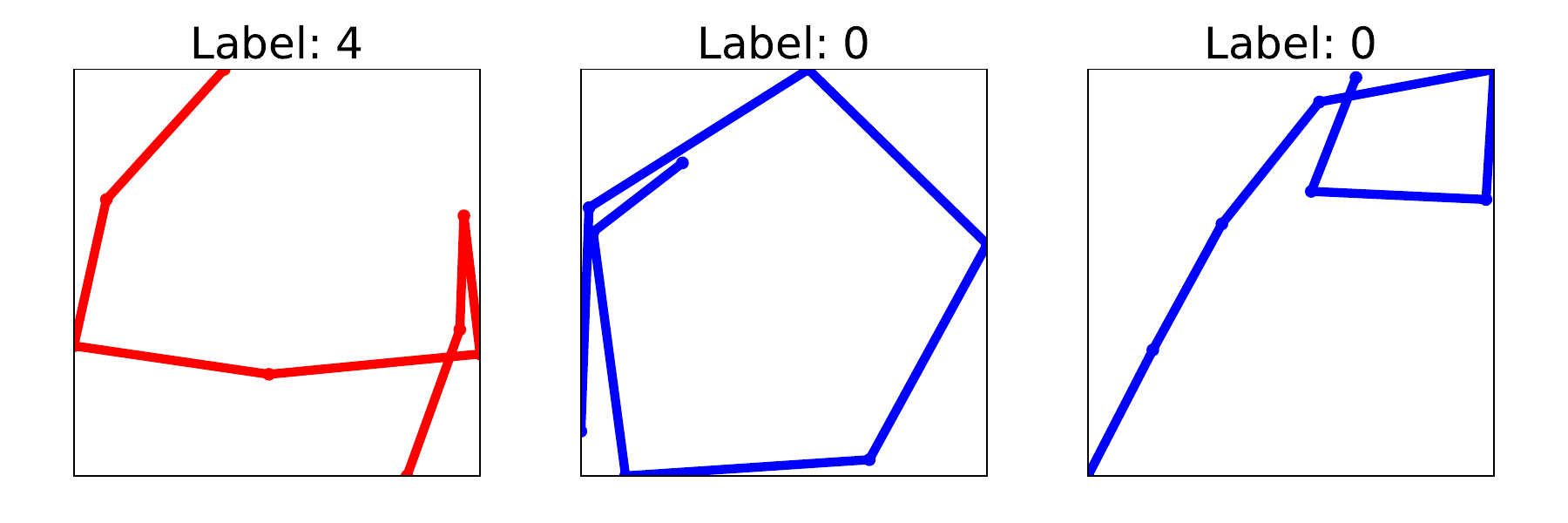}  
    \includegraphics[width=0.45\textwidth]{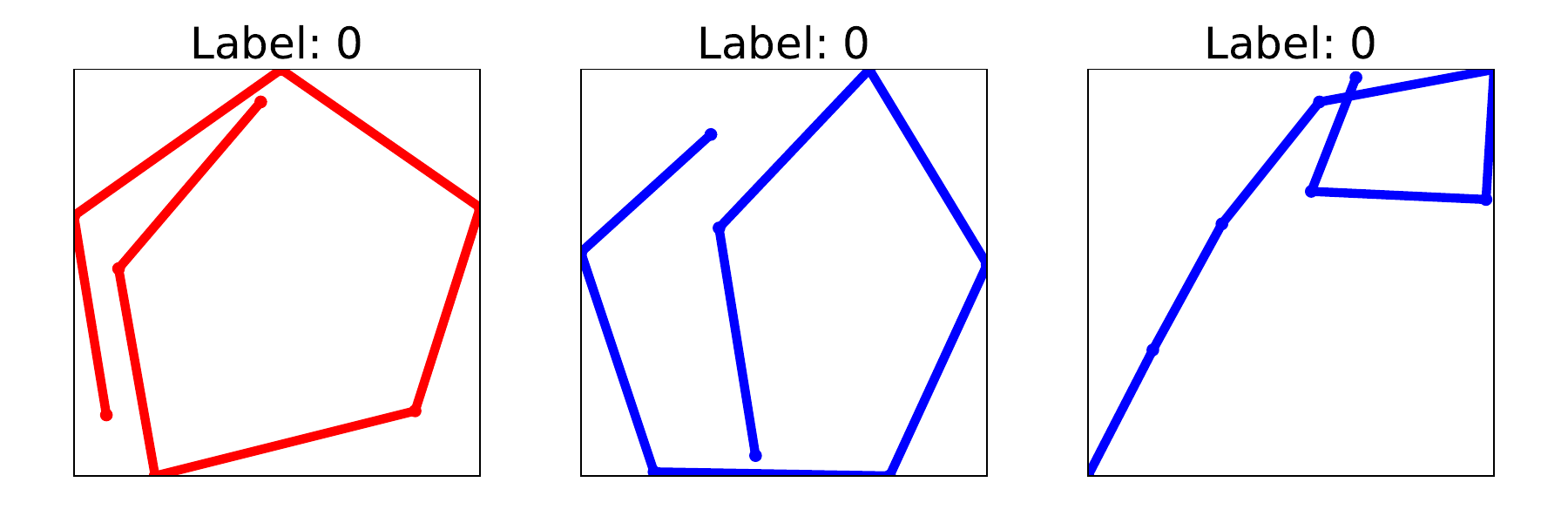}  
    \caption{Robustness 2.}  
    \label{fig:digit_attack_test111_train585}  
\end{figure}  

\clearpage

\subsection{Remaining Proofs}
\begin{theorem}
    For dataset $\ctdata$ and a test point $(x_t, y_t)$, $\lowrob \le \robustness$.
\end{theorem}
\begin{proof}
    We prove by contradiction. Assume $\robustness < \lowrob$. Consider the training points from $\ctdata$ that are witness to $\robustness$ (i.e., these training points results in $\robustness$). Now, consider a partition of $\ctdata$ into $\ctdata^1 \dots \ctdata^k$. Suppose our partition-based algorithm computes robustness of these partitions as $\lowrob^1 \dots \lowrob^k$. Let $\robustness^i$ be the witness points with respect to partition $\ctdata^i$ (i.e, training points from $\ctdata^i$). By our assumption, we have $\sum_i^{k} \robustness^i < \sum_i^{k} \lowrob^i$. Furthermore, each $\robustness^i < \lowrob^i$ as there can not be overlaps as $\ctdata^1 \dots \ctdata^k$ are partitions. But, $\lowrob^i$ is the minimal robustness for $\ctdata^i$, which is ensured by the encoding. Hence, it is not possible to have $\robustness^i < \lowrob^i$, a contradiction.  
\end{proof}

\begin{theorem}
    For dataset $\ctdata$ and a test point $(x_t, y_t)$, $\robustness \le \uprob$.
\end{theorem}
\begin{proof}
    Consider $\uprob$ computed via augmentation. The method ensures that the test point classifies as desired. It is not possible to have $\uprob < \robustness$, as $\robustness$ is the minimal perturbations required to classify the test point as desired. 
    Furthermore, an upper bound for $k'$ can also be computed by assuming certain properties on the loss function $\lossfn$. For instance, consider that $\lossfn$ is bounded by an interval $[L, U]$. Then, we have $k' \le \frac{\tsize U}{L}$. 
\end{proof}

\subsection{Additional Research Question}
In this section, we evaluate our tool \robint against an addition research question that could not be included in the evaluaion section due to space constraints.
\subsubsection{RQ5: Can a closely related sanitization technique help in increasing robustness?}
\begin{figure}[h]  
    \centering  
    \includegraphics[width=0.45\textwidth]{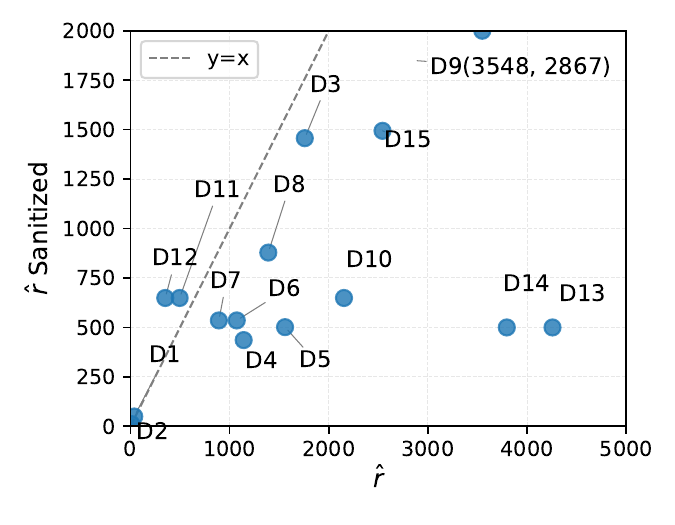}  
    \caption{  
        Average robustness before and after KNN sanitization~\cite{paudice2019label}.
    }  
    \label{fig:originalVsSanitized}  
\end{figure}    
The work in ~\cite{paudice2019label} also computes upper bound robustness in a targeted setting, but it requires a predefined bound and assumes a stronger threat model. Along with robustness, a technique based on K-nearest neighbors (KNN) is presented to sanitize the dataset. We implemented this sanitization method in our tool and computed robustness before and after sanitization. The results are in Figure~\ref{fig:originalVsSanitized}. We observed that this process increased the average robustness of only 2 out of 15 datasets, while it decreased in the remaining ones. We hypothesize this could be because the robustness we compute may not align with the nearest neighbor structure.

\subsection{Table 1 with Standard Deviation}
Table~\ref{tab:robustness_metrics2} presents standard deviations for the average values tabulated in Table~\ref{tab:robustness_metrics}. 

\begin{sidewaystable}[htbp]  
\centering  
\begingroup  
\begin{tabular}{lrrrrrrrrrrrrrr}  
\toprule  
\textbf{Dataset} & $\dimension$ & $\mathbf{\tsize}$ & $\mathbf{\#x_t}$ & $\mathbf{\uprob}$ & $\mathbf{\uprob}_{\ipr}$ (N) & $\mathbf{\uprob}_{\ipr}$ & \%Found$\mathbf{\uprob}_{\ipr}$ & $\rho$ & $\rho_{\ipr}$ & $\mathbf{\lowrob}$ & $\mathbf{\uprob}_{\ipr}\ \mathrm{SD}$ & $\mathbf{\uprob}_{\ipr} (N) \mathrm{SD}$ & $\mathbf{\uprob} SD$ & $\mathbf{\lowrob} SD$ \\  
\midrule  
Letter Recognition     & 16   & 1335   & 149  & \textbf{38.16}    & 569.62     & 168.18      & 66  & \textbf{0.55} & 0.34 & 0.68 & 192     & 577.8   & 36.7   & 2.3   \\  
Digits Recognition    & 16   & 1404   & 156  & \textbf{10.79}    & 1305.88    & 226.54      & 8   & \textbf{0.17} & 0.04 & 0.03 & 292.1    & 336.2   & 19.3   & 0.3   \\  
SST (BOW)             & 300  & 6920   & 872  & 1758.08           & \textbf{93.76}     & 66.99       & 100 & \textbf{0.99} & 0.61 & 0.96 & 131.1    & 474     & 24.3   & 1.3   \\  
SST (BERT)            & 768  & 6920   & 872  & 1141.54           & \textbf{439.64}    & 247.90      & 97  & \textbf{0.97} & 0.52 & 0.00 & 453.1    & 1273    & 35.8   & 0     \\  
Emo (BOW)             & 300  & 9025   & 1003 & 1558.92           & \textbf{237.95}    & 167.28      & 99  & \textbf{0.99} & 0.52 & 0.00 & 382.9    & 875.5   & 37.4   & 0     \\  
Speech (BOW)          & 300  & 9632   & 1071 & 1071.90           & \textbf{624.09}    & 373.20      & 97  & \textbf{0.61} & 0.58 & 0.49 & 464.5    & 1572.2  & 15.7   & 6.5   \\  
Speech (BERT)         & 768  & 9632   & 1071 & \textbf{891.28}   & 931.04     & 593.20      & 96  & \textbf{0.54} & 0.49 & 0.10 & 850.3    & 1907.3  & 17.9   & 0.9   \\  
Fashion-MNIST         & 784  & 10800  & 1200 & 1389.91           & \textbf{1037.25}   & 274.13      & 93  & \textbf{0.98} & 0.45 & 0.01 & 651.1    & 2800.6  & 47.8   & 0.1   \\  
Loan (BOW)            & 18   & 11200  & 2800 & \textbf{3548.16}  & 10069.39   & 42.22       & 60  & \textbf{0.93} & 0.46 & 111.21 & 88.8     & 12317.5 & 1364.1 & 125.5 \\  
Emo (BERT)            & 768  & 11678  & 1298 & \textbf{2153.26}  & 2913.84    & 232.23      & 77  & \textbf{0.97} & 0.40 & 560.45 & 445.9    & 4864.7  & 1097.7 & 1086.3 \\  
Essays (BOW)          & 300  & 11678  & 1298 & \textbf{495.90}   & 2884.23    & 1664.89     & 88  & 0.21          & \textbf{0.50} & 0.00 & 1644     & 3619.6  & 19.6   & 0     \\  
Essays (BERT)         & 768  & 11678  & 1298 & \textbf{350.85}   & 3761.81    & 1245.04     & 76  & 0.22          & \textbf{0.39} & 0.00 & 1478     & 4646.6  & 11.9   & 0     \\  
Tweet (BOW)           & 300  & 18000  & 1000 & 4256.26           & \textbf{304.15}    & 216.32      & 99  & \textbf{1.00} & 0.54 & 7.70  & 533      & 1022.3  & 50.8   & 4.3   \\  
Tweet (BERT)          & 768  & 18000  & 1000 & 3795.24           & \textbf{352.38}    & 343.72      & 100 & \textbf{0.99} & 0.55 & 0.32  & 704.3    & 755.3   & 58.5   & 0.8   \\  
Census Income         & 41   & 80136  & 8905 & \textbf{2542.06}  & 66003.10   & 7012.80     & 19  & \textbf{0.33} & 0.07 & 87.19  & 12422.1  & 29395.1 & 1685.3 & 379.6 \\  
\bottomrule  
\end{tabular}  
\endgroup  
\caption{Summary of results. Here, $\dimension$, $\tsize$ -- dimension and size of datasets, $\mathbf{\#x_t}$-- number of test points, $\mathbf{\uprob}$, $\mathbf{\lowrob}$ -- average upper and lower bound robustness from \robint, $\mathbf{\uprob}_{\ipr}$(normalized), $\mathbf{\uprob}_{\ipr}$ -- average upper bound robustness from \iprtool, \%Found $\mathbf{\uprob}_{\ipr}$-- percentage of robustness found by \iprtool, and $\rho$, $\rho_{\ipr}$-- average likelihood of flipping by \robint and \iprtool. The last four columns are standard deviations of average values.}  
\label{tab:robustness_metrics2}  
\end{sidewaystable}  


\subsection{Histograms of $\uprob$ and $\lowrob$}

This section presents histograms of $\uprob$ and $\lowrob$ for all the datasets.

\begin{figure}[!htbp]  
    \centering  
    \includegraphics[width=0.45\textwidth]{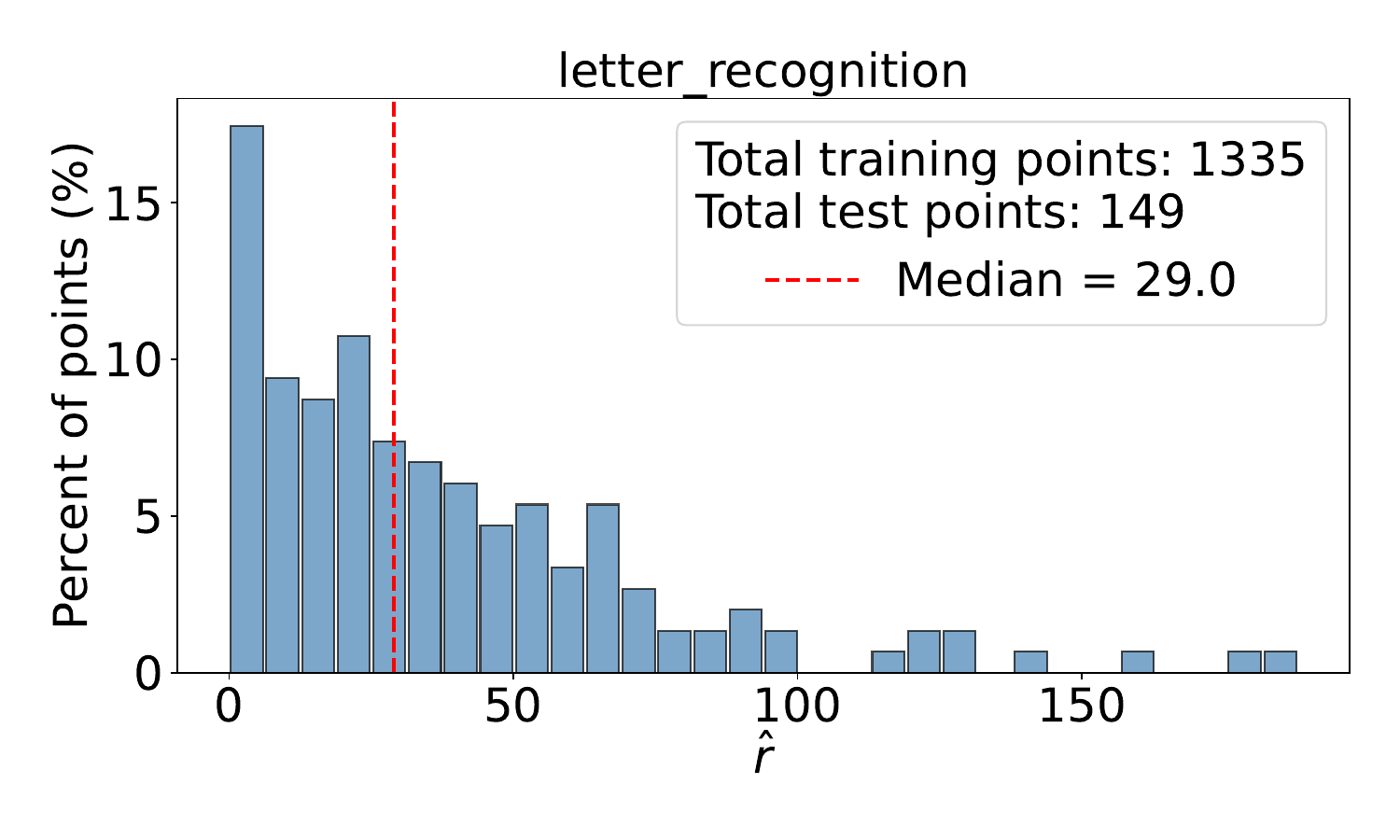}  
    \caption{Histogram of $\uprob$ for Letter Recognition dataset.}  
    \label{fig:ub_hist_letter}  
\end{figure}  

\begin{figure}[!htbp]  
    \centering  
    \includegraphics[width=0.45\textwidth]{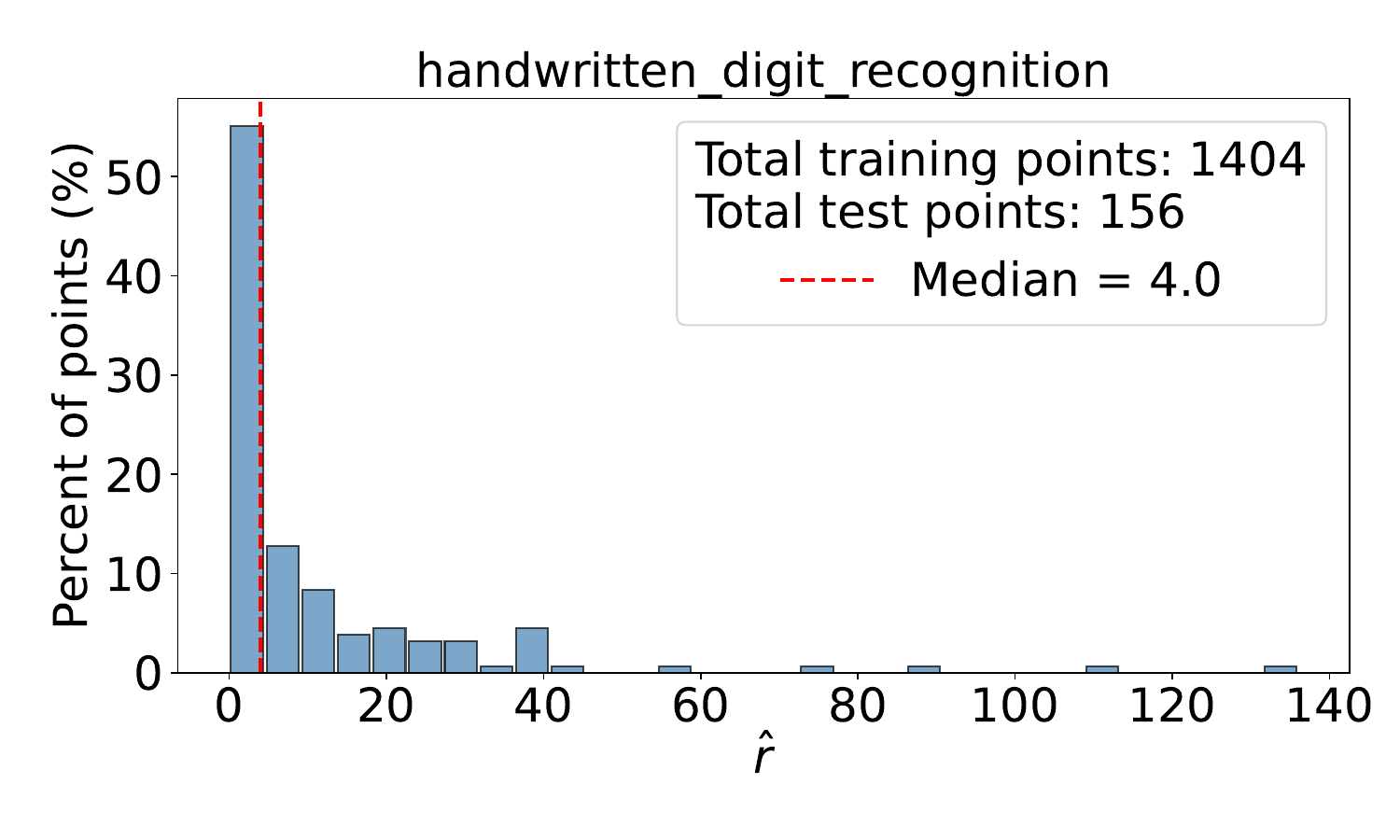}  
    \caption{Histogram of $\uprob$ for Digits Recognition dataset.}  
    \label{fig:ub_hist_handwritten}  
\end{figure}  

\begin{figure}[!htbp]  
    \centering  
    \includegraphics[width=0.45\textwidth]{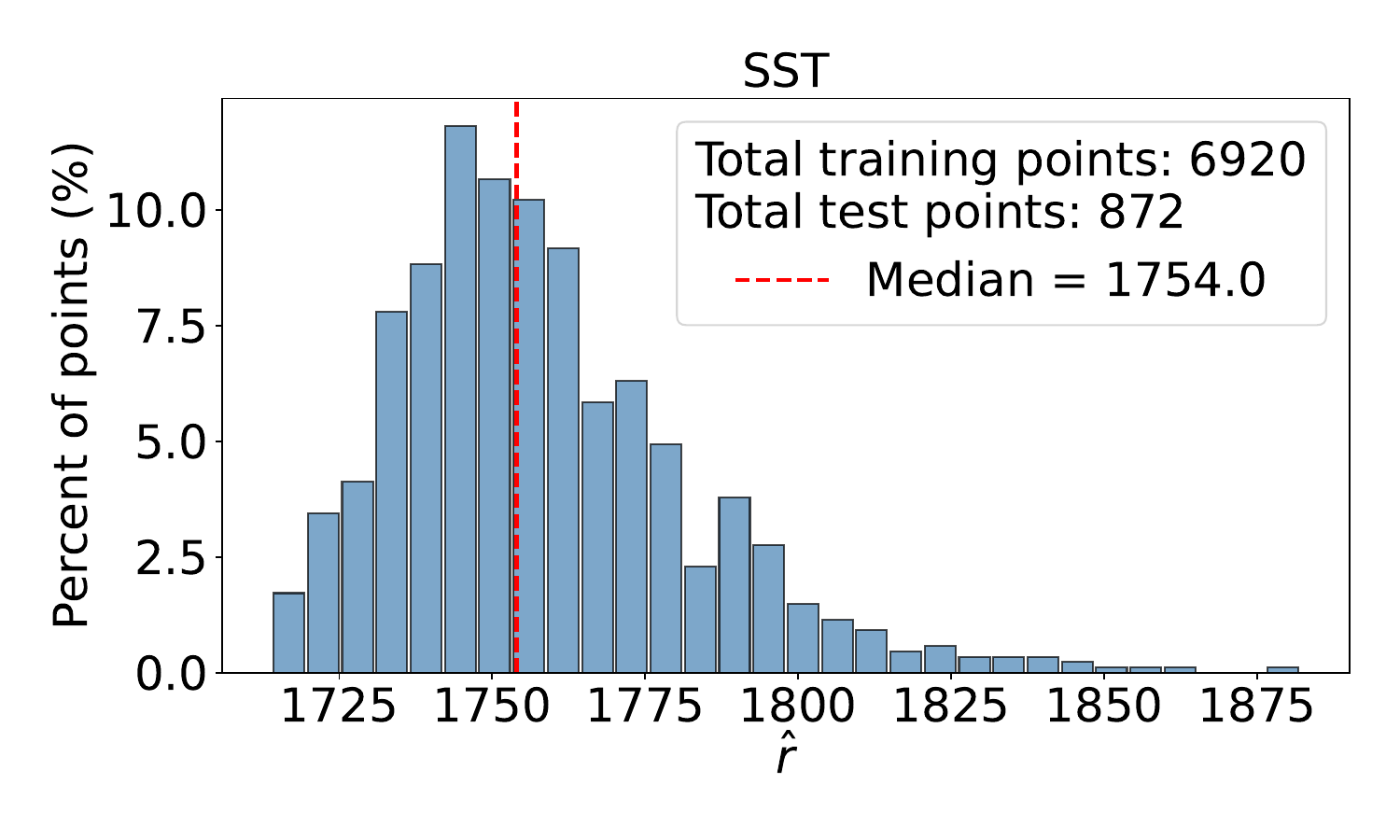}  
    \caption{Histogram of $\uprob$ for SST (BOW) dataset.}  
    \label{fig:ub_hist_SST}  
\end{figure}  
  
\begin{figure}[!htbp]  
    \centering  
    \includegraphics[width=0.45\textwidth]{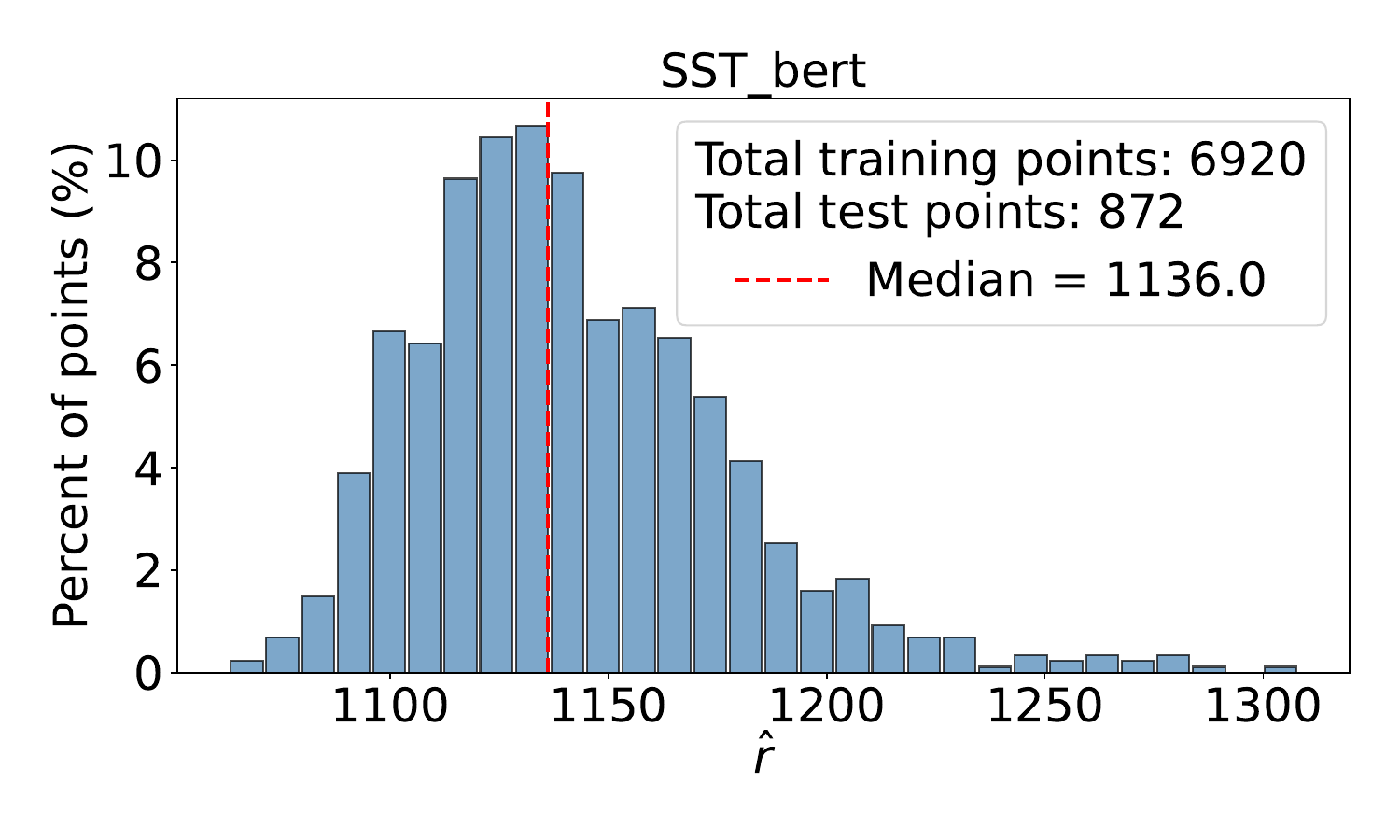}  
    \caption{Histogram of $\uprob$ for SST (BERT) dataset.}  
    \label{fig:ub_hist_SST_bert}  
\end{figure}  

\begin{figure}[!htbp]  
    \centering  
    \includegraphics[width=0.45\textwidth]{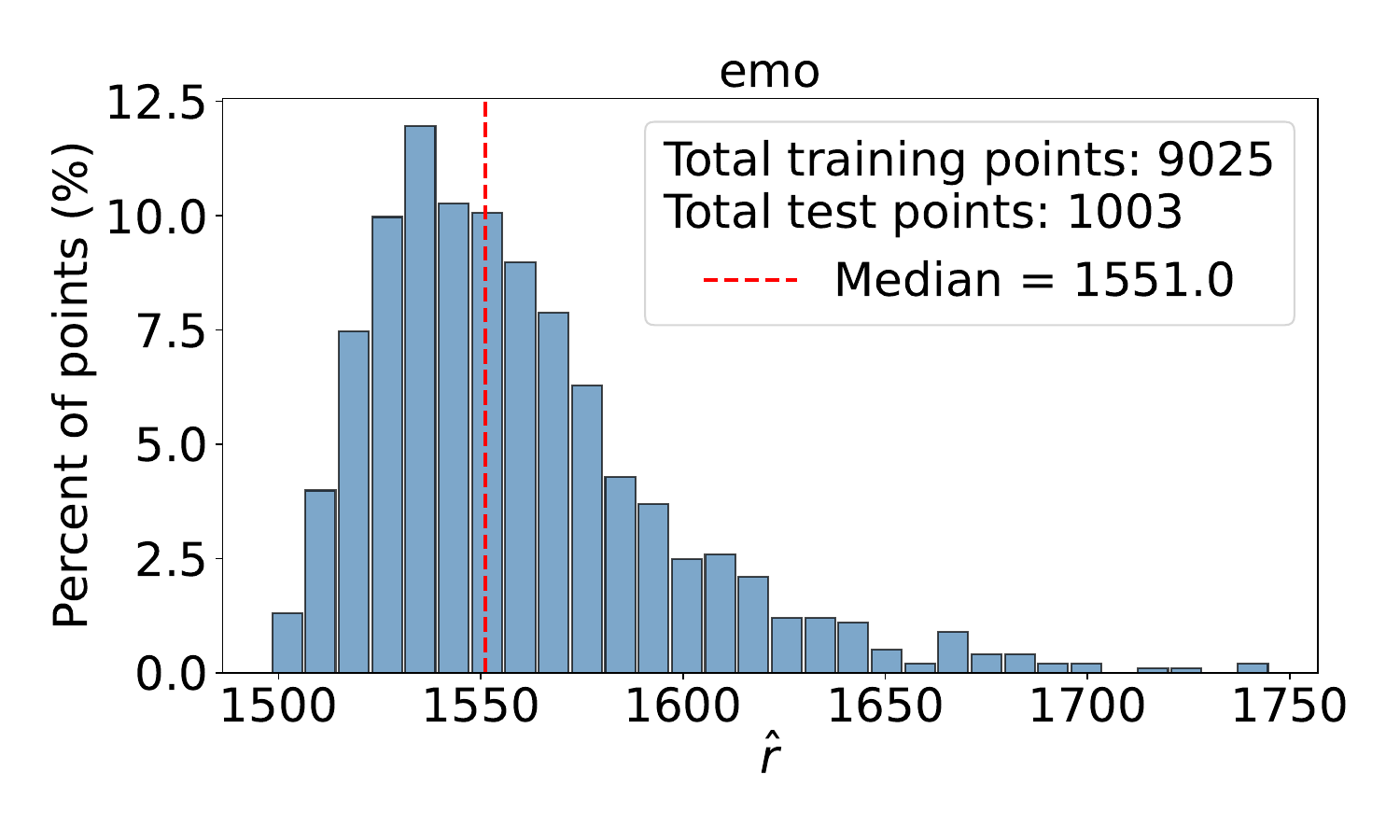}  
    \caption{Histogram of $\uprob$ for Emo (BOW) dataset.}  
    \label{fig:ub_hist_emo}  
\end{figure}

\begin{figure}[!htbp]  
    \centering  
    \includegraphics[width=0.45\textwidth]{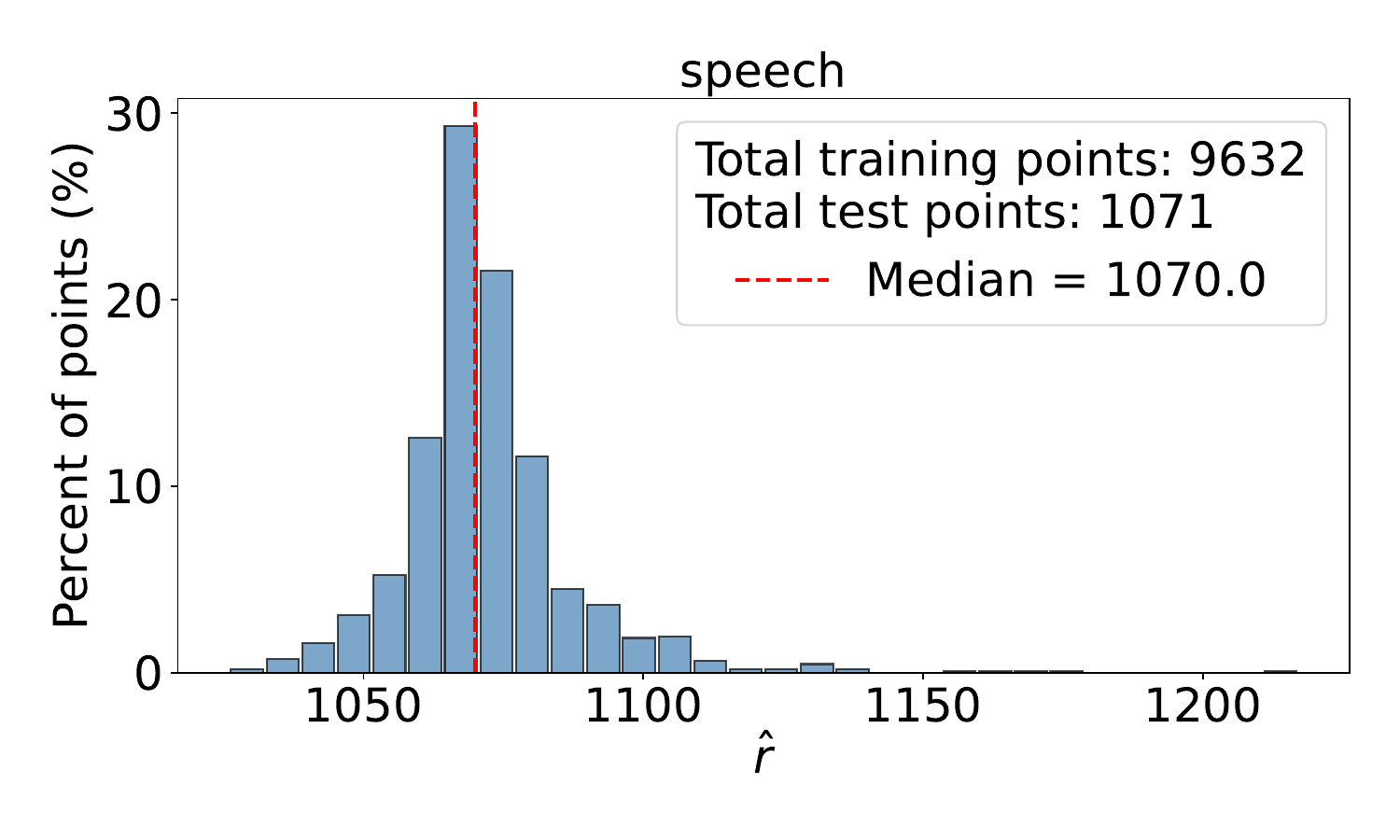}  
    \caption{Histogram of $\uprob$ for Speech (BOW) dataset.}  
    \label{fig:ub_hist_speech}  
\end{figure}  
  
\begin{figure}[!htbp]  
    \centering  
    \includegraphics[width=0.45\textwidth]{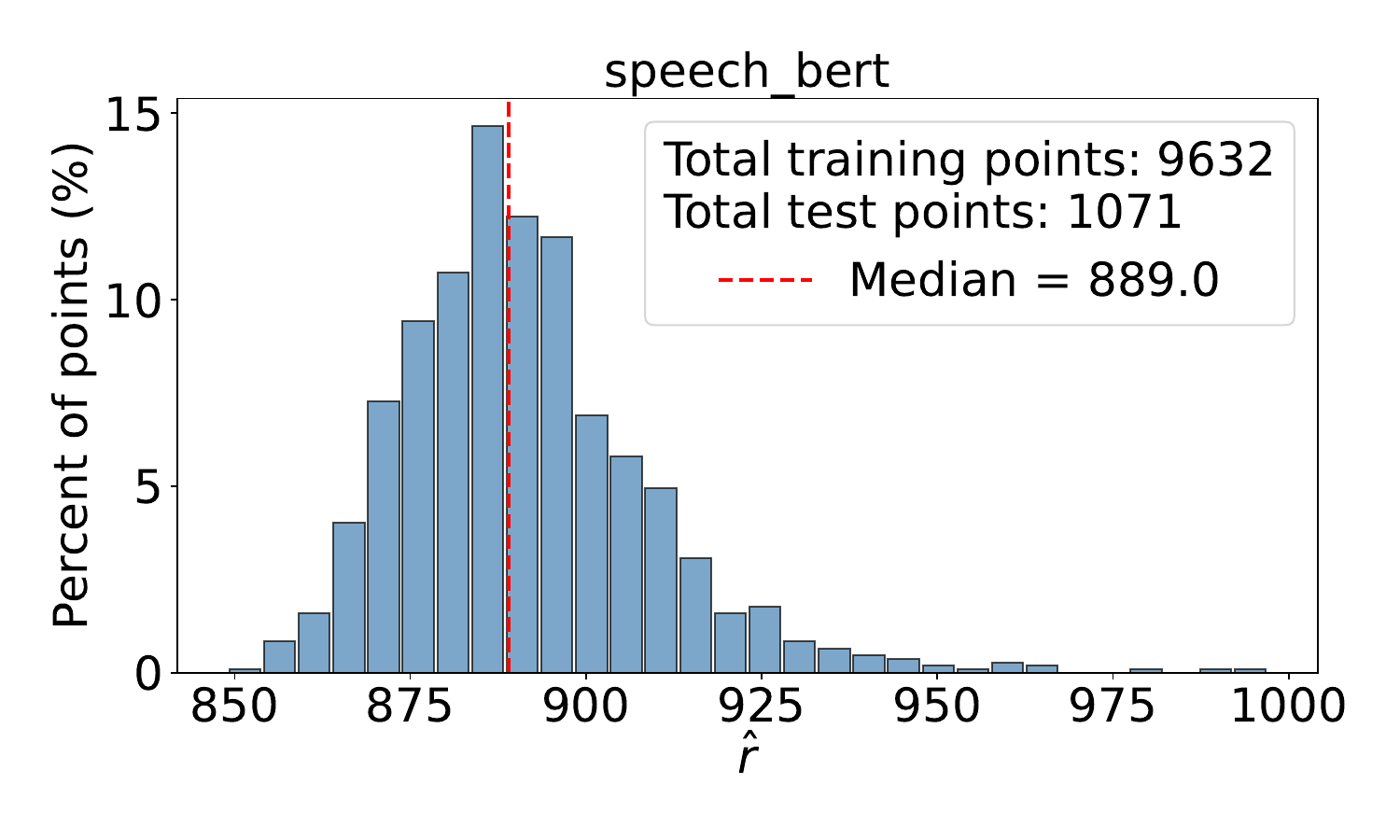}  
    \caption{Histogram of $\uprob$ for Speech (BERT) dataset.}  
    \label{fig:ub_hist_speech_bert}  
\end{figure}

\begin{figure}[!htbp]  
    \centering  
    \includegraphics[width=0.45\textwidth]{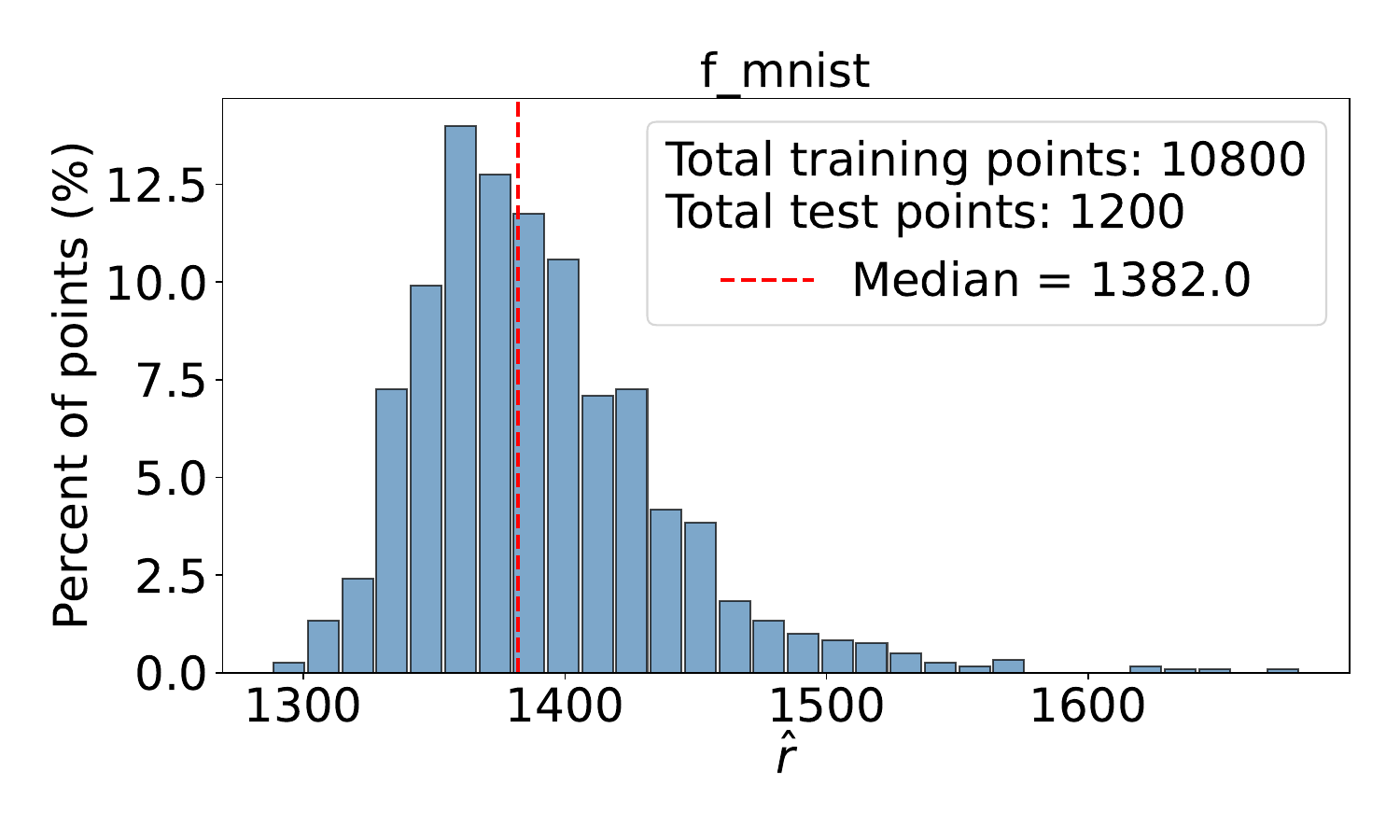}  
    \caption{Histogram of $\uprob$ for Fashion MNIST dataset.}  
    \label{fig:ub_hist_f_mnist}  
\end{figure}  

\begin{figure}[!htbp]  
    \centering  
    \includegraphics[width=0.45\textwidth]{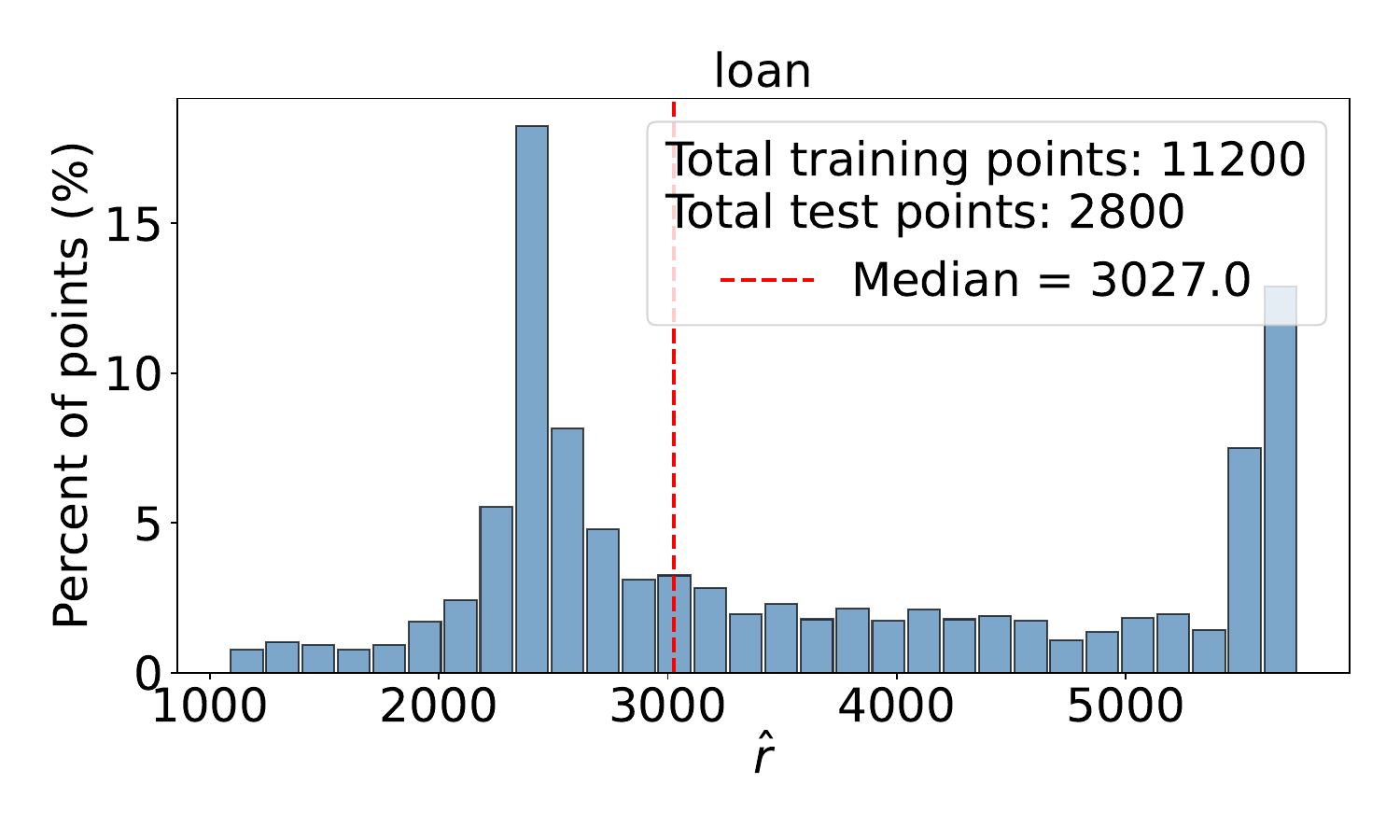}  
    \caption{Histogram of $\uprob$ for Loan (BOW) dataset.}  
    \label{fig:ub_hist_loan}  
\end{figure}  

\begin{figure}[!htbp]  
    \centering  
    \includegraphics[width=0.45\textwidth]{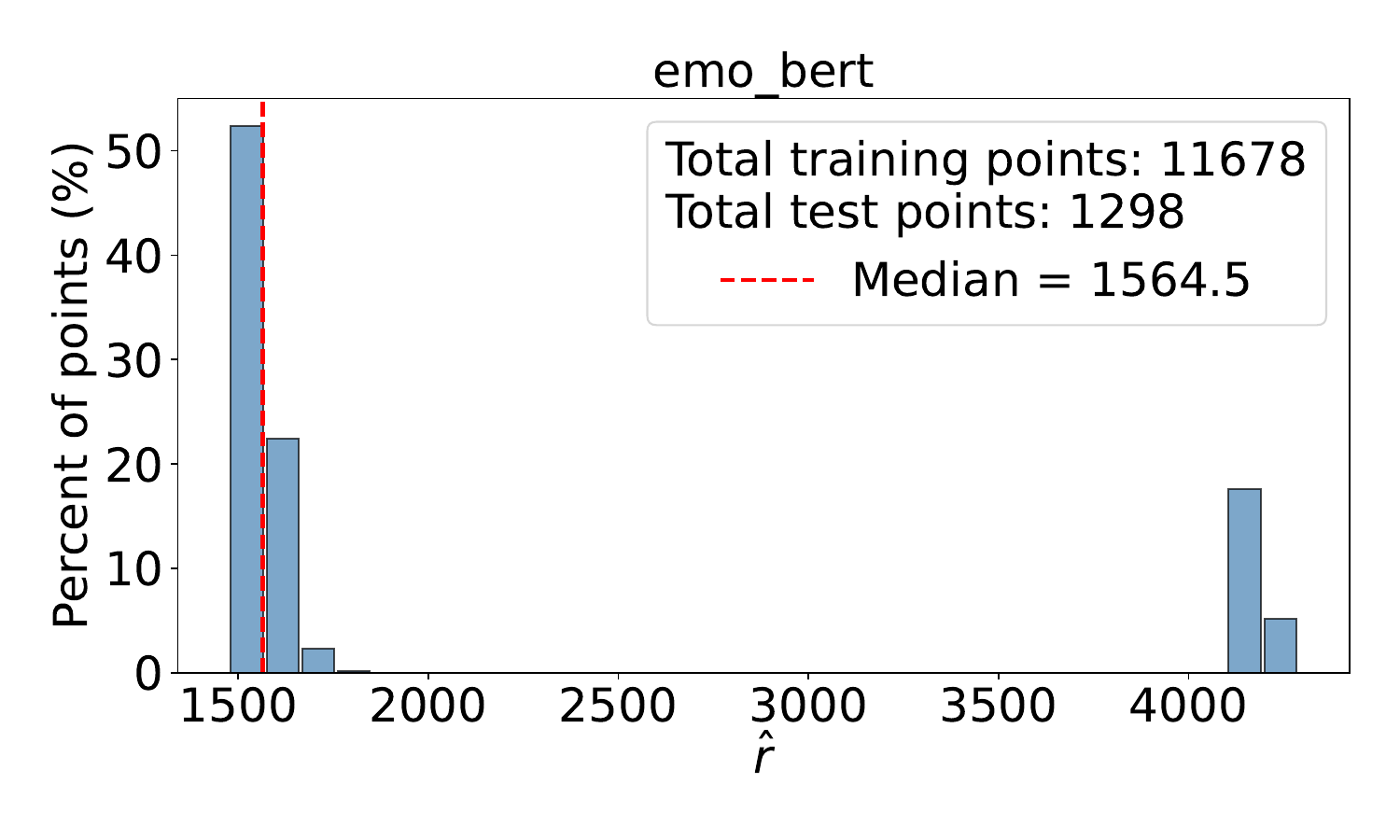}  
    \caption{Histogram of $\uprob$ for Emo (BERT) dataset.}  
    \label{fig:ub_hist_emo_bert}  
\end{figure}  
  
\begin{figure}[!htbp]  
    \centering  
    \includegraphics[width=0.45\textwidth]{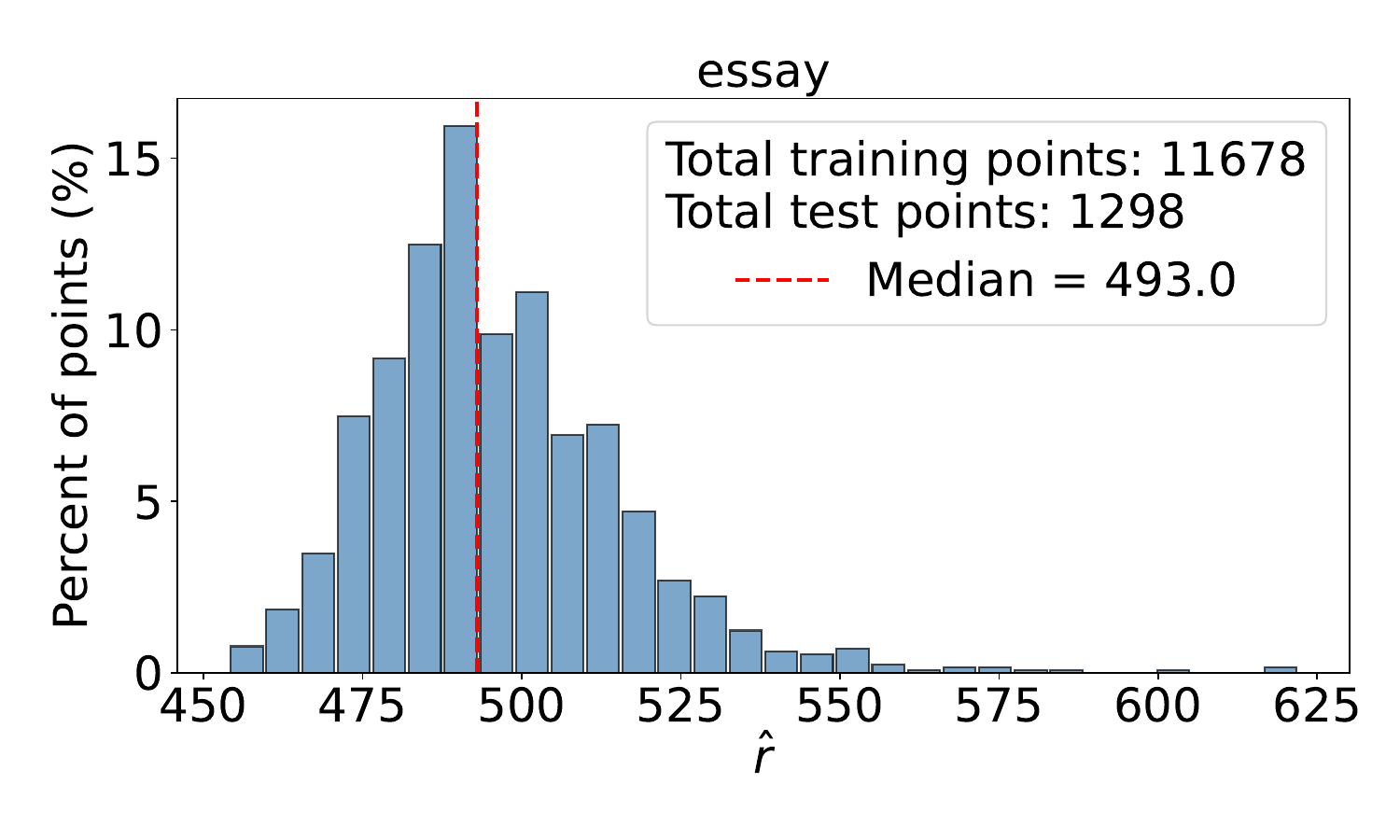}  
    \caption{Histogram of $\uprob$ for Essays (BOW) dataset.}  
    \label{fig:ub_hist_essay}  
\end{figure}  
  
\begin{figure}[!htbp]  
    \centering  
    \includegraphics[width=0.45\textwidth]{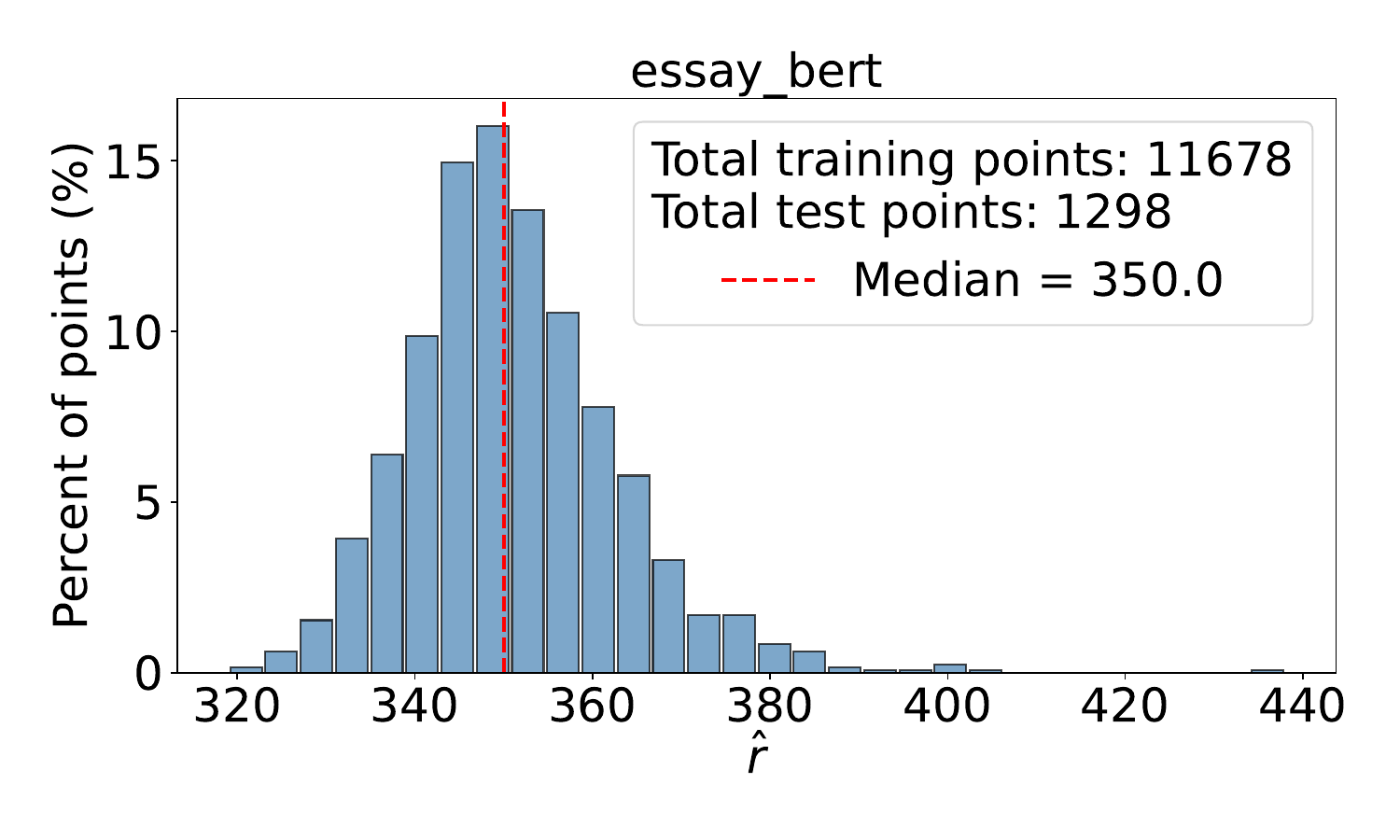}  
    \caption{Histogram of $\uprob$ for Essays (BERT) dataset.}  
    \label{fig:ub_hist_essay_bert}  
\end{figure}  
    
\begin{figure}[!htbp]  
    \centering  
    \includegraphics[width=0.45\textwidth]{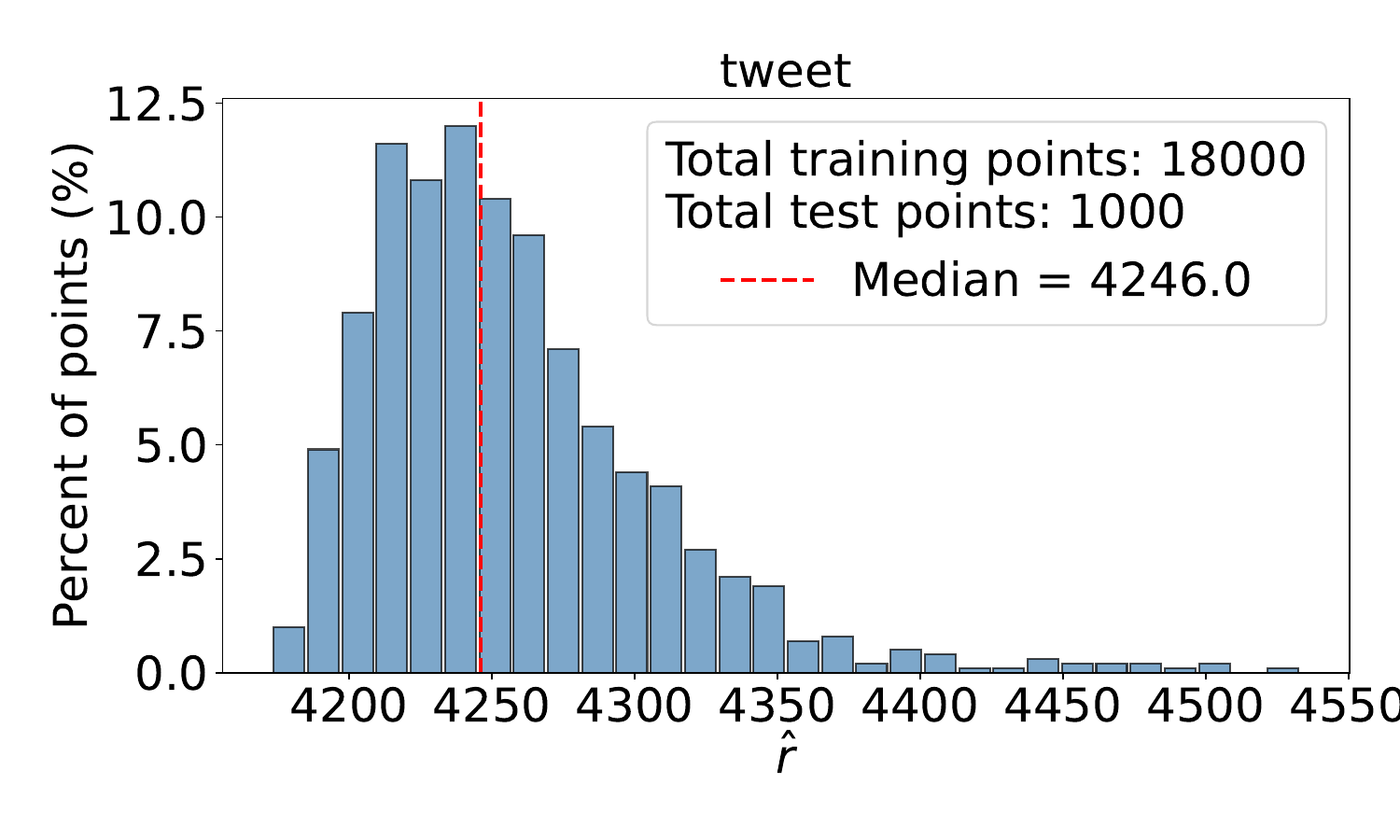}  
    \caption{Histogram of $\uprob$ for Tweet (BOW) dataset.}  
    \label{fig:ub_hist_tweet}  
\end{figure}  
  
\begin{figure}[!htbp]  
    \centering  
    \includegraphics[width=0.45\textwidth]{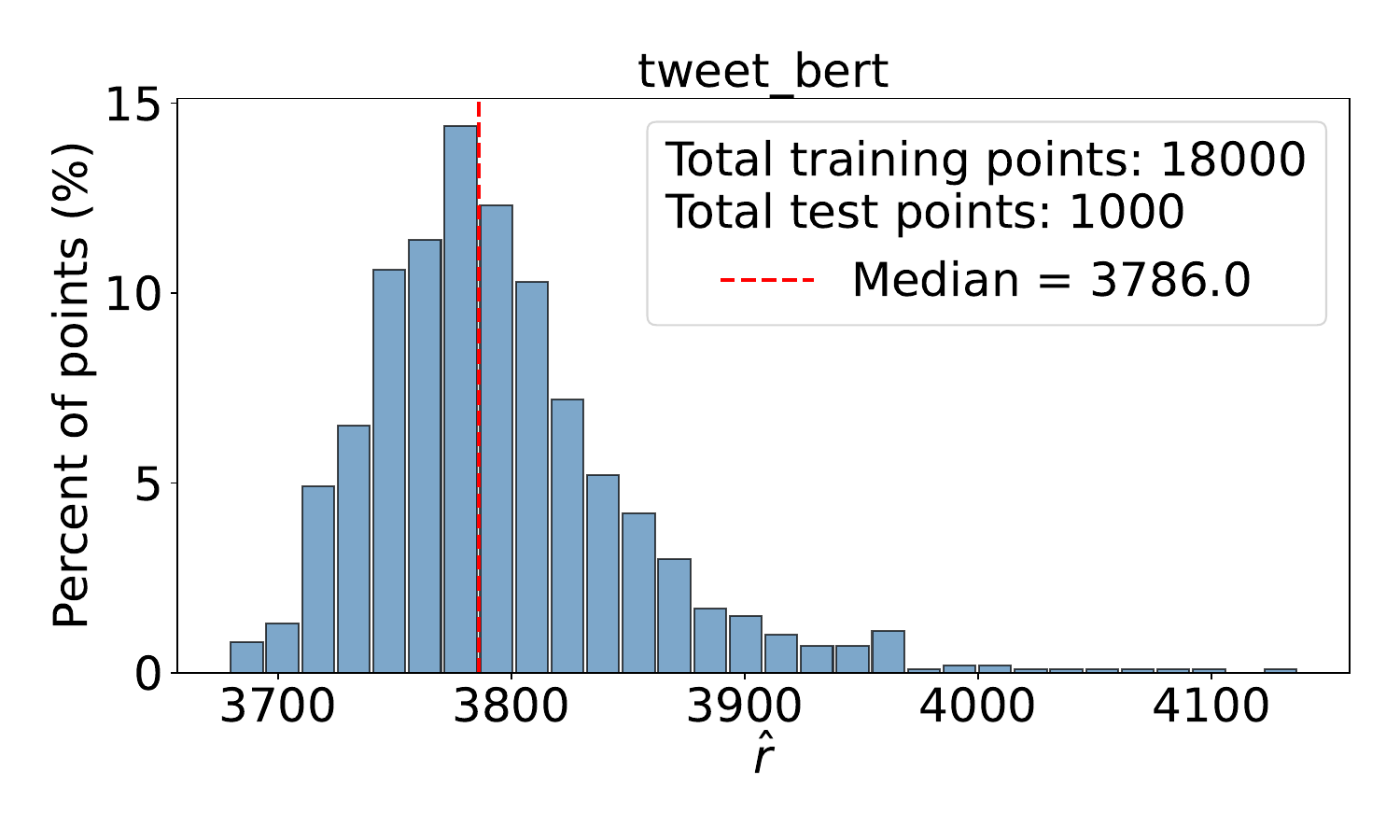}  
    \caption{Histogram of $\uprob$ for Tweet (BERT) dataset.}  
    \label{fig:ub_hist_tweet_bert}  
\end{figure}  
  
\begin{figure}[!htbp]  
    \centering  
    \includegraphics[width=0.45\textwidth]{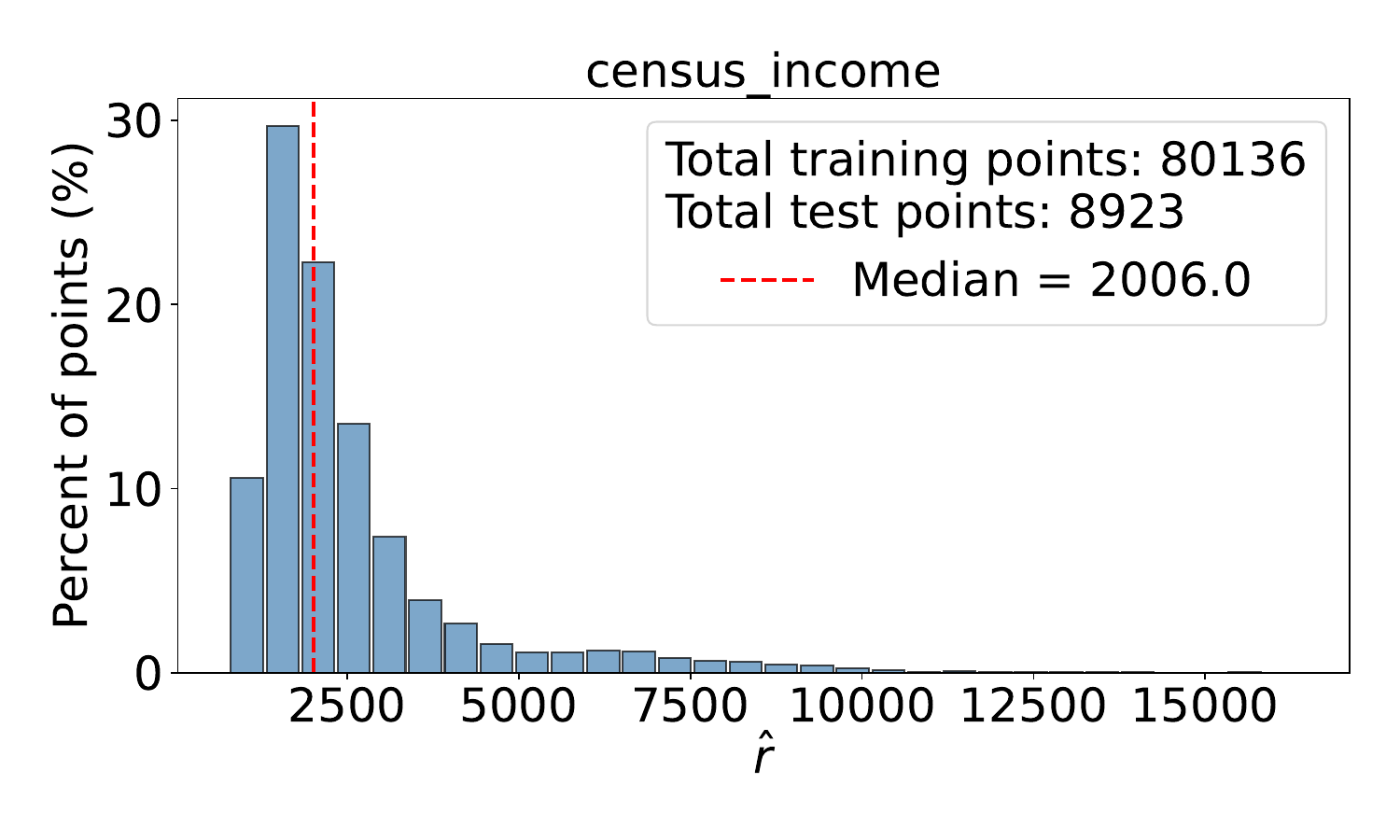}  
    \caption{Histogram of $\uprob$ for Census Income dataset.}  
    \label{fig:ub_hist_census_income}  
\end{figure}


  
\begin{figure}[!htbp]  
    \centering  
    \includegraphics[width=0.45\textwidth]{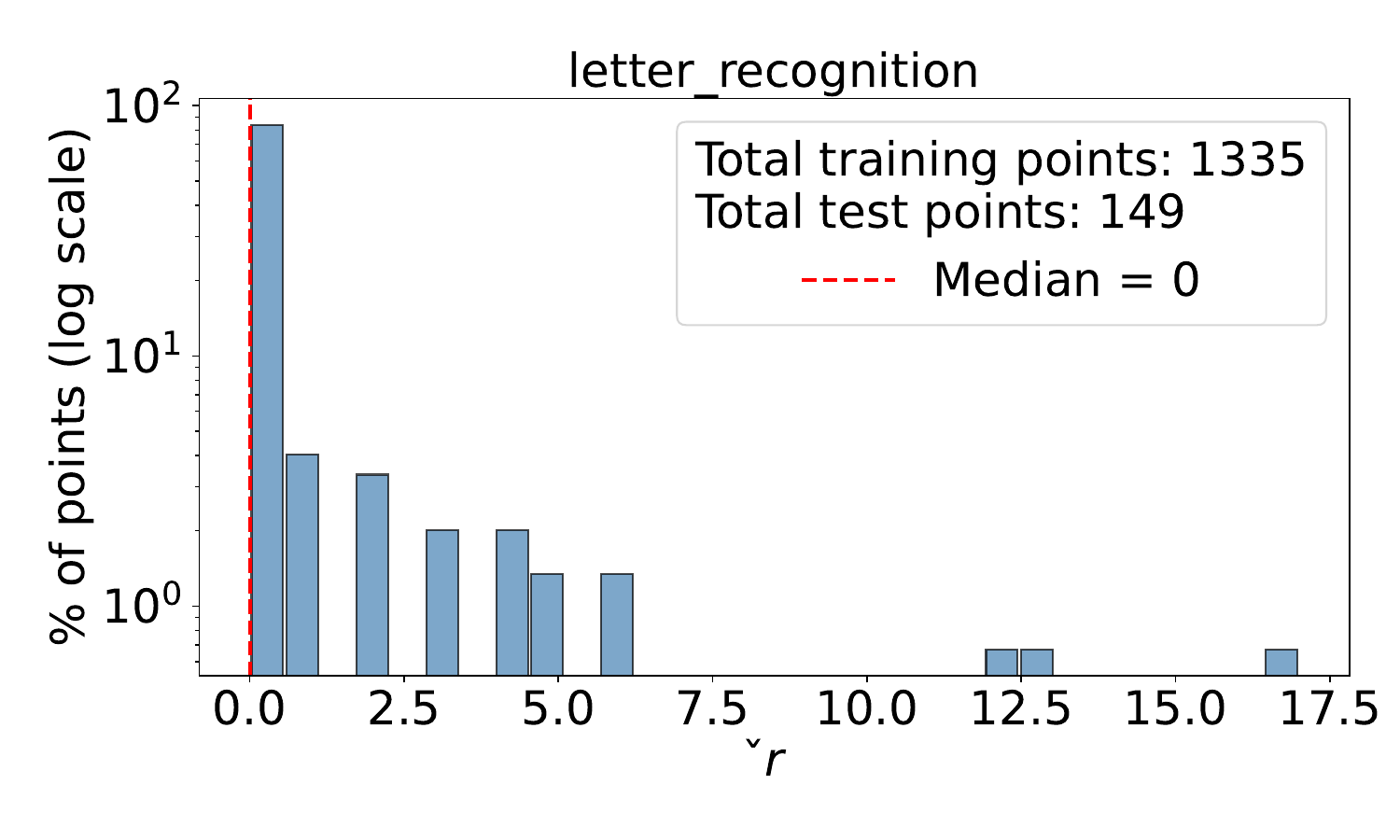}  
    \caption{Histogram of $\lowrob$ for Letter Recognition dataset.}  
    \label{fig:lb_hist_letter}  
\end{figure}  
   
\begin{figure}[!htbp]  
    \centering  
    \includegraphics[width=0.45\textwidth]{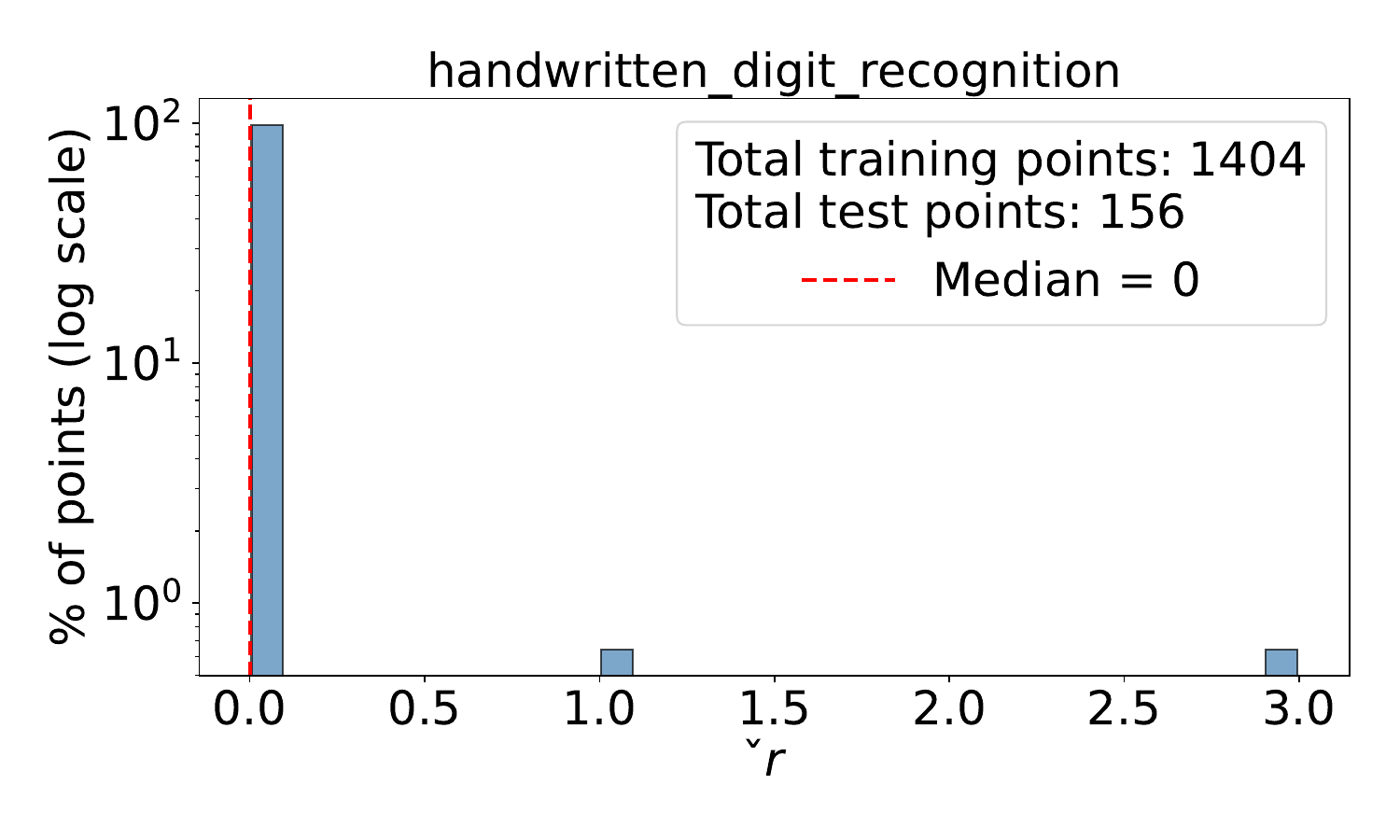}  
    \caption{Histogram of $\lowrob$ for Digits Recognition dataset.}  
    \label{fig:lb_hist_handwritten}  
\end{figure}  
  
\begin{figure}[!htbp]  
    \centering  
    \includegraphics[width=0.45\textwidth]{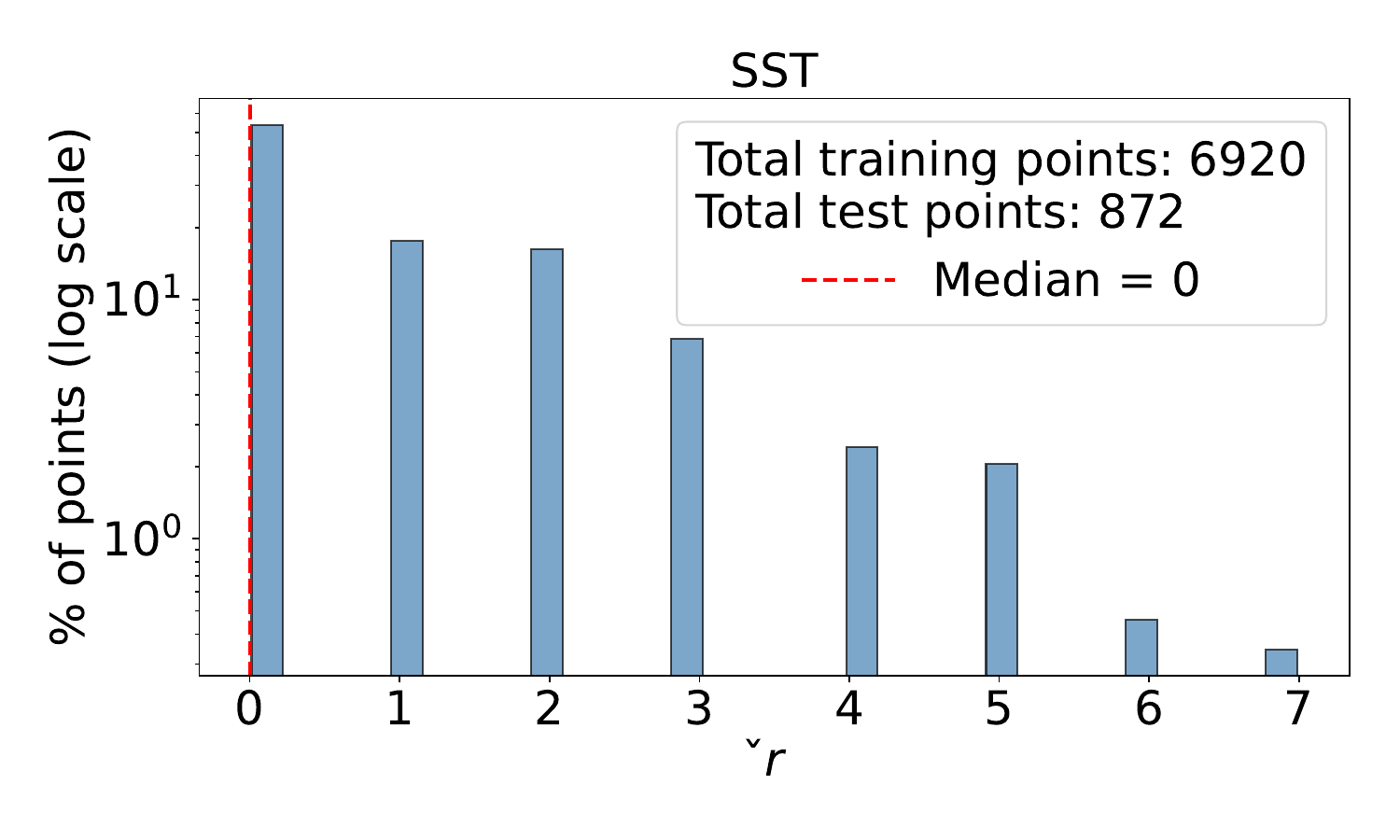}  
    \caption{Histogram of $\lowrob$ for SST (BOW) dataset.}  
    \label{fig:lb_hist_SST}  
\end{figure}  
  
\begin{figure}[!htbp]  
    \centering  
    \includegraphics[width=0.45\textwidth]{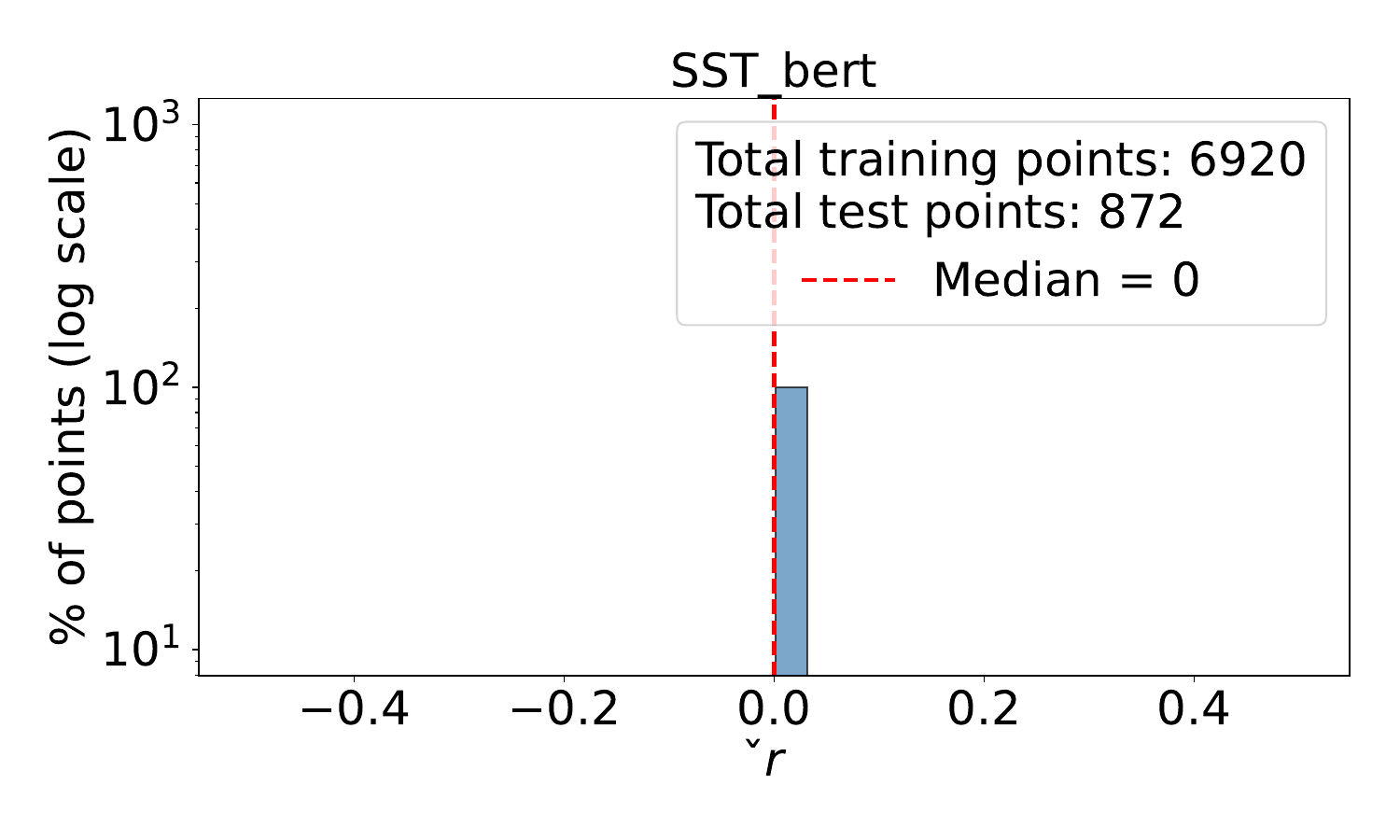}  
    \caption{Histogram of $\lowrob$ for SST (BERT) dataset.}  
    \label{fig:lb_hist_SST_bert}  
\end{figure}  

\begin{figure}[!htbp]  
    \centering  
    \includegraphics[width=0.45\textwidth]{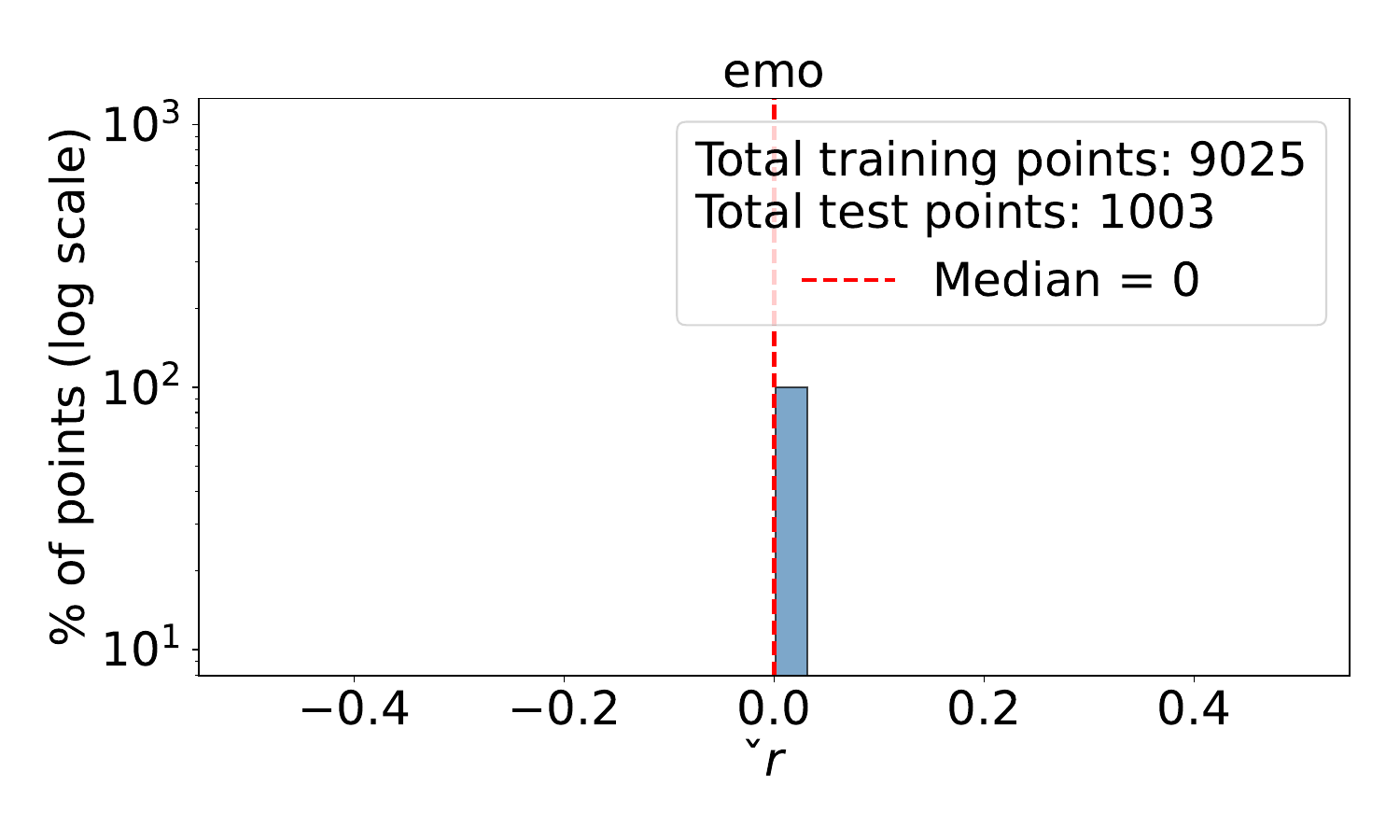}  
    \caption{Histogram of $\lowrob$ for Emo (BOW) dataset.}  
    \label{fig:lb_hist_emo}  
\end{figure}  

\begin{figure}[!htbp]  
    \centering  
    \includegraphics[width=0.45\textwidth]{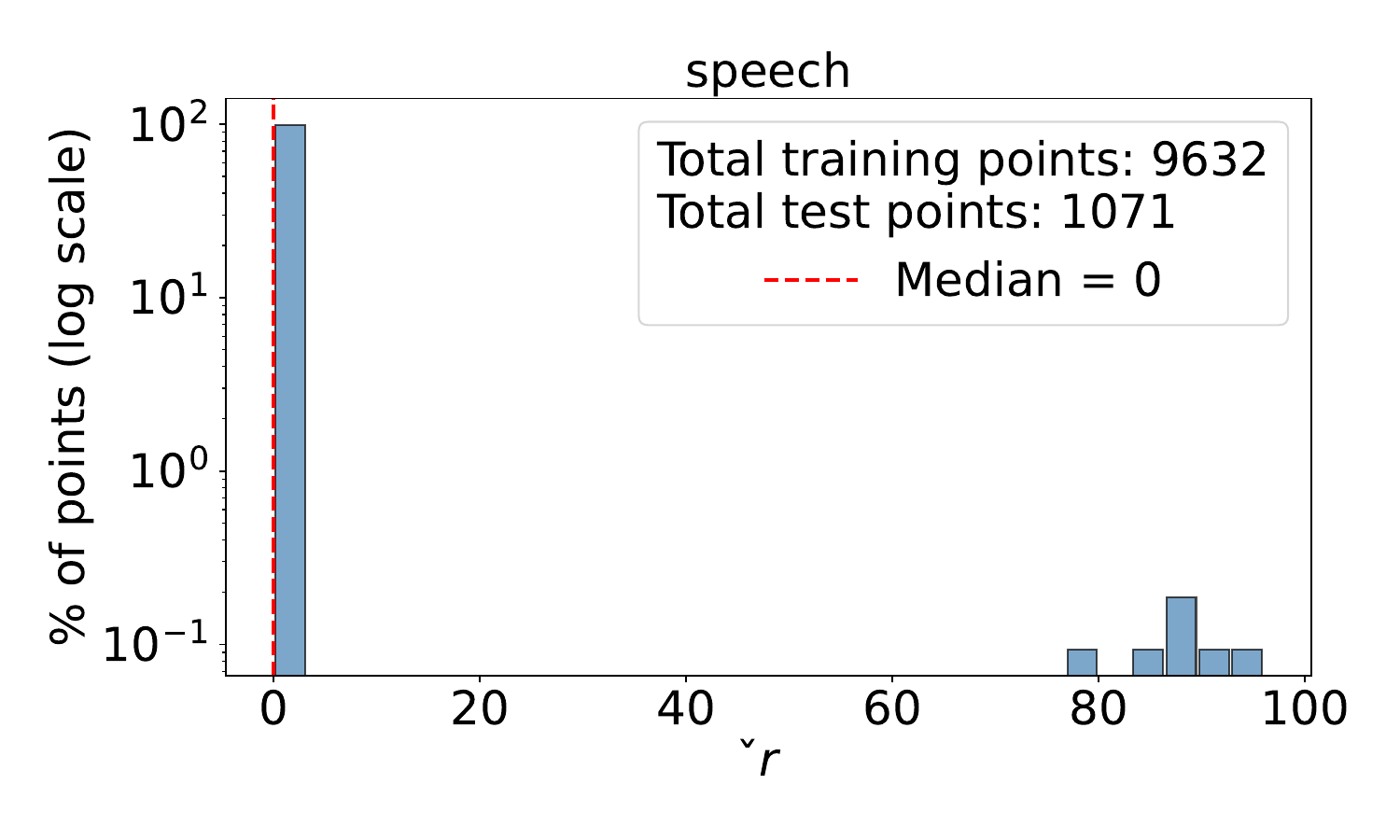}  
    \caption{Histogram of $\lowrob$ for Speech (BOW) dataset.}  
    \label{fig:lb_hist_speech}  
\end{figure}  
  
\begin{figure}[!htbp]  
    \centering  
    \includegraphics[width=0.45\textwidth]{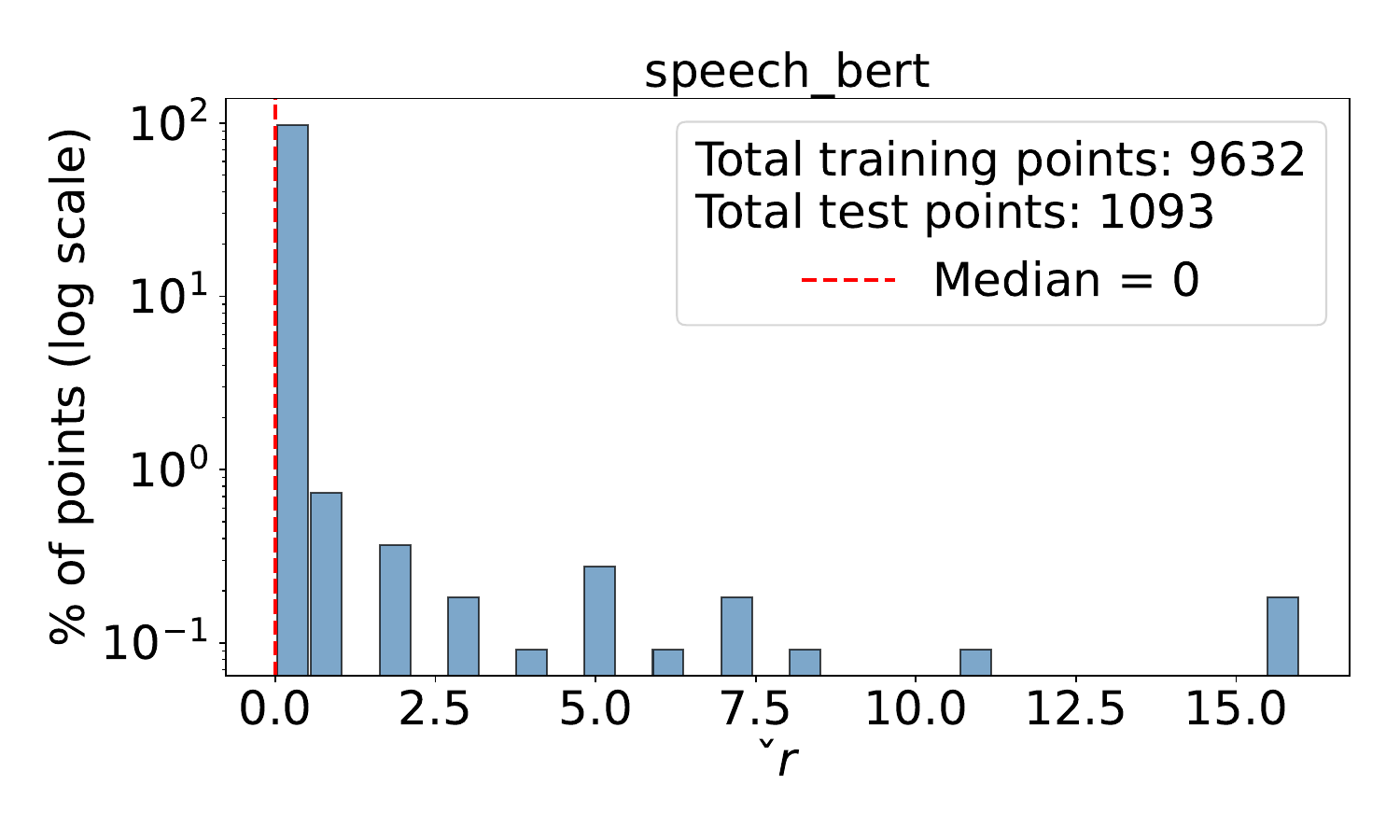}  
    \caption{Histogram of $\lowrob$ for Speech (BERT) dataset.}  
    \label{fig:lb_hist_speech_bert}  
\end{figure}  

\begin{figure}[!htbp]  
    \centering  
    \includegraphics[width=0.45\textwidth]{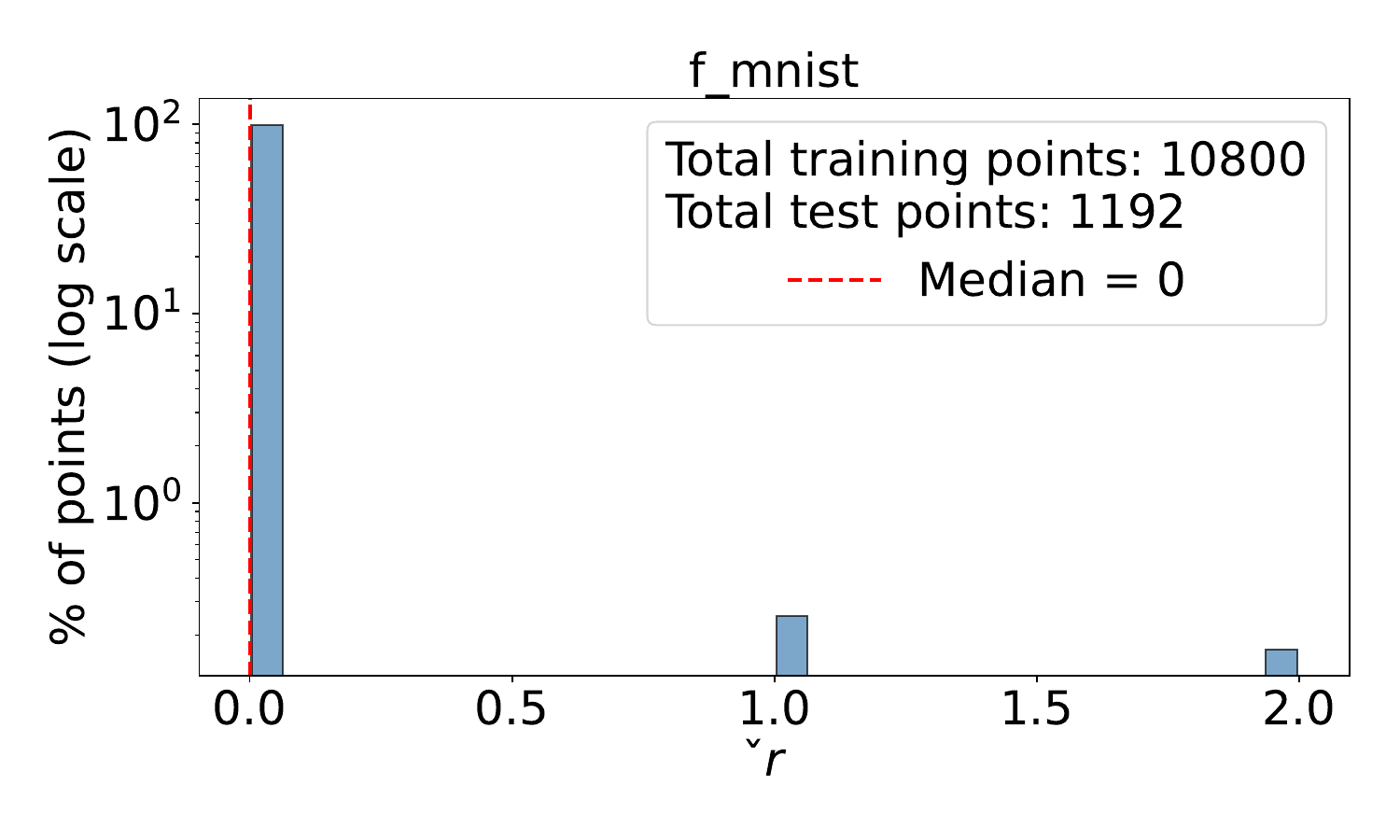}  
    \caption{Histogram of $\lowrob$ for Fashion MNIST dataset.}  
    \label{fig:lb_hist_f_mnist}  
\end{figure}

\begin{figure}[!htbp]  
    \centering  
    \includegraphics[width=0.45\textwidth]{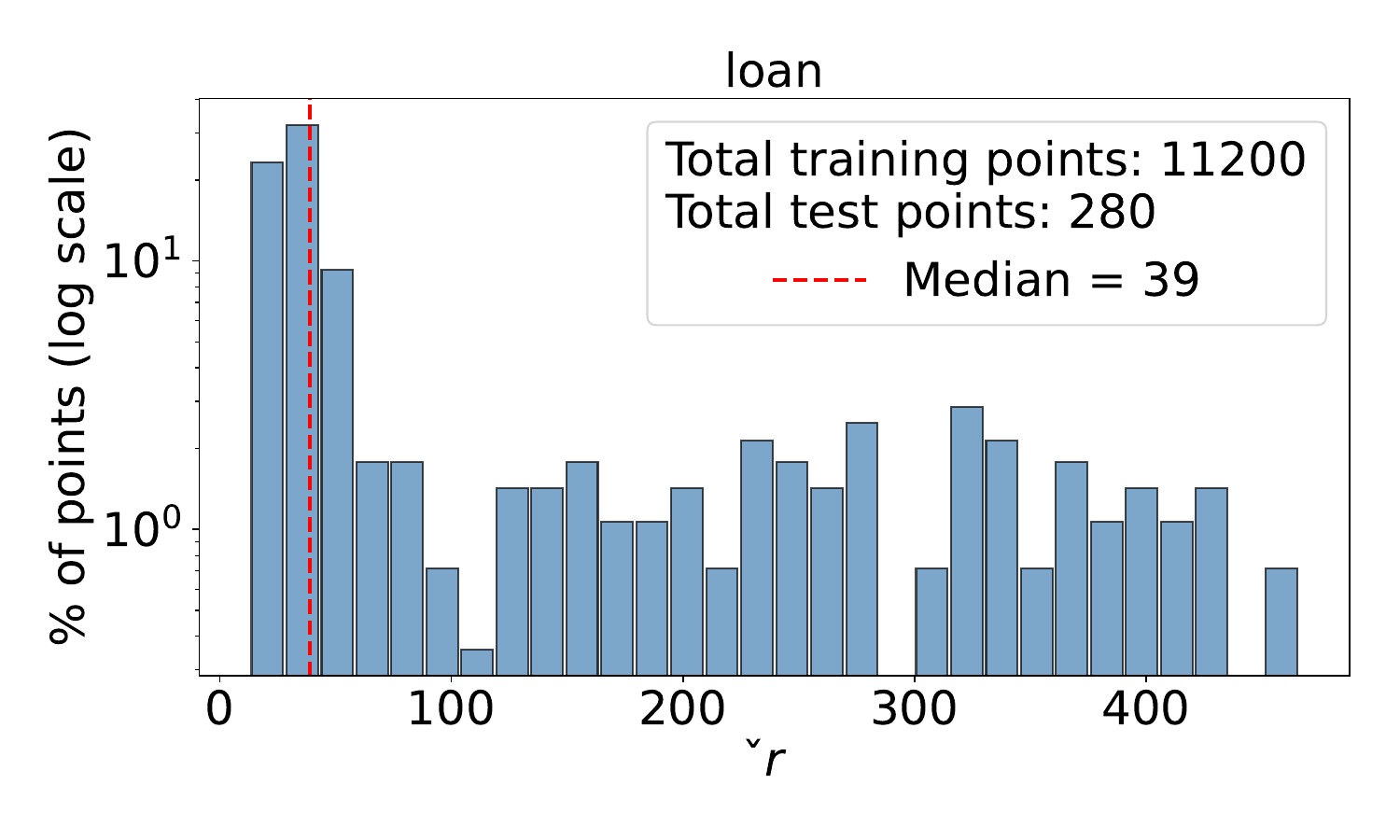}  
    \caption{Histogram of $\lowrob$ for Loan (BOW) dataset.}  
    \label{fig:lb_hist_loan}  
\end{figure}  

\begin{figure}[!htbp]  
    \centering  
    \includegraphics[width=0.45\textwidth]{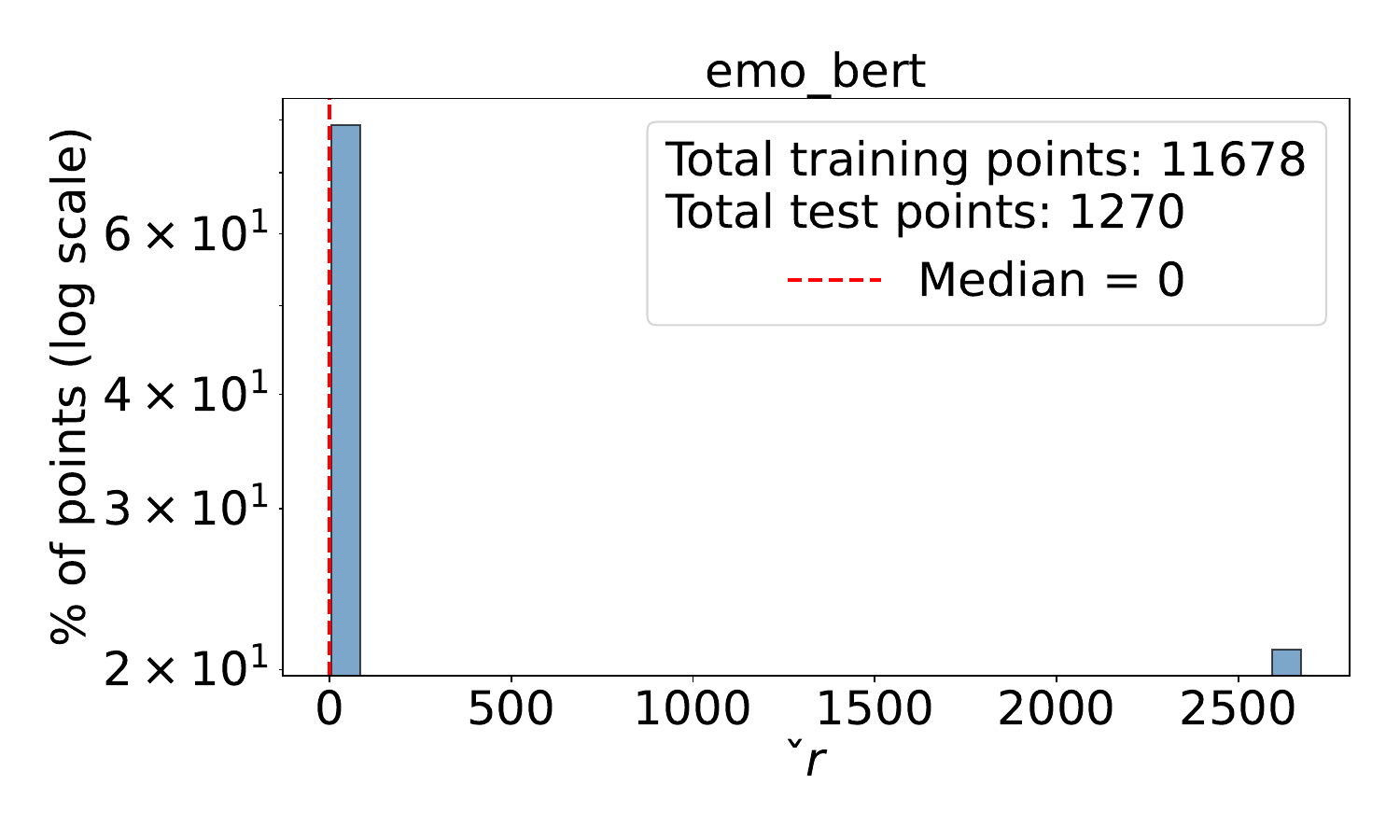}  
    \caption{Histogram of $\lowrob$ for Emo (BERT) dataset.}  
    \label{fig:lb_hist_emo_bert}  
\end{figure}  
  
\begin{figure}[!htbp]  
    \centering  
    \includegraphics[width=0.45\textwidth]{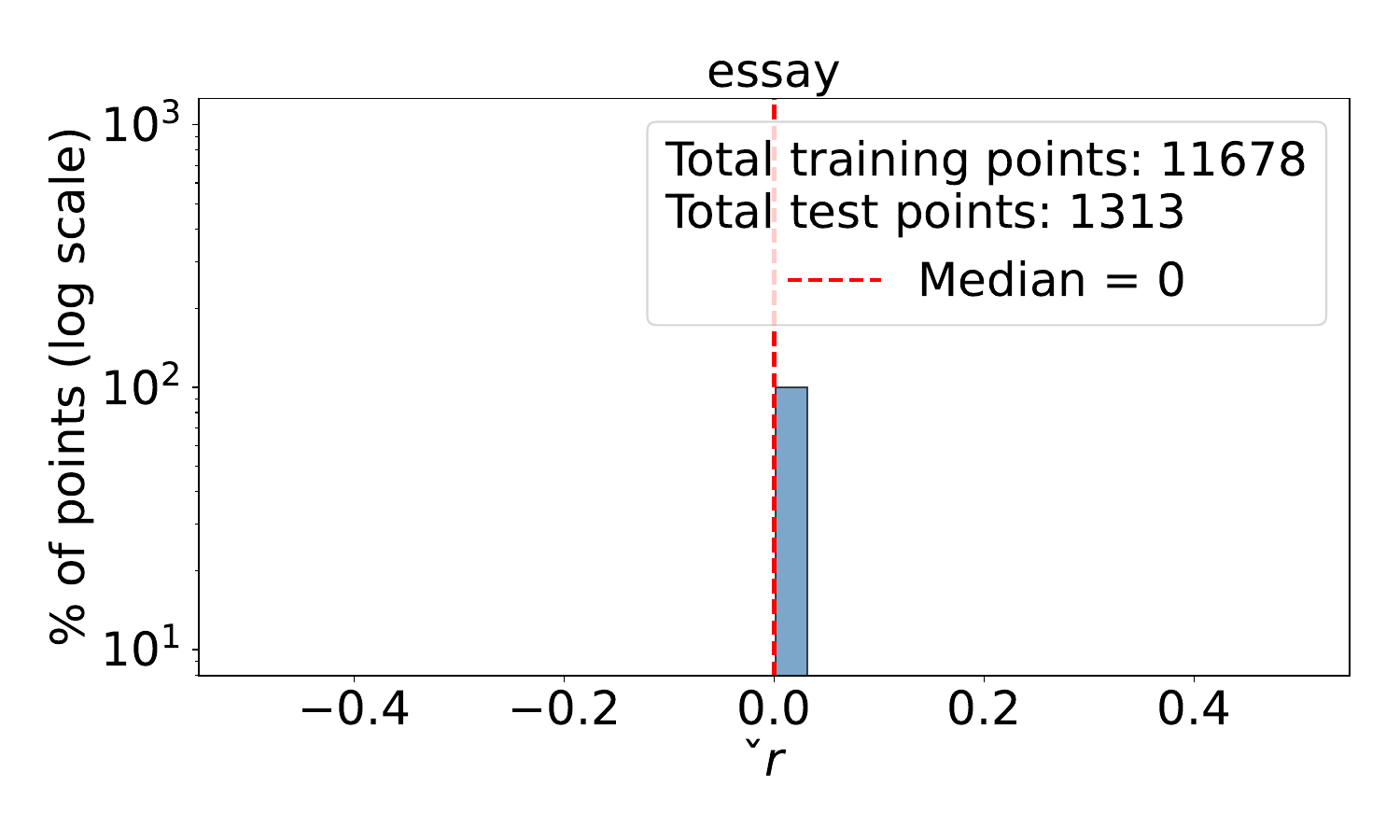}  
    \caption{Histogram of $\lowrob$ for Essays (BOW) dataset.}  
    \label{fig:lb_hist_essay}  
\end{figure}  
  
\begin{figure}[!htbp]  
    \centering  
    \includegraphics[width=0.45\textwidth]{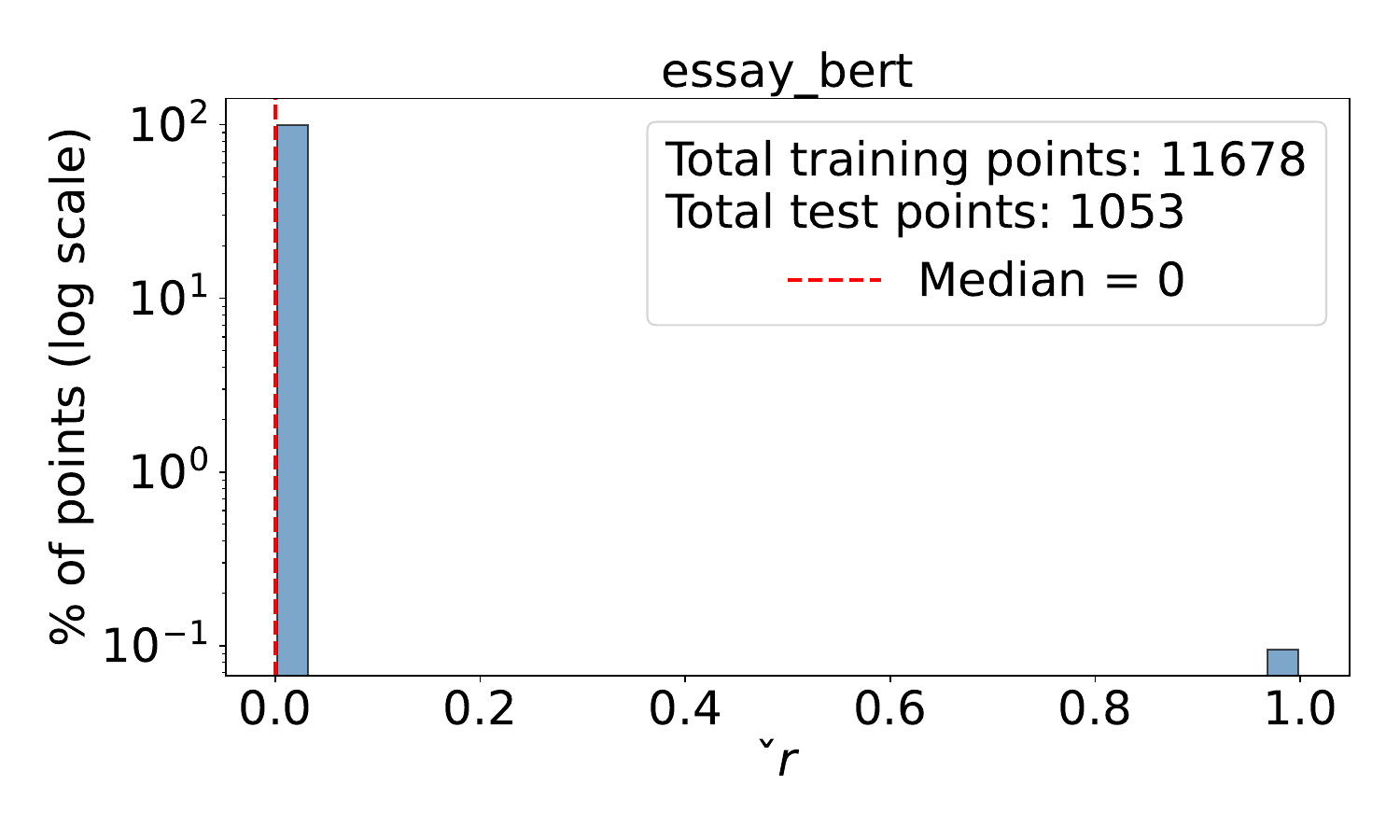}  
    \caption{Histogram of $\lowrob$ for Essays (BERT) dataset.}  
    \label{fig:lb_hist_essay_bert}  
\end{figure}  
  
\begin{figure}[!htbp]  
    \centering  
    \includegraphics[width=0.45\textwidth]{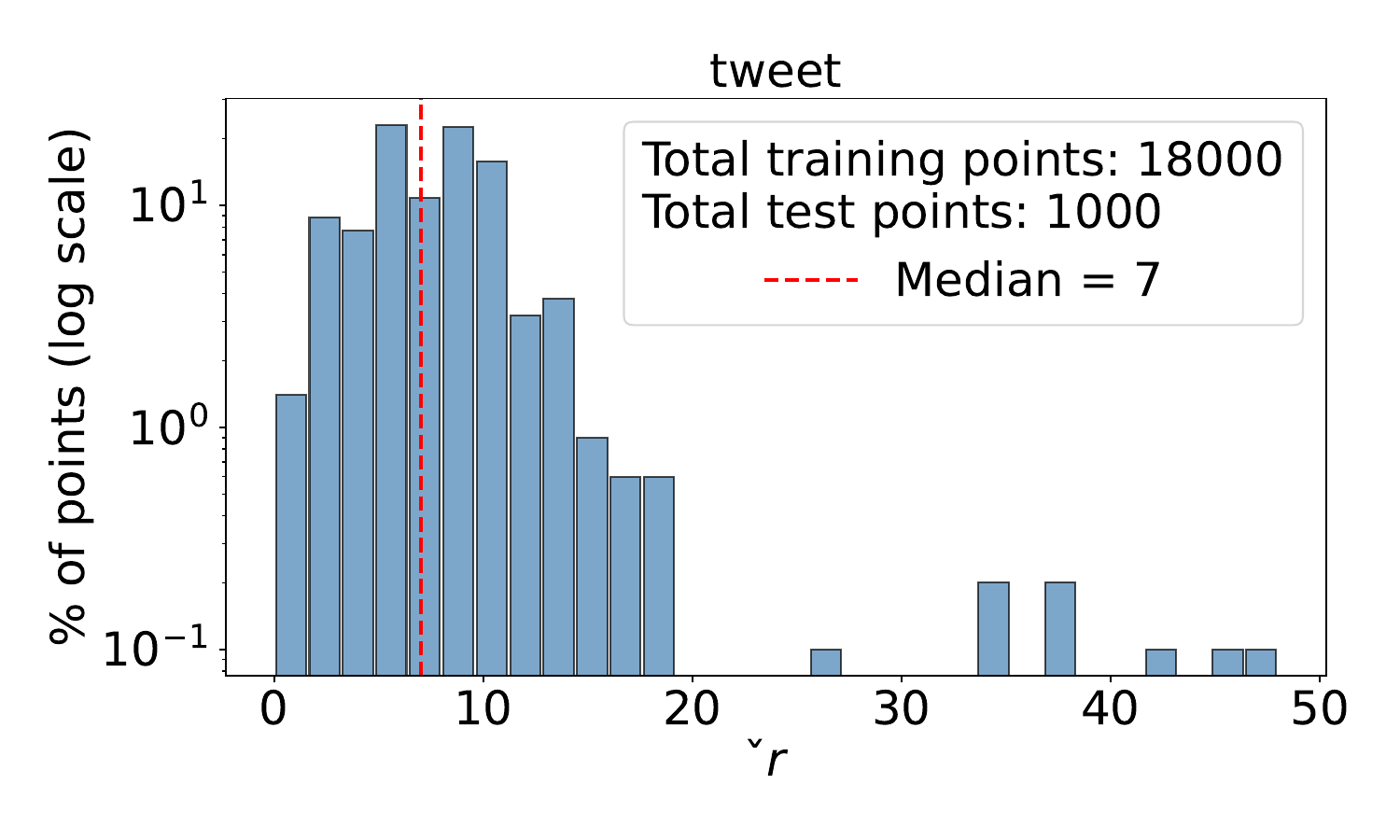}  
    \caption{Histogram of $\lowrob$ for Tweet (BOW) dataset.}  
    \label{fig:lb_hist_tweet}  
\end{figure}  
  
\begin{figure}[!htbp]  
    \centering  
    \includegraphics[width=0.45\textwidth]{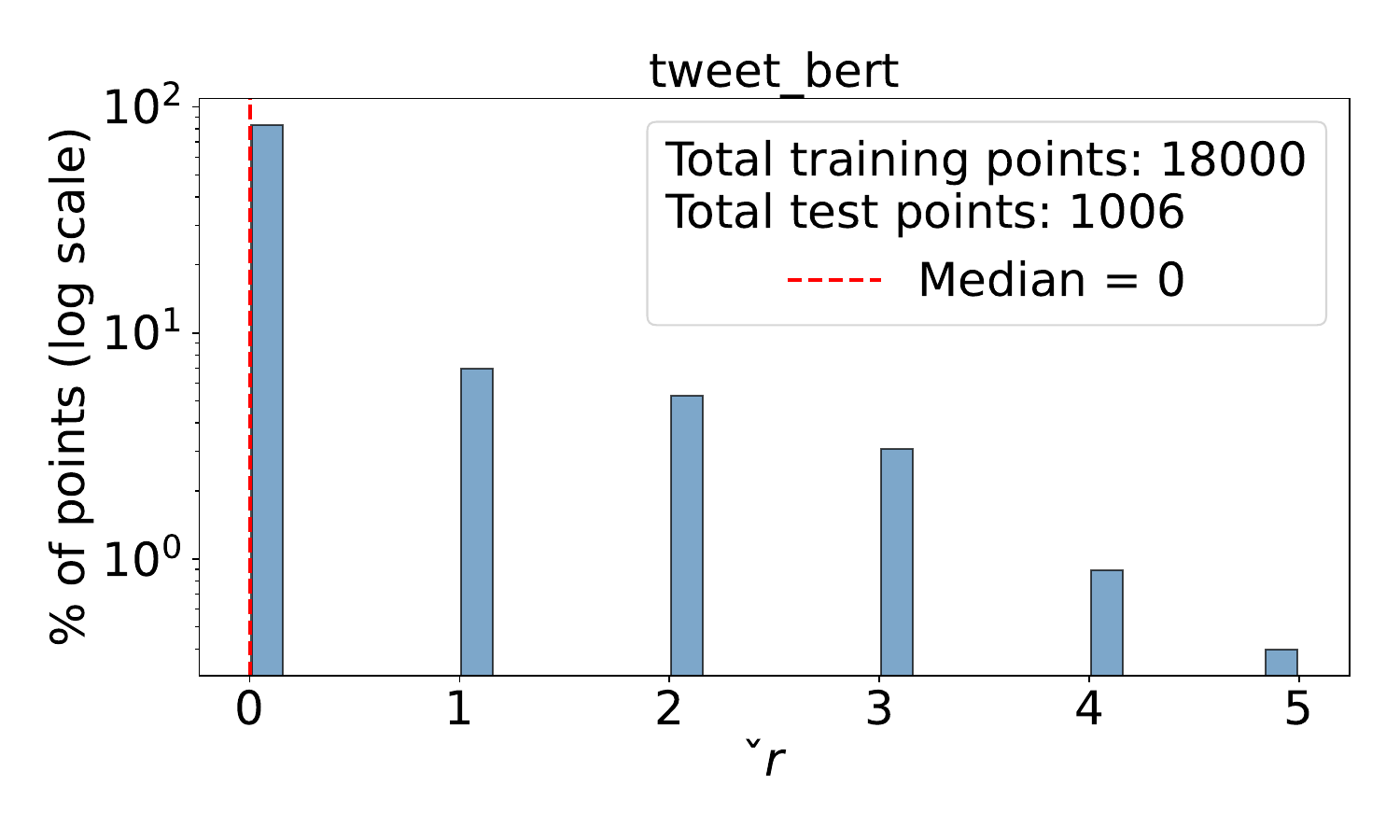}  
    \caption{Histogram of $\lowrob$ for Tweet (BERT) dataset.}  
    \label{fig:lb_hist_tweet_bert}  
\end{figure}  

\begin{figure}[!htbp]  
    \centering  
    \includegraphics[width=0.45\textwidth]{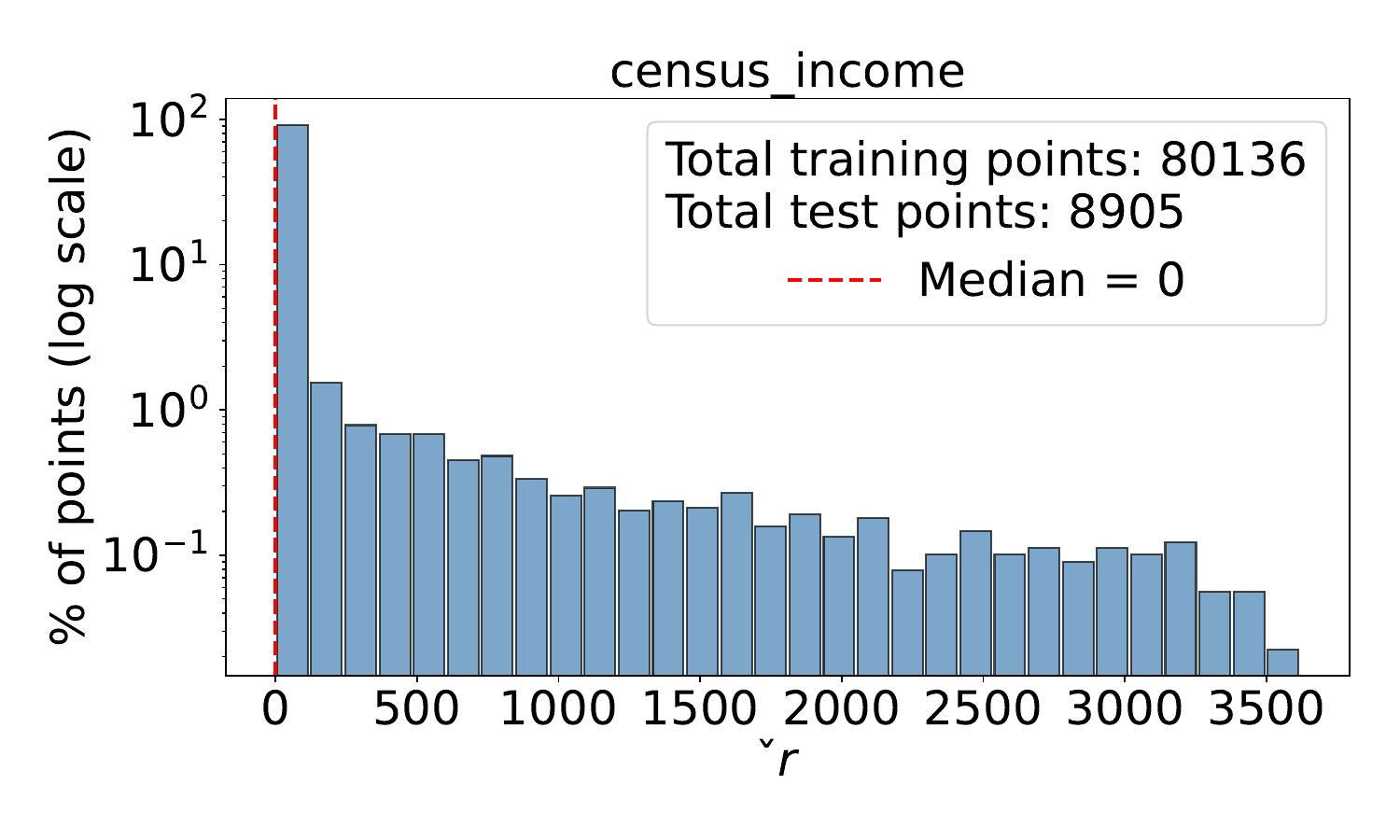}  
    \caption{Histogram of $\lowrob$ for Census Income dataset.}  
    \label{fig:lb_hist_census_income}  
\end{figure}  
 
\clearpage

\subsection{Accuracy vs $\rho$ for all datasets}
This section presents comparison of average accuracy vs $\rho$ (average likelihood of getting desired classification for test points) for all datasets.

\begin{figure*}[!htbp]  
    \centering  
    \includegraphics[width=\textwidth]{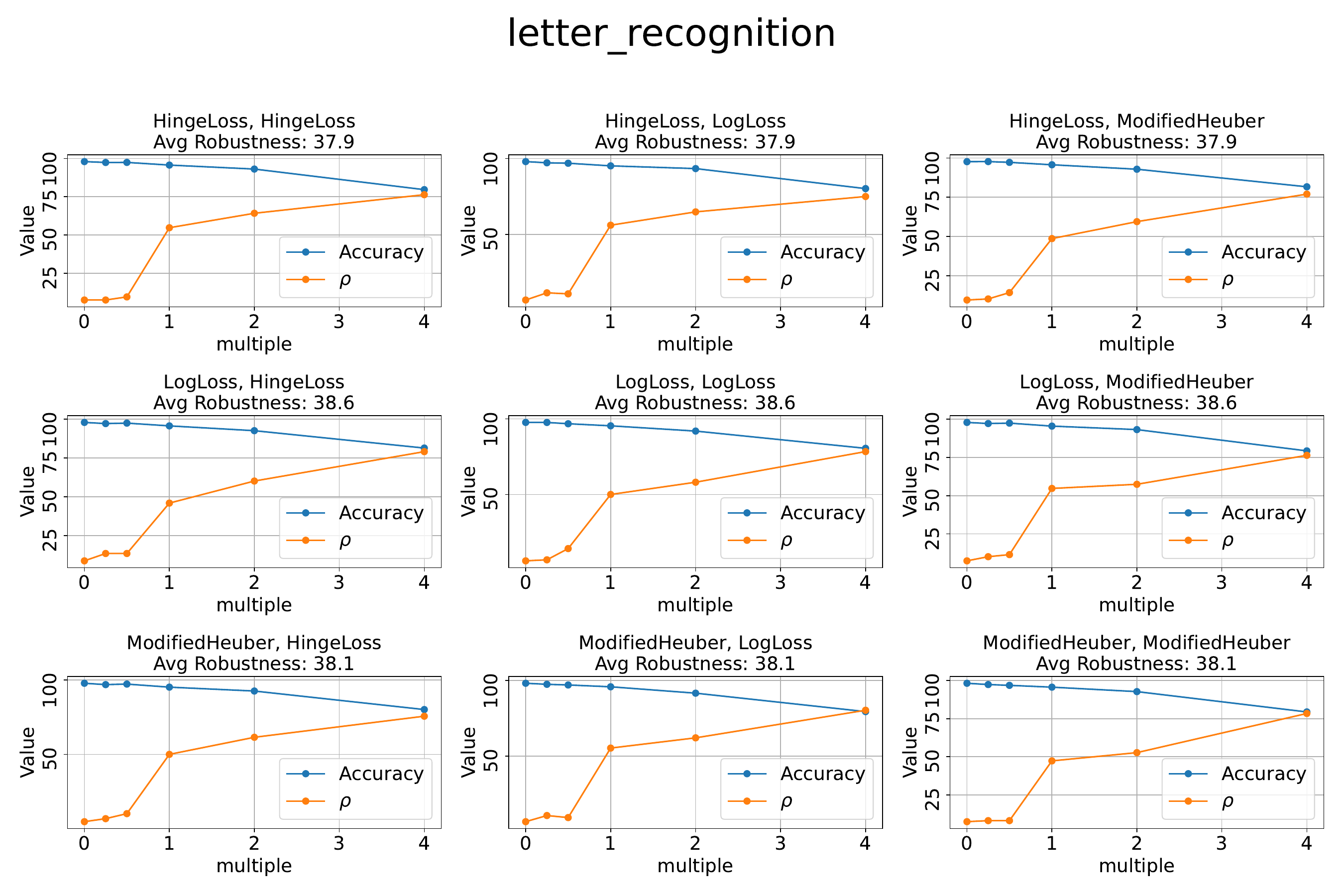}  
    \caption{Comparison of average accuracy vs $\rho$  for Letter Recognition dataset with $\{0, \frac{\uprob}{4}, \frac{\uprob}{2}, \uprob, 2\uprob, 4\uprob\}$ and loss functions.}  
    \label{fig:grid_letter}  
\end{figure*}  

\begin{figure*}[!htbp]  
    \centering  
    \includegraphics[width=\textwidth]{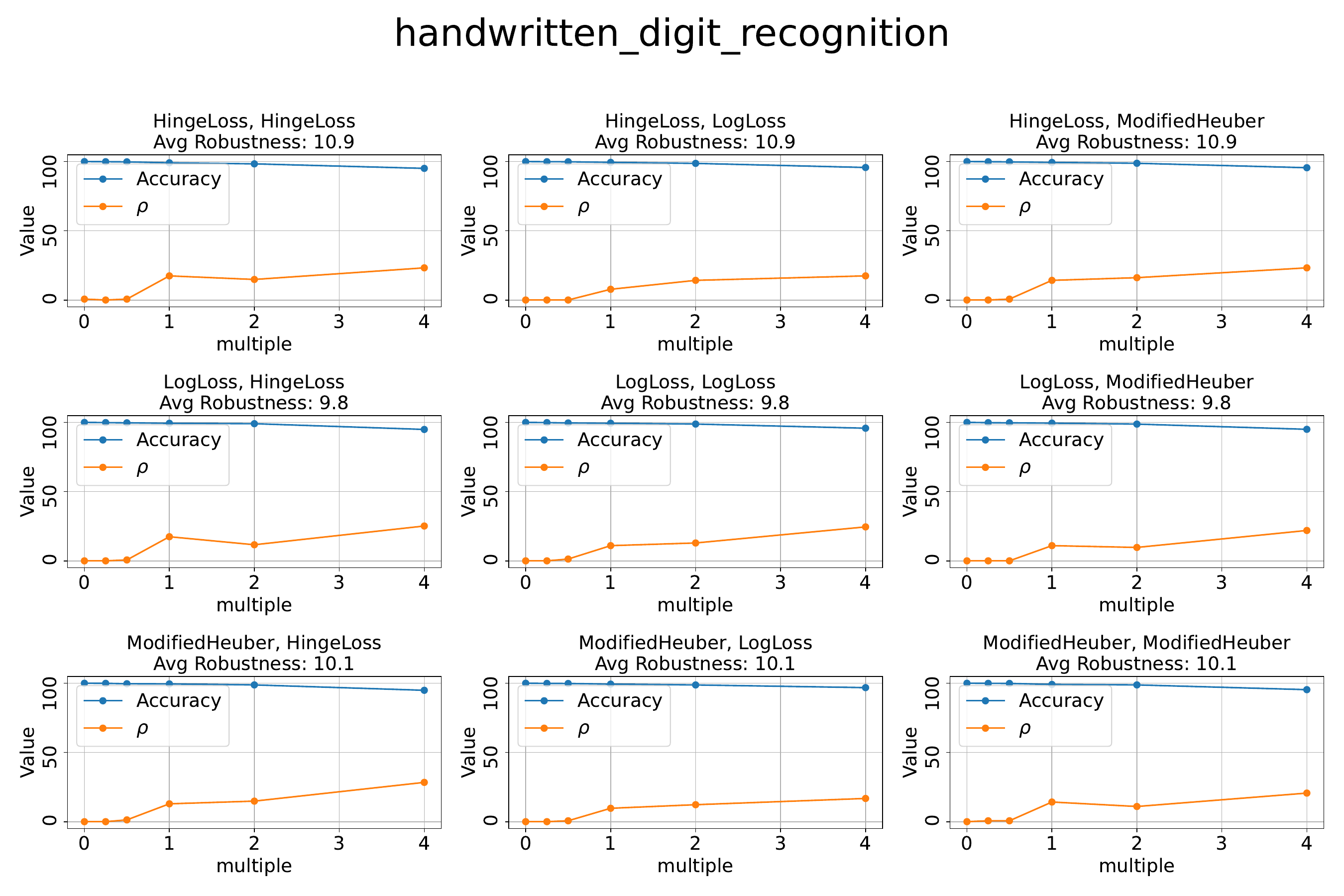}  
    \caption{Comparison of average accuracy vs $\rho$  for Digits Recognition dataset with $\{0, \frac{\uprob}{4}, \frac{\uprob}{2}, \uprob, 2\uprob, 4\uprob\}$ and loss functions.}  
    \label{fig:grid_handwritten}  
\end{figure*}

\begin{figure*}[!htbp]  
    \centering  
    \includegraphics[width=\textwidth]{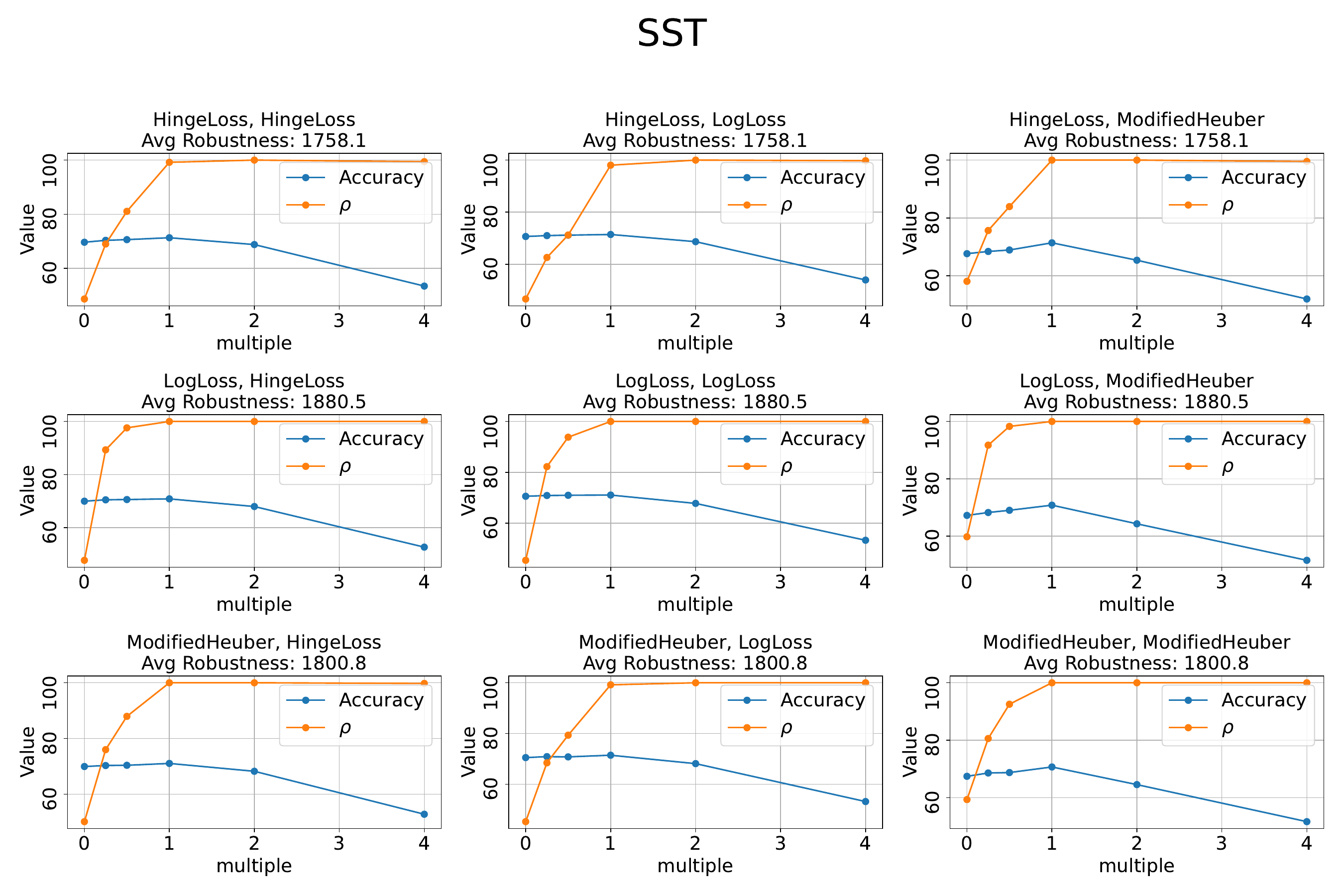}  
    \caption{Comparison of average accuracy vs $\rho$  for SST (BOW) dataset with $\{0, \frac{\uprob}{4}, \frac{\uprob}{2}, \uprob, 2\uprob, 4\uprob\}$ and loss functions.}  
    \label{fig:grid_SST}  
\end{figure*}

\begin{figure*}[!htbp]  
    \centering  
    \includegraphics[width=\textwidth]{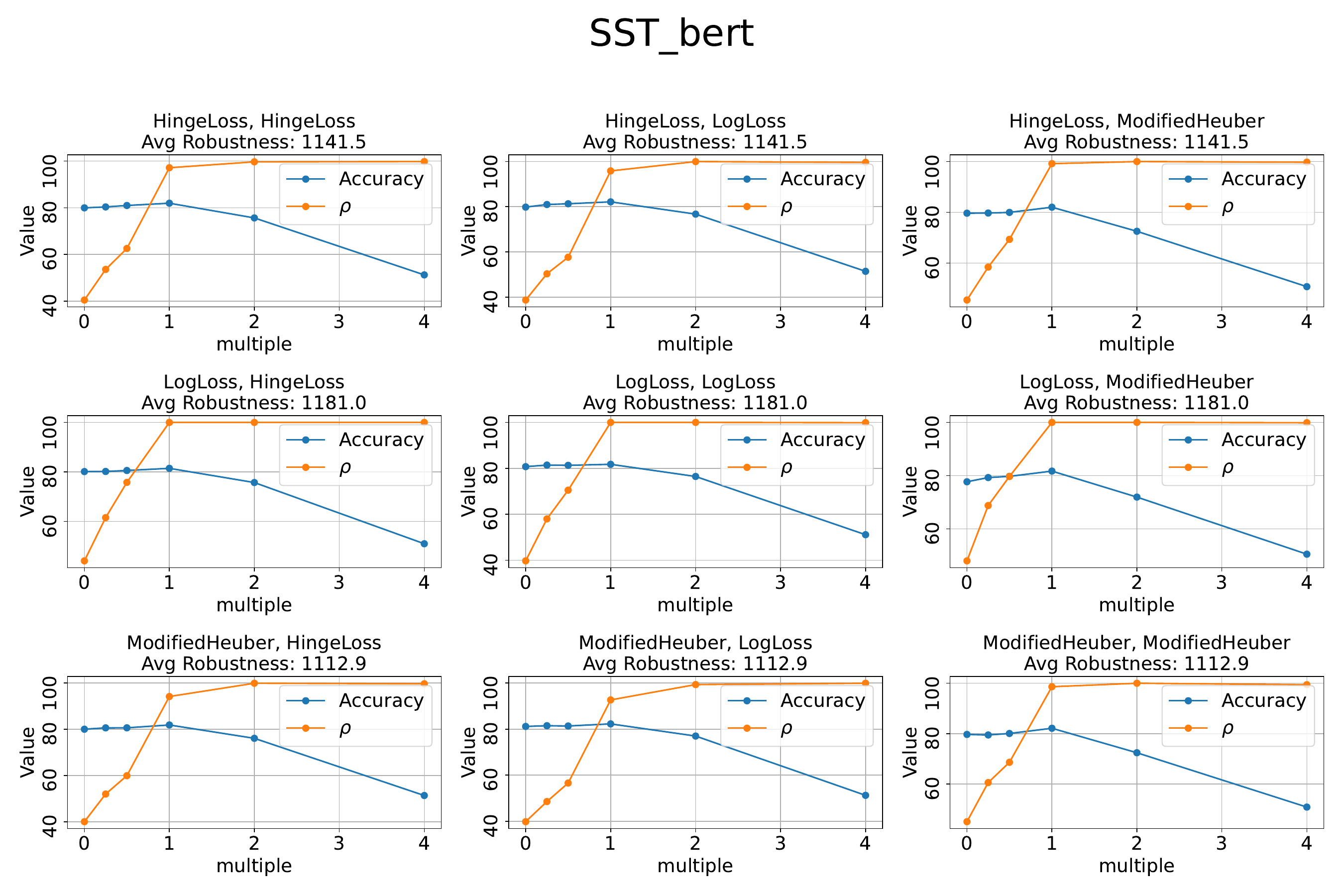}  
    \caption{Comparison of average accuracy vs $\rho$  for SST (BERT) dataset with $\{0, \frac{\uprob}{4}, \frac{\uprob}{2}, \uprob, 2\uprob, 4\uprob\}$ and loss functions.}  
    \label{fig:grid_SST_bert}  
\end{figure*}

\begin{figure*}[!htbp]  
    \centering  
    \includegraphics[width=\textwidth]{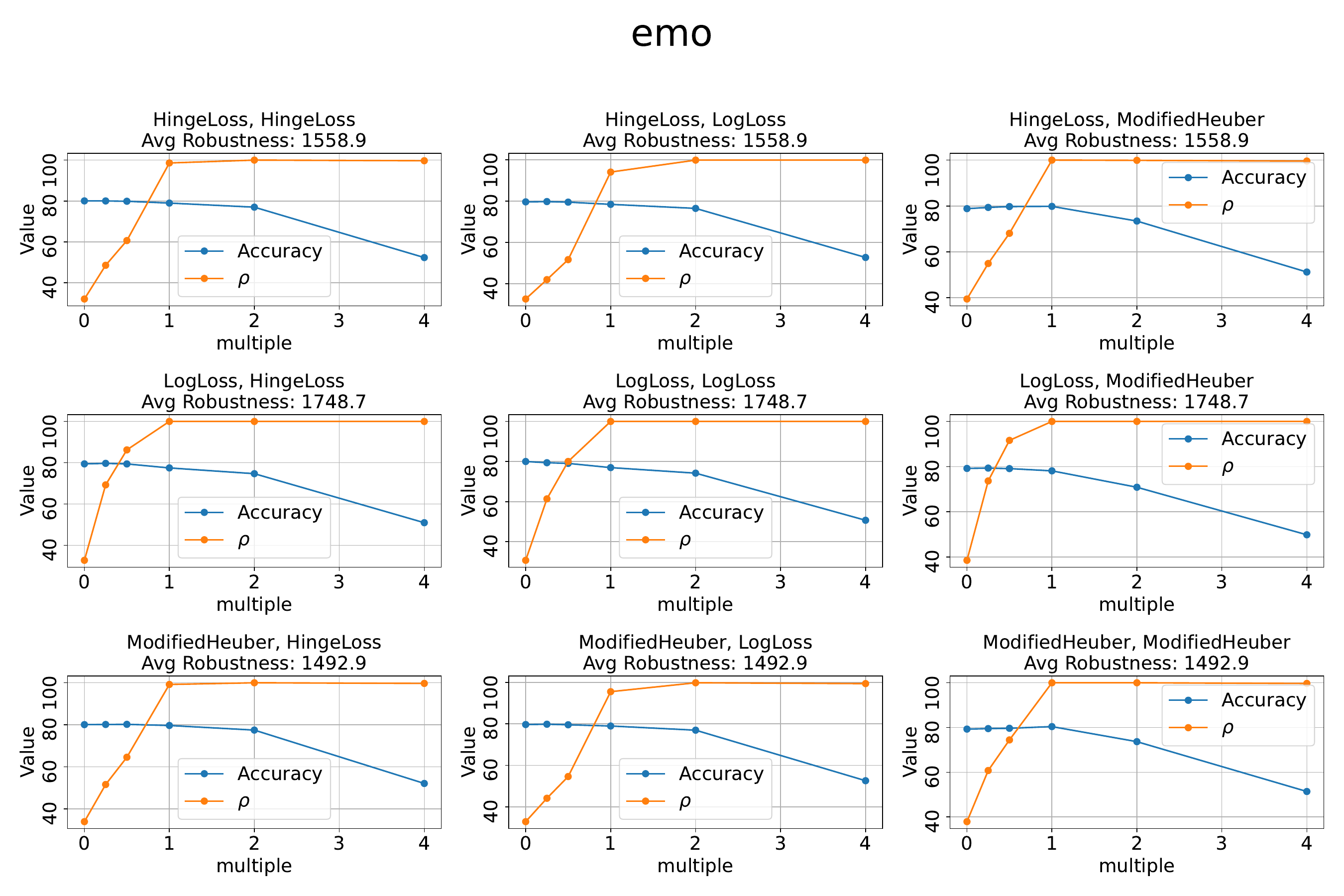}  
    \caption{Comparison of average accuracy vs $\rho$  for Emo (BOW) dataset with $\{0, \frac{\uprob}{4}, \frac{\uprob}{2}, \uprob, 2\uprob, 4\uprob\}$ and loss functions.}  
    \label{fig:grid_emo}  
\end{figure*}

\begin{figure*}[!htbp]  
    \centering  
    \includegraphics[width=\textwidth]{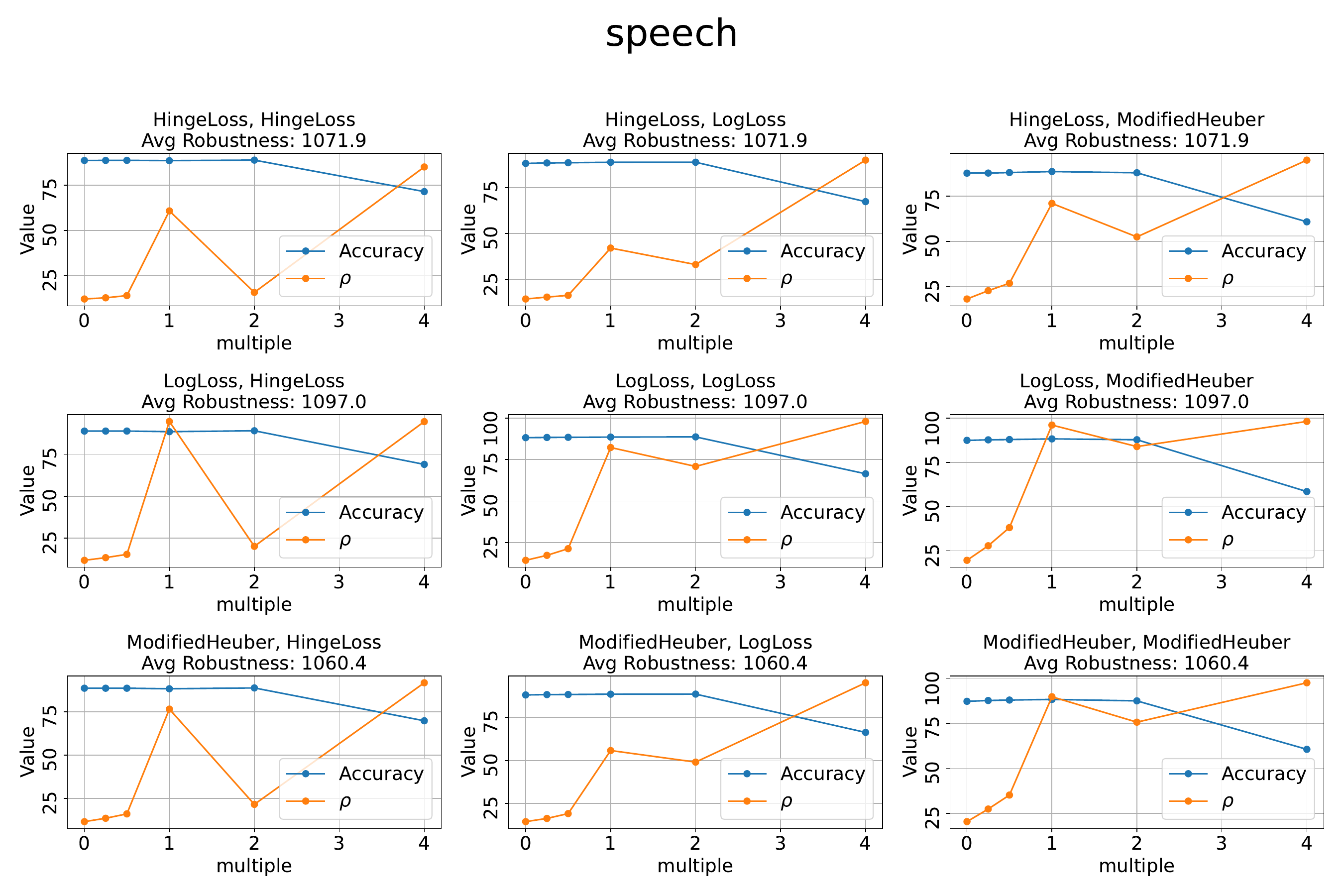}  
    \caption{Comparison of average accuracy vs $\rho$  for Speech (BOW) dataset with $\{0, \frac{\uprob}{4}, \frac{\uprob}{2}, \uprob, 2\uprob, 4\uprob\}$ and loss functions.}  
    \label{fig:grid_speech}  
\end{figure*}  

\begin{figure*}[!htbp]  
    \centering  
    \includegraphics[width=\textwidth]{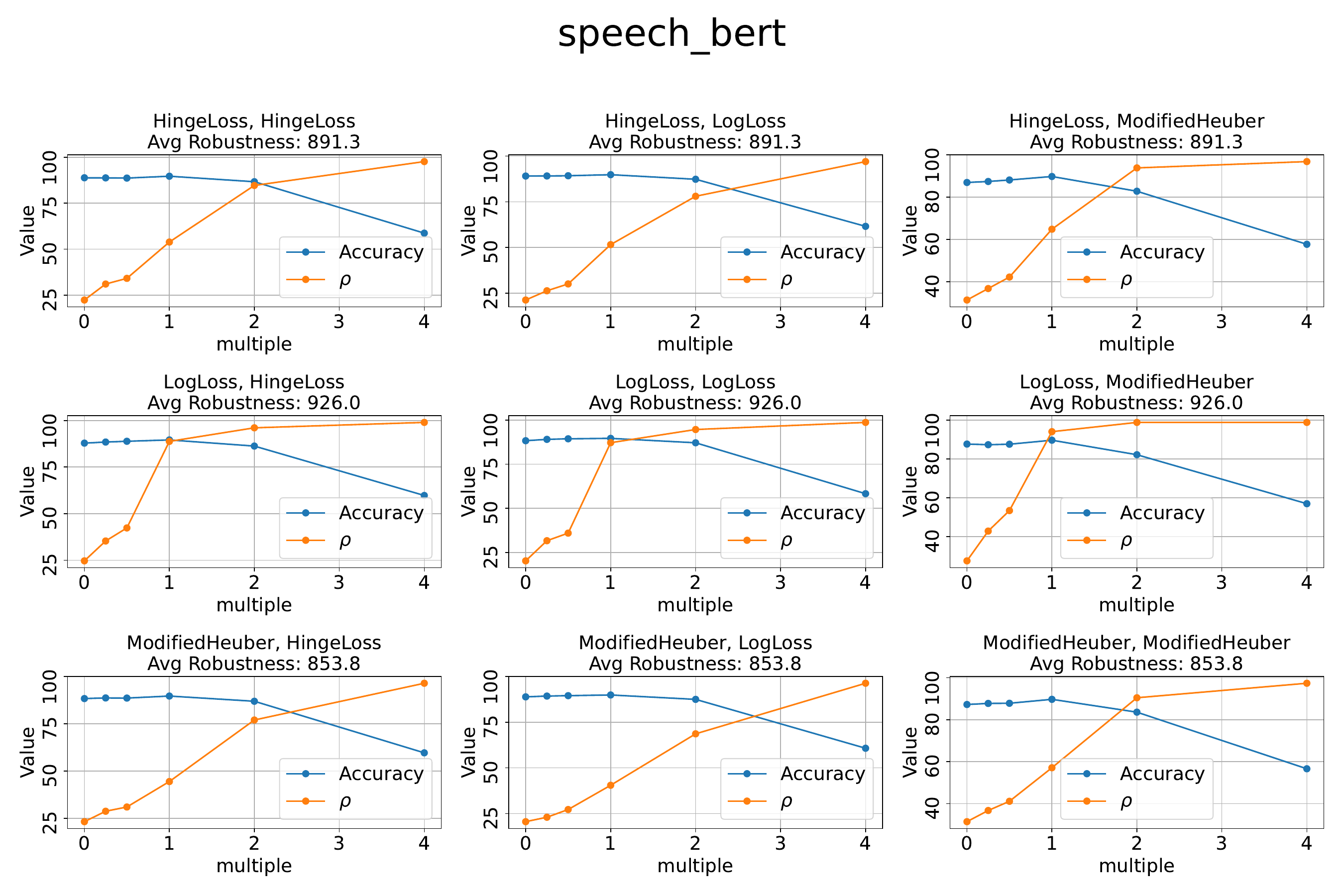}  
    \caption{Comparison of average accuracy vs $\rho$  for Speech (BERT) dataset with $\{0, \frac{\uprob}{4}, \frac{\uprob}{2}, \uprob, 2\uprob, 4\uprob\}$ and loss functions.}  
    \label{fig:grid_speech_bert}  
\end{figure*}  
  
\begin{figure*}[!htbp]  
    \centering  
    \includegraphics[width=\textwidth]{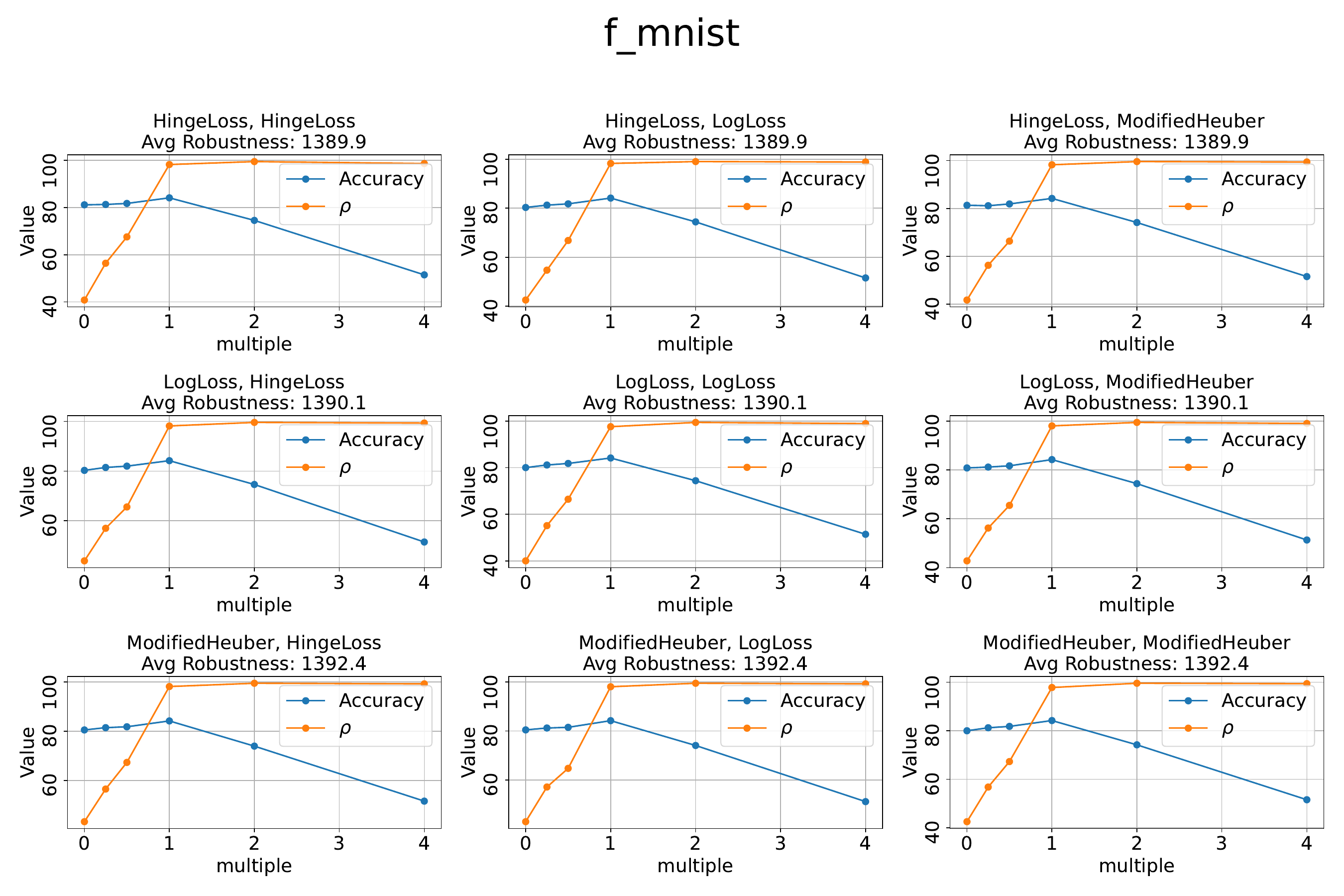}  
    \caption{Comparison of average accuracy vs $\rho$  for Fashion MNIST dataset with $\{0, \frac{\uprob}{4}, \frac{\uprob}{2}, \uprob, 2\uprob, 4\uprob\}$ and loss functions.}  
    \label{fig:grid_f_mnist}  
\end{figure*}  

\begin{figure*}[!htbp]  
    \centering  
    \includegraphics[width=\textwidth]{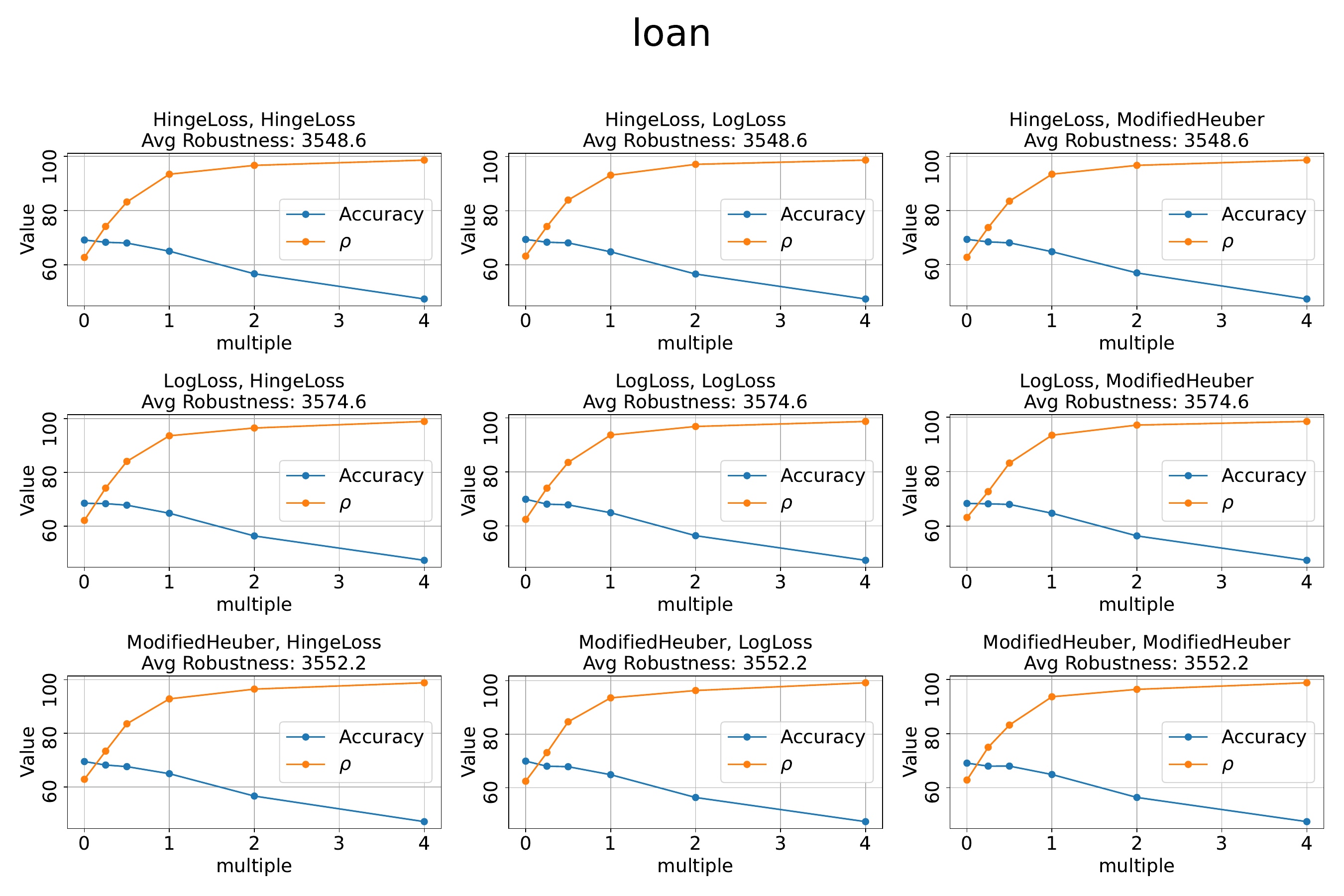}  
    \caption{Comparison of average accuracy vs $\rho$  for Loan (BOW) dataset with $\{0, \frac{\uprob}{4}, \frac{\uprob}{2}, \uprob, 2\uprob, 4\uprob\}$ and loss functions.}  
    \label{fig:grid_loan}  
\end{figure*}  

\begin{figure*}[!htbp]  
    \centering  
    \includegraphics[width=\textwidth]{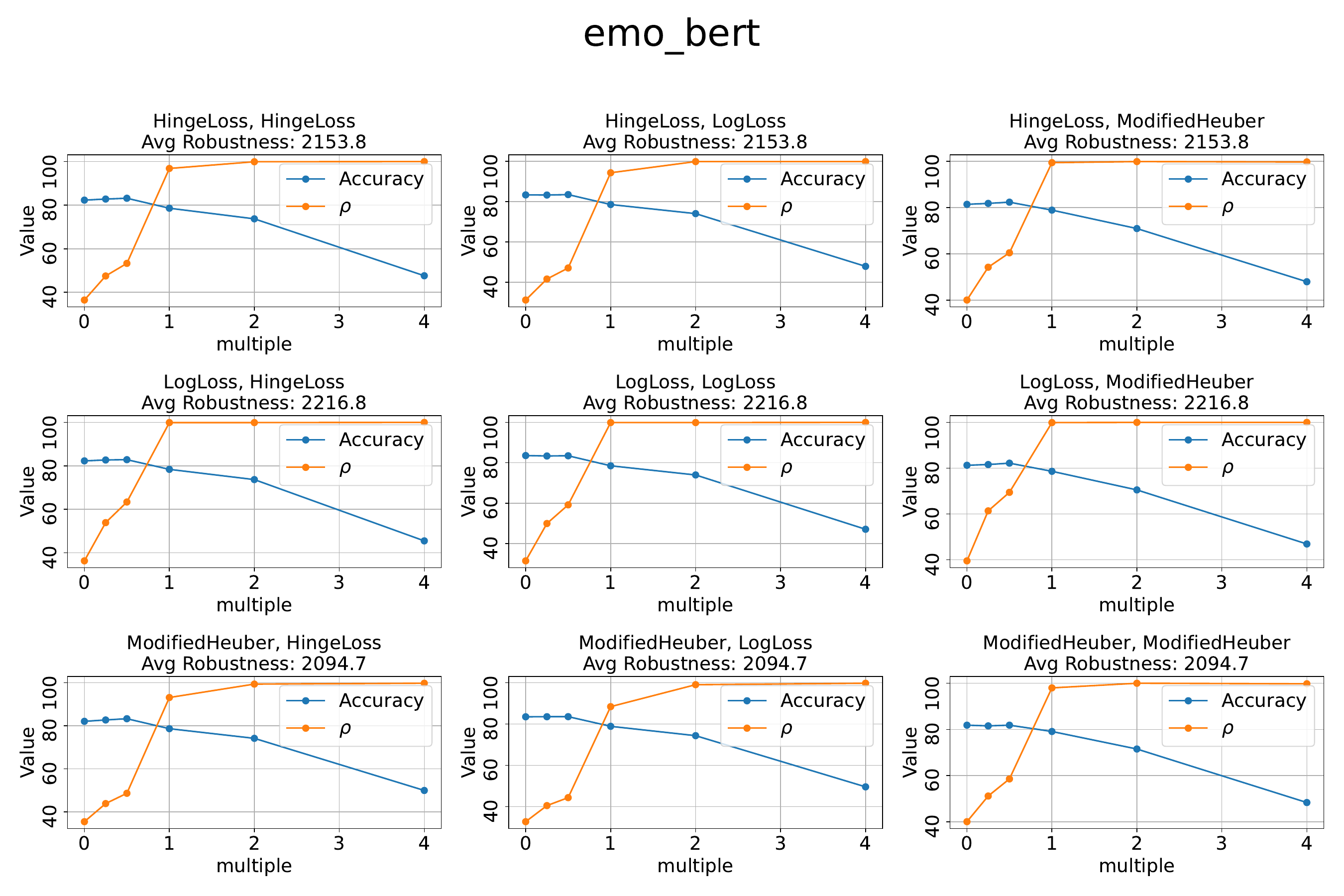}  
    \caption{Comparison of average accuracy vs $\rho$  for Emo (BERT) dataset with $\{0, \frac{\uprob}{4}, \frac{\uprob}{2}, \uprob, 2\uprob, 4\uprob\}$ and loss functions.}  
    \label{fig:grid_emo_bert}  
\end{figure*}  

\begin{figure*}[!htbp]  
    \centering  
    \includegraphics[width=\textwidth]{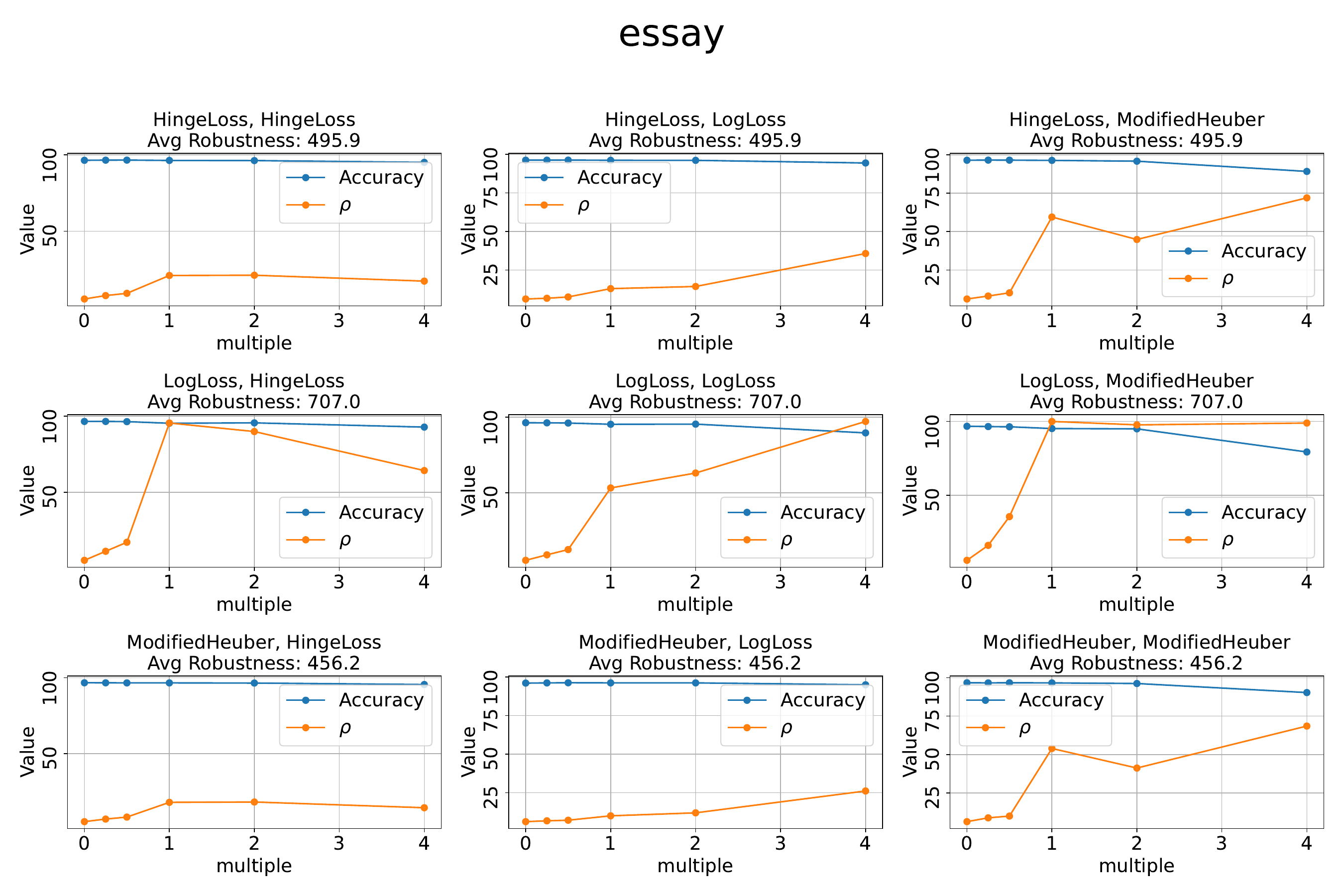}  
    \caption{Comparison of average accuracy vs $\rho$  for Essays (BOW) dataset with $\{0, \frac{\uprob}{4}, \frac{\uprob}{2}, \uprob, 2\uprob, 4\uprob\}$ and loss functions.}  
    \label{fig:grid_essay}  
\end{figure*}  

  \begin{figure*}[!htbp]  
    \centering  
    \includegraphics[width=\textwidth]{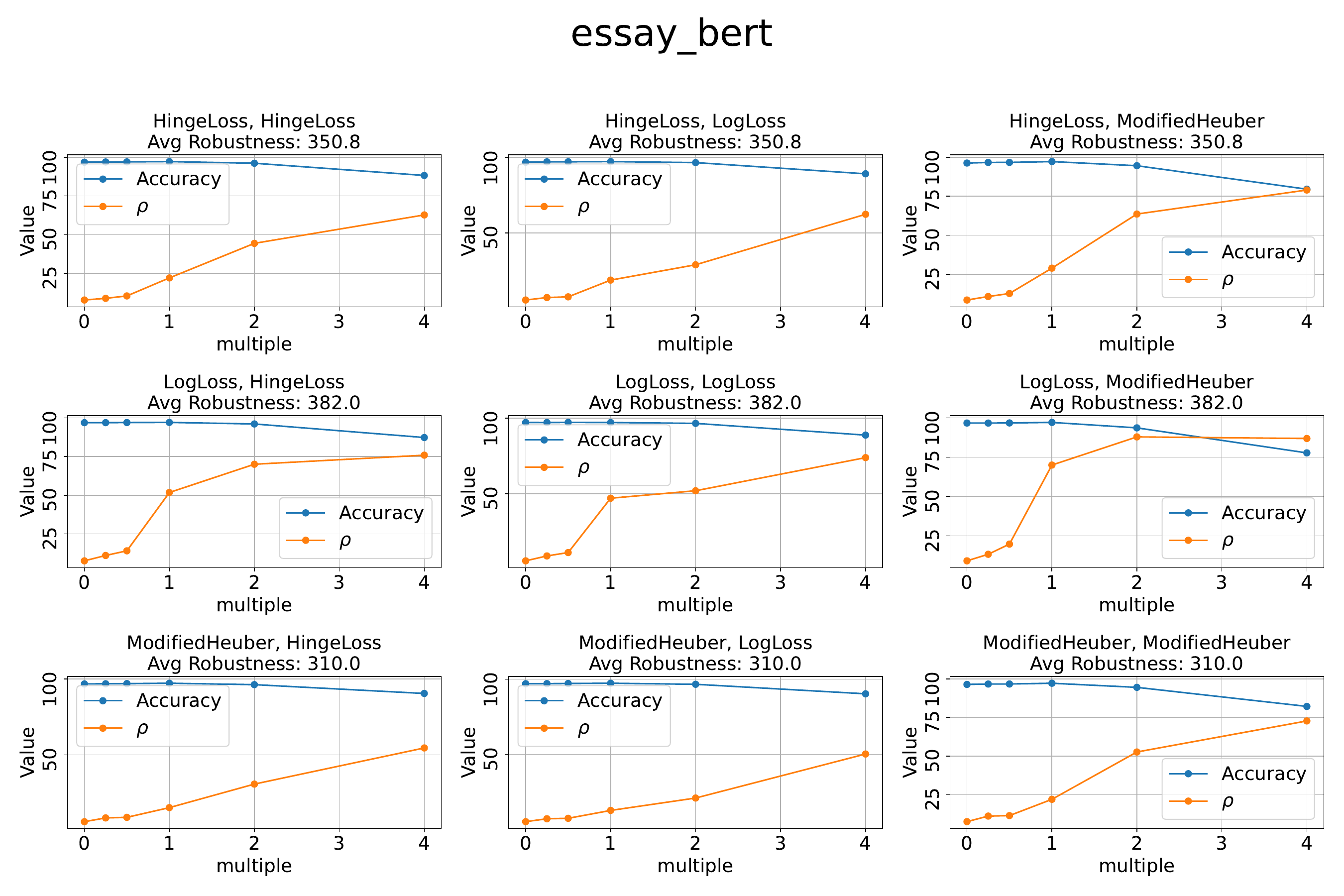}  
    \caption{Comparison of average accuracy vs $\rho$  for Essays (BERT) dataset with $\{0, \frac{\uprob}{4}, \frac{\uprob}{2}, \uprob, 2\uprob, 4\uprob\}$ and loss functions.}  
    \label{fig:grid_essay_bert}  
\end{figure*}

\begin{figure*}[!htbp]  
    \centering  
    \includegraphics[width=\textwidth]{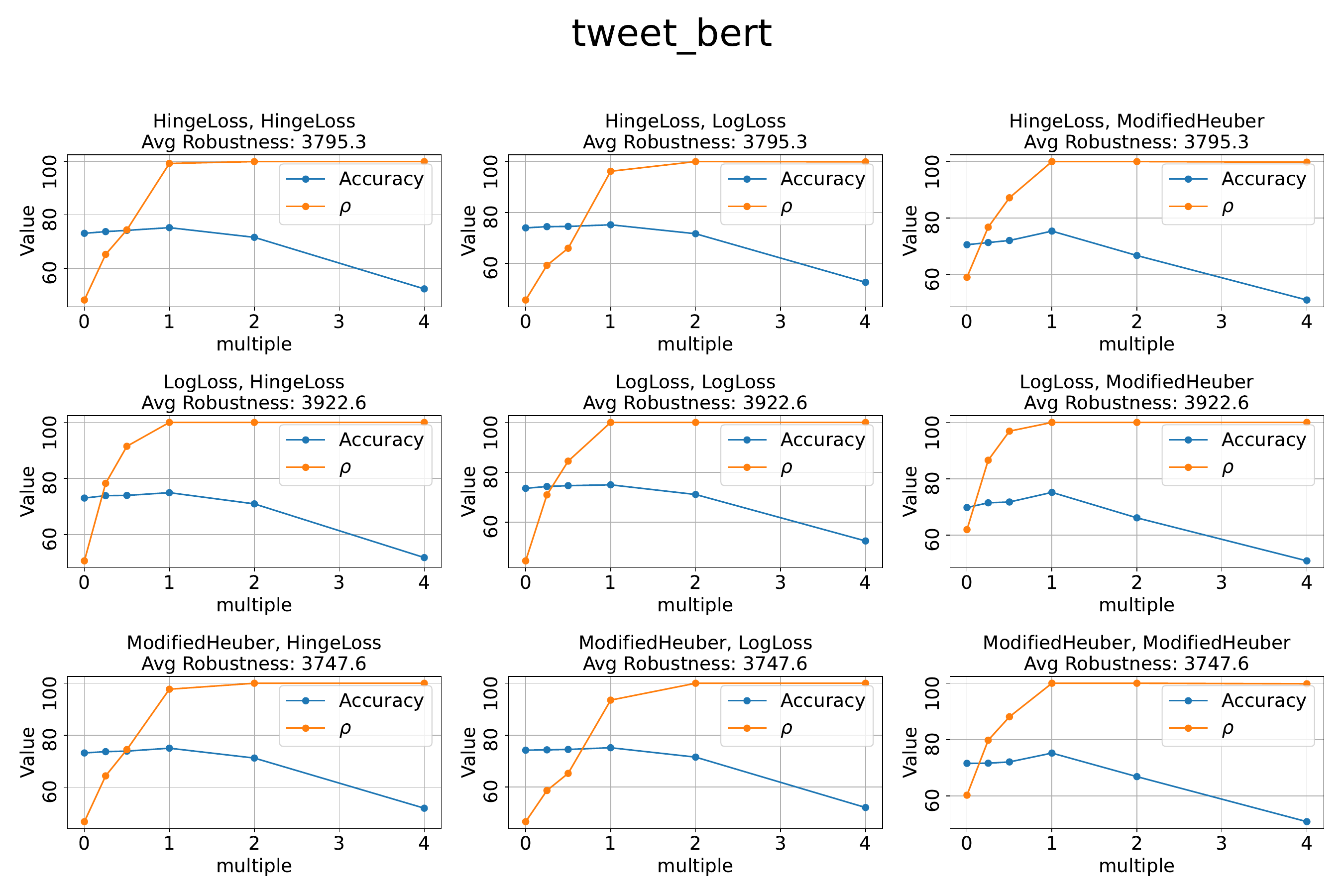}  
    \caption{Comparison of average accuracy vs $\rho$  for Tweet (BERT) dataset with $\{0, \frac{\uprob}{4}, \frac{\uprob}{2}, \uprob, 2\uprob, 4\uprob\}$ and loss functions.}  
    \label{fig:grid_tweet_bert}  
\end{figure*}  
  
\begin{figure*}[!htbp]  
    \centering  
    \includegraphics[width=\textwidth]{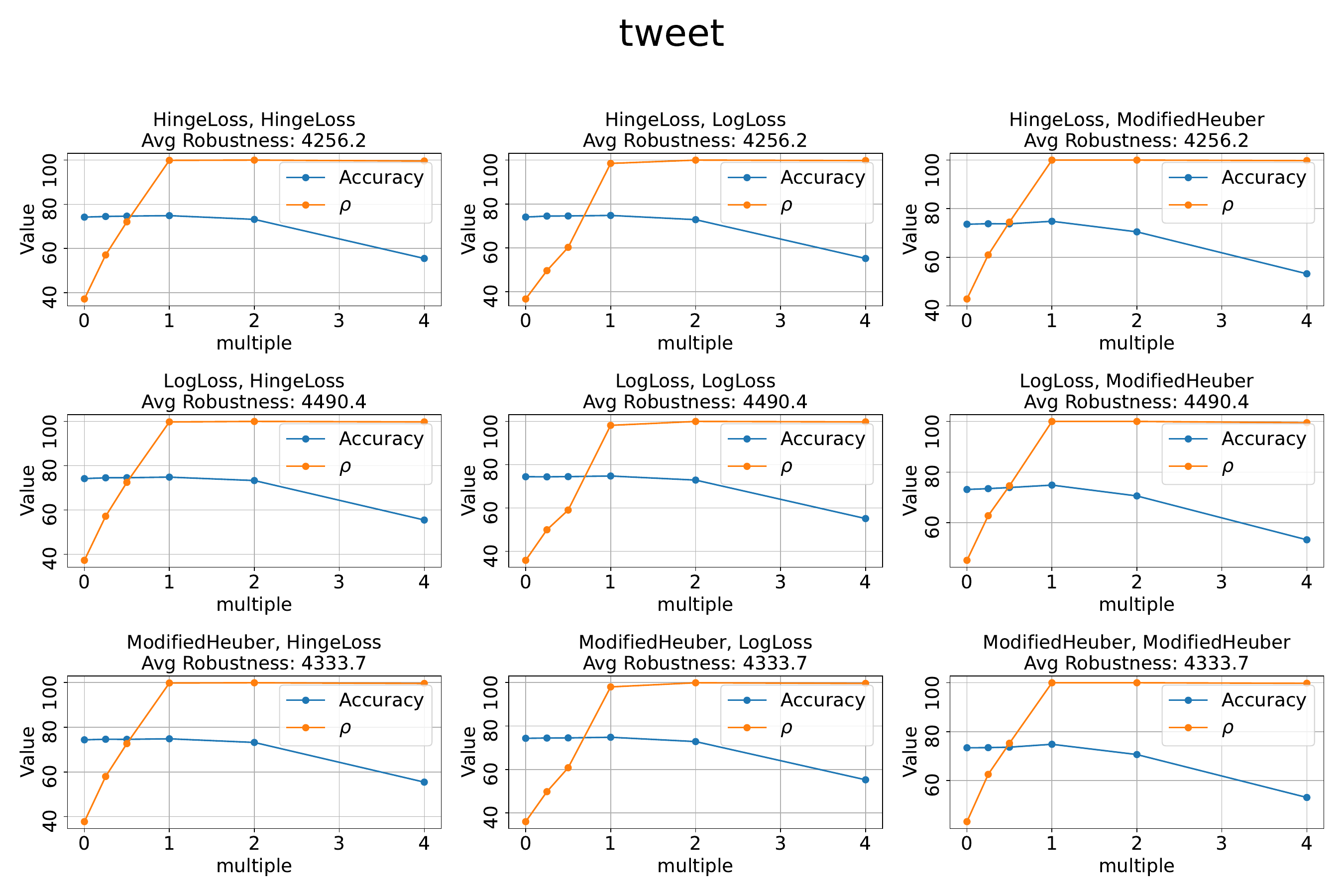}  
    \caption{Comparison of average accuracy vs $\rho$  for Tweet (BOW) dataset with $\{0, \frac{\uprob}{4}, \frac{\uprob}{2}, \uprob, 2\uprob, 4\uprob\}$ and loss functions.}  
    \label{fig:grid_tweet}  
\end{figure*}  
  
\begin{figure*}[!htbp]  
    \centering  
    \includegraphics[width=\textwidth]{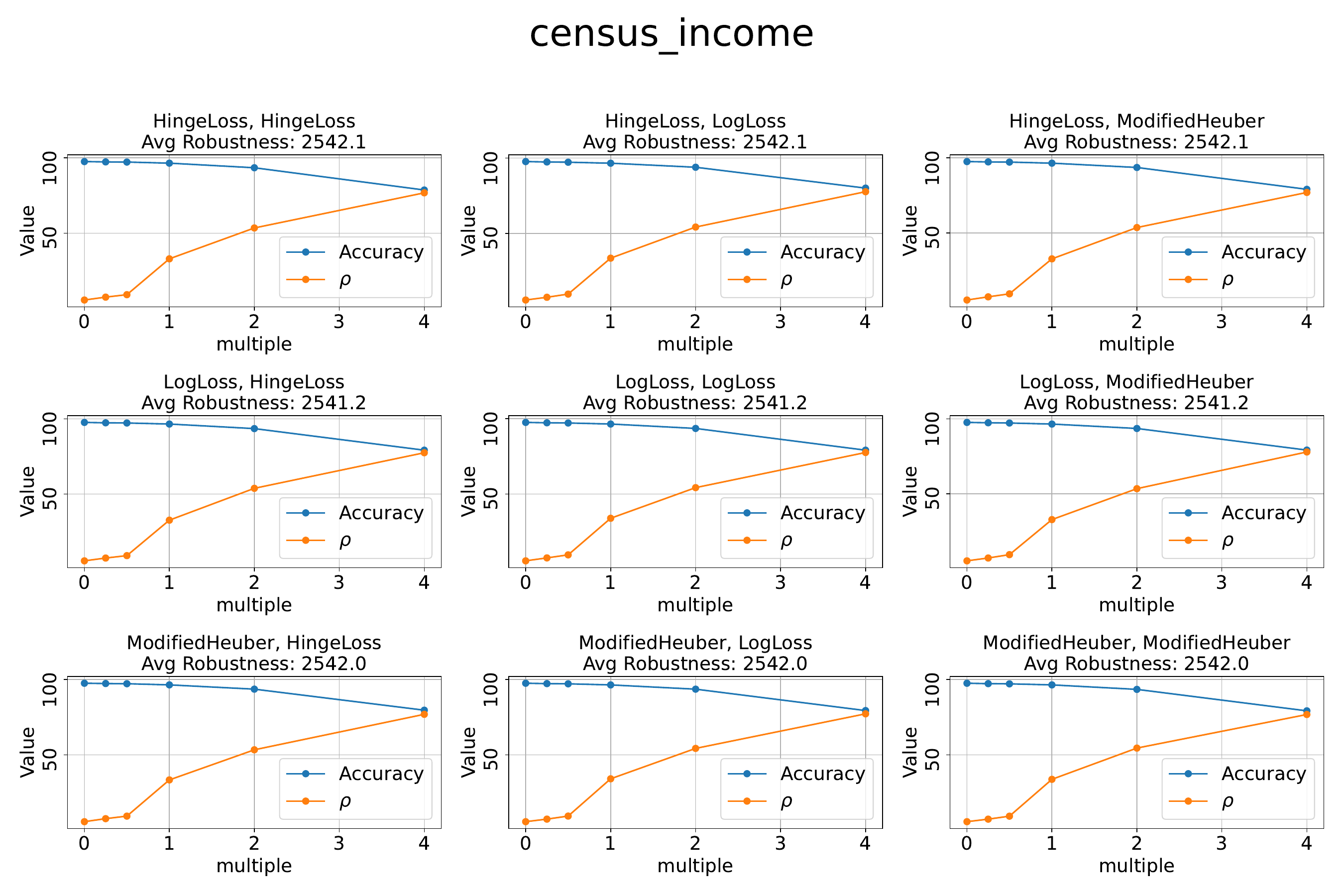}  
    \caption{Comparison of average accuracy vs $\rho$  for Census Income dataset with $\{0, \frac{\uprob}{4}, \frac{\uprob}{2}, \uprob, 2\uprob, 4\uprob\}$ and loss functions.}  
    \label{fig:grid_census_income}  
\end{figure*}  
\clearpage

\end{document}